\documentclass{article}

\usepackage{microtype}
\usepackage{graphicx}
\usepackage{subcaption}
\usepackage{booktabs} 

\usepackage{hyperref}

\PassOptionsToPackage{noend}{algorithmic}
\usepackage{algorithmic}

\usepackage[preprint]{icml2026}

\usepackage{amsmath}
\usepackage{amssymb}
\usepackage{mathtools}
\usepackage{amsthm}

\usepackage[capitalize,noabbrev]{cleveref}

\theoremstyle{plain}
\newtheorem{theorem}{Theorem}[section]
\newtheorem{proposition}[theorem]{Proposition}
\newtheorem{lemma}[theorem]{Lemma}
\newtheorem{corollary}[theorem]{Corollary}
\theoremstyle{definition}
\newtheorem{definition}[theorem]{Definition}

\theoremstyle{remark}

\usepackage[textsize=tiny]{todonotes}

\usepackage{bbm}
\usepackage{tikz}
\usetikzlibrary{arrows.meta, positioning}
\usepackage{subcaption}
\usepackage{enumitem}
\usepackage{titletoc}

\usepackage[table]{xcolor}  
\definecolor{gold}{rgb}{1.0, 0.84, 0.0}
\definecolor{silver}{rgb}{0.75, 0.75, 0.75}
\definecolor{bronze}{rgb}{0.8, 0.5, 0.2}

\makeatletter
\newcommand{\capfirst}[1]{%
  \expandafter\capfirst@i#1\relax
}
\def\capfirst@i#1#2\relax{%
  \MakeUppercase{#1}#2%
}
\makeatother

\newcommand{\smallest}{basis}
\newcommand{\largest}{closure}

\tikzset{
  var/.style={circle, draw, minimum size=8mm, inner sep=0pt},
  target/.style={var, thick, fill=gray!10},
  >=Stealth
}

\begin{document}

\twocolumn[
  \icmltitle{Exactly Computing do-Shapley Values}



  \icmlsetsymbol{equal}{*}

  \begin{icmlauthorlist}
    \icmlauthor{R. Teal Witter}{equal,cmc}
    \icmlauthor{Álvaro Parafita}{equal,supercomp}
    \icmlauthor{Tomàs Garriga}{supercomp,novartis}
    \icmlauthor{Maximilian Muschalik}{lmu}
    \icmlauthor{Fabian Fumagalli}{lmu}
    \icmlauthor{Axel Brando}{supercomp}
    \icmlauthor{Lucas Rosenblatt}{nyu}
  \end{icmlauthorlist}

  \icmlaffiliation{cmc}{Mathematical Sciences Department, Claremont McKenna College}
  \icmlaffiliation{supercomp}{Barcelona Supercomputing Center, Barcelona, Spain}
  \icmlaffiliation{novartis}{Novartis}
  \icmlaffiliation{lmu}{LMU Munich, MCML}
  \icmlaffiliation{nyu}{Department of Computer Science and Engineering, New York University}

  \icmlcorrespondingauthor{R. Teal Witter}{rtealwitter@cmc.edu}
  \icmlcorrespondingauthor{Álvaro Parafita}{parafita.alvaro@gmail.com}

  \icmlkeywords{Machine Learning, ICML}

  \vskip 0.3in
]



\printAffiliationsAndNotice{}  

\begin{abstract}
Structural Causal Models (SCM) are a powerful framework for describing complicated dynamics across the natural sciences.
A particularly elegant way of interpreting SCMs is do-Shapley, a game-theoretic method of quantifying the average effect of $d$ variables across exponentially many interventions.
Like Shapley values, computing do-Shapley values generally requires evaluating exponentially many terms.
The foundation of our work is a reformulation of do-Shapley values in terms of the \text{irreducible sets} of the underlying SCM.
Leveraging this insight, we can exactly compute do-Shapley values in time linear in the number of irreducible sets $r$, which itself can range from $d$ to $2^d$ depending on the graph structure of the SCM.
Since $r$ is unknown a priori, we complement the exact algorithm with an estimator that, like general Shapley value estimators, can be run with any query budget.
As the query budget approaches $r$, our estimators can produce more accurate estimates than prior methods by several orders of magnitude, and, when the budget reaches $r$, return the Shapley values up to machine precision.
Beyond computational speed, we also reduce the identification burden: we prove that non-parametric identifiability of do-Shapley values requires only the identification of interventional effects for the $d$ singleton coalitions, rather than all classes.
\end{abstract}

\section{Introduction}

\begin{figure}[t]
    \centering
    \includegraphics[width=\linewidth]{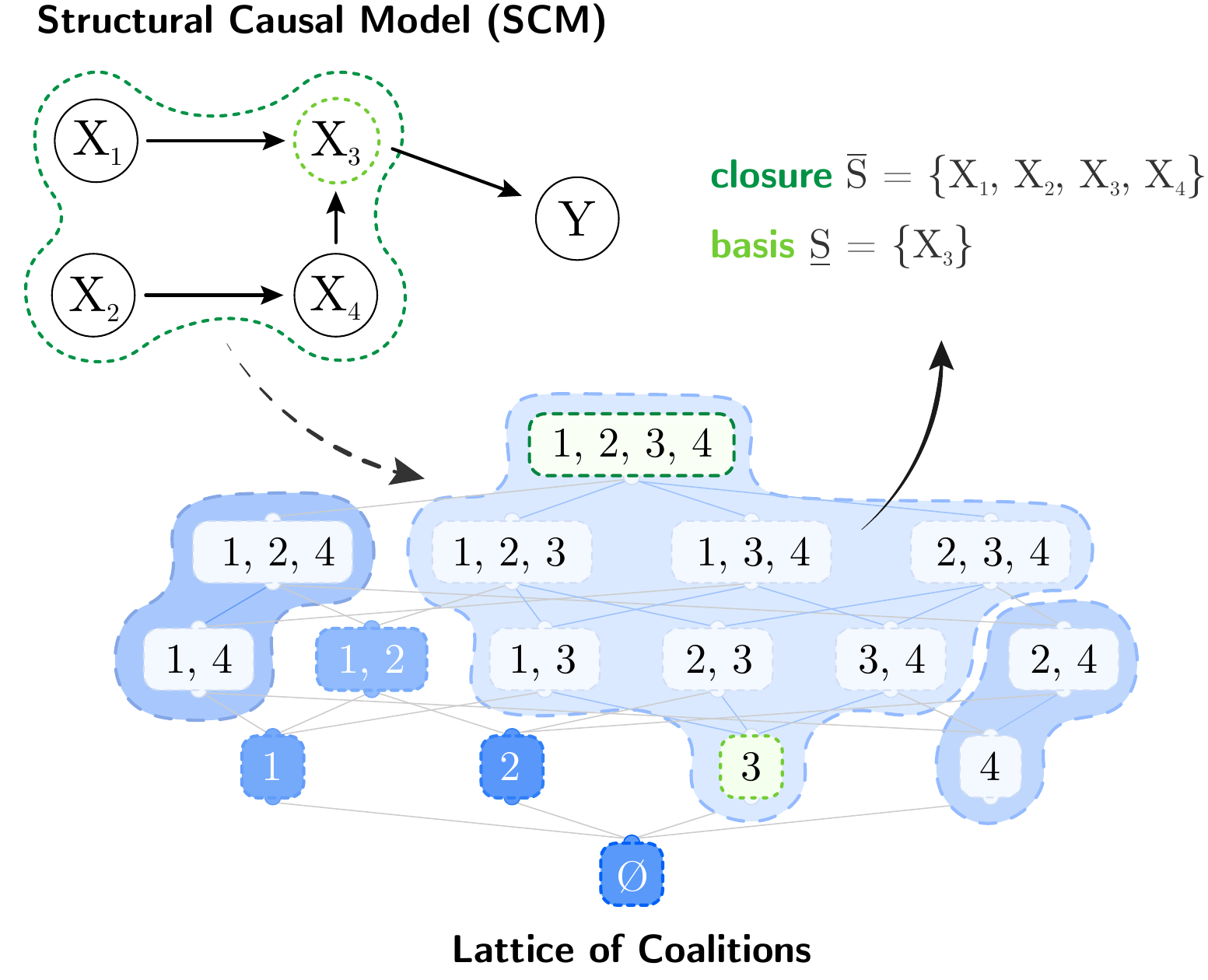}
    \caption{An example Structural Causal Model (SCM) and the corresponding lattice of coalitions. Because of the graph structure, intervening on some nodes is redundant. For example, setting $\{X_1, X_2, X_3, X_4\}$ has the same effect as setting $\{X_3\}$ because $X_3$ blocks all directed paths from the other nodes to $Y$. For such a \textit{class}, we refer to its smallest coalition (e.g., $\{X_3\}$) as the \textit{basis}, and the largest coalition (e.g., $\{X_1, X_2,X_3,X_4\}$) as the \textit{closure.}}
    \label{fig:intro_illustration}
\end{figure}

The question of causality is crucial to scientific inquiry, ranging from policy evaluation in economics to treatment effects in healthcare. Yet, observational data alone is often insufficient due to the fundamental problem of causal inference: because we cannot observe the counterfactual world where a specific intervention did \textit{not} occur, we cannot definitively state, based on data alone, that one event \textit{caused} another \citep{holland1986statistics, rubin1974estimating}.

Structural Causal Models (SCMs) offer a powerful solution by explicitly modeling the underlying mechanisms of a system \citep{pearl2009causality}. 
Whether derived from established domain knowledge or learned via causal discovery algorithms \citep{peters2017elements}, SCMs formalize the data-generating process:
a directed acyclic graph $G$ representing causal relationships, and a set of structural equations that determine the value of each node as a function of its parents and exogenous noise \citep{bareinboim2016causal}.
We provide a more formal introduction to SCMs in Appendix \ref{app:backgroundscm}.

\begin{figure*}[t!]
    \centering
    
    \includegraphics[width=0.9\linewidth]{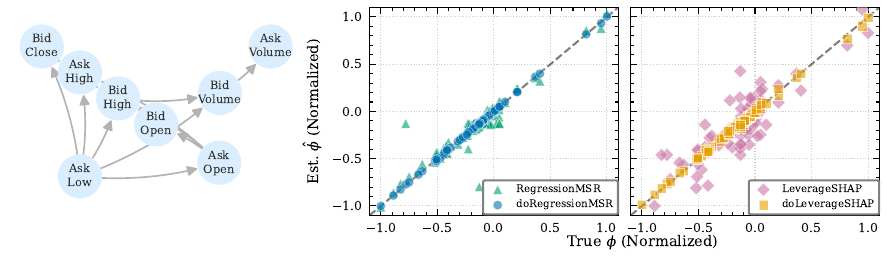}
    \vspace{-0.5em}
    \caption{\textbf{Left: A learned SCM} from a TALENT dataset.
    Nodes and edges represent the learned causal graph used to define the interventional value function $\nu(S)=\mathbb{E}[Y\mid \mathrm{do}(S=\mathbf{x}_S)]$ for a fixed instance $\mathbf{x}$. 
    \textbf{Right: Plots of estimated vs true do-Shapley values} on the learned SCM for randomly sampled $\mathbf{x}$. Compared to the value-function-agnostic state-of-the-art \texttt{RegressionMSR} and \texttt{LeverageSHAP} estimators, our \textit{doEstimator} variants provide substantially more accurate do-Shapley value estimates.}
    \label{fig:scm_single}
\end{figure*}

With a fully specified SCM, we can rigorously evaluate the effect of specific actions using the $do$-operator \citep{pearl2009causality}.
This operator simulates an intervention where a subset of variables is forced to take specific values, independent of their natural causes.
Consider a specific instance of interest $\mathbf{x} \in \mathbb{R}^d$. We define the value function $\nu(S)$ as the expected value of the target outcome $Y$ when the subset of features $S \subseteq [d]$ is intervened upon to match their observed values in $\mathbf{x}$:
\begin{align}
\nu(S) = \mathbb{E}[Y \mid \text{do}(S = \mathbf{x}_S)].
\label{eq:dovaluefunction}
\end{align}

This formulation enables us to precisely answer hypothetical queries, such as: 
\textit{``If we explicitly set this student's income and tutoring time, how would their probability of admission change?''}
or
\textit{``If a patient were administered prednisone and made to stop smoking, what would be their expected pain level?''}
However, characterizing the system's behavior purely through these individual queries is challenging. 
As the number of features $d$ grows, the number of possible interventional subsets scales as $2^d$. 
To extract interpretable insights from this combinatorial landscape, we need a unified framework to attribute the complicated dynamics of the SCM to individual features.

The Shapley value \citep{shapley1953shap} provides a rigorous framework for such explanations by attributing the changes in the outcome $Y$ to individual variables based on their marginal contributions.
Formally, the $i$th Shapley value captures the weighted average effect of adding variable $i$ to a coalition $S$:
\begin{align}
\phi_i = \sum_{S \subseteq [d] \setminus {i }}
[\nu(S \cup {i }) - \nu(S)] p_{|S|}
\label{eq:shapley}
\end{align}
where the weight $p_\ell = \frac1d \binom{d-1}{\ell}^{-1}$ can be interpreted as a probability distribution.

When the value function is defined via the interventional $do$-operator (Equation \ref{eq:dovaluefunction}), the result is the \textit{do-Shapley value} \citep{jung2022do_shap}, also referred to as the causal Shapley value \citep{heskes2020causal_shap}.\footnote{For conciseness, we henceforth refer to the do-value function simply as the value function, and the do-Shapley value as the Shapley value.}
Unlike standard formulations that rely on conditional expectations \citep{lundberg2017shap} or restrictive path-dependent permutations \citep{frye2020asymmetric_shap}, this metric strictly isolates the total causal effect of a feature intervention.
This rigorous isolation allows us to translate abstract model dynamics into concrete causal attributions, such as statements like \textit{``High family income increased acceptance probability by $10\%$''} or \textit{``Prescribing prednisone decreased reported pain by two marks.''}

The challenge in computing the Shapley value, of course, is that there are still $2^d$ terms $\nu(S)$.
So, without additional structure in $\nu$, exactly computing the Shapley value would take exponential time.
To address this, the standard approach is to approximate the Shapley value using stochastic estimators that evaluate $\nu(S)$ on a limited budget of sampled coalitions.
A diverse array of model-agnostic estimators has been developed for this purpose, including direct Monte Carlo estimators \citep{strumbelj2014explaining}, permutation-based sampling \citep{Castro.2009}, and regression-based formulations such as \texttt{KernelSHAP} and \texttt{LeverageSHAP} \citep{lundberg2017shap, covert2021improving,musco2025provably}.

For do-Shapley values, recent work has exploited the observation that the topological structure of SCMs often renders specific interventions redundant \citep{parafita2025practical}.
For example, in the causal graph depicted in Figure \ref{fig:intro_illustration}, the intervention on $\{X_3\}$ results in the same value as the intervention on $\{X_1, X_3, X_4\}$, because the paths from $X_1$ and $X_2$ to $Y$ are blocked by $X_3$.
To formalize this, define the \textit{\smallest{}} of a coalition $S$ as the subset $\underline{S} \subseteq S$ containing precisely those variables $j \in S$ that possess a directed path to $Y$ that does \textit{not} traverse any other node in $S$.
In effect, any variable in $S \setminus \underline{S}$ is intercepted by $\underline{S}$ and yields no additional impact on the outcome, ensuring $\nu(S) = \nu(\underline{S})$.
A set is \textit{irreducible} if it is its own basis.
This property enables a caching strategy: rather than naively evaluating the SCM for every query $\nu(S)$, the estimator first computes the basis $\underline{S}$ and checks if $\nu(\underline{S})$ has been previously memoized.
If so, the cached value is returned; if not, the SCM is evaluated and the result stored.
This mechanism avoids redundant evaluations of the underlying model, resulting in significant computational speedups \citep{parafita2025practical}.

We extend this insight by observing that every \smallest{} $\underline{S}$ is associated with a unique \textit{\largest{}} $\bar{S} \supseteq \underline{S}$—the maximal coalition such that intervening on $\bar{S}$ yields the identical effect as intervening on $\underline{S}$. 
For example, the \largest{} of $\{X_3\}$ in Figure \ref{fig:intro_illustration} is $\{X_1,X_2,X_3,X_4\}$.
Together, these bounds define an equivalence class of coalitions $\{ S \subseteq [d] : \underline{S} \subseteq S \subseteq \bar{S} \}$, all of which map to the same value $\nu(\underline{S})$.

Crucially, these classes constitute a partition of the powerset of all $d$ features. Letting $r$ denote the total number of such classes $c_1, \dots, c_r$, we leverage this structure to compress the Shapley summation into a linear combination of class values:
\begin{align}
\phi_i = \sum_{j=1}^r \nu(c_j) w_i(c_j)
\end{align}
where the class weight $w_i(c_j)$ is derived in Equation \ref{eq:weights}.
A similar decomposition is known for trees \cite{zern2023interventional,witter2025regressionadjusted}; but, unlike trees where the structure can be read in a linear pass of the leaves, efficiently finding the class structure of an SCM is non-trivial.

This formulation reduces the exact computation of Shapley values to a sum over $r$ terms. To identify these classes efficiently, we propose a lattice exploration algorithm that circumvents the exhaustive enumeration of the powerset. Leveraging a structural property of closed sets (Lemma \ref{lemma:closed_plus}), our algorithm enumerates all $r$ classes in $O(r(d+e+T))$ time, where $e$ is the number of edges in the causal graph and $T$ is the time to query the value function once.

The efficiency of this approach is strictly governed by the underlying graph topology.
As illustrated in Figure \ref{fig:num_irreducible_sets}, the number of classes $r$ varies with the graph structure, ranging from a linear $d$ in the best case to $2^d$ in the worst case.

To quantify how much this compression manifests in practice, we plot the number of irreducible sets $r$ against the number of variables $d$ across real datasets in Figure~\ref{fig:complexity}.
The observed scaling typically lies below the worst-case $2^d$, reflecting the sparsity of learned causal graphs in real-world tabular domains.

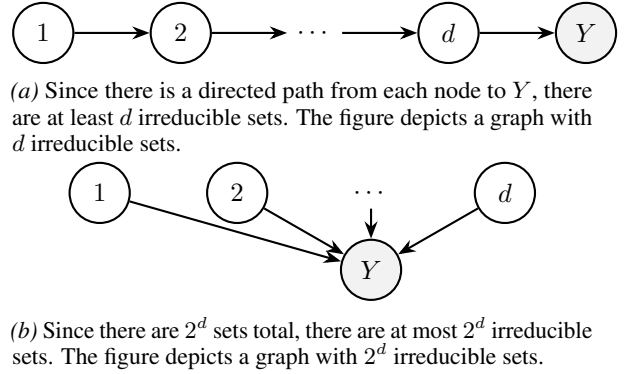
\begin{figure}[t]
  \centering
  \begin{subfigure}{0.45\textwidth}
    \centering
    \begin{tikzpicture}[node distance=1cm, thick]
      \node[var] (x1) {$1$};
      \node[var, right=of x1] (x2) {$2$};
      \node[right=of x2] (dots) {$\cdots$};
      \node[var, right=of dots] (xd) {$d$};
      \node[target, right=of xd] (y) {$Y$};

      \draw[->] (x1) -- (x2);
      \draw[->] (x2) -- (dots);
      \draw[->] (dots) -- (xd);
      \draw[->] (xd) -- (y);
    \end{tikzpicture}
    \caption{
        Since there is a directed path from each node to $Y$, there are at least $d$ irreducible sets. The figure depicts a graph with $d$ irreducible sets.
    }
  \end{subfigure}
  \hfill
  \begin{subfigure}{0.45\textwidth}
    \centering
    \begin{tikzpicture}[node distance=1cm, thick]
      \node[var] (x1) {$1$};
      \node[var, right=of x1] (x2) {$2$};
      \node[right=of x2] (dots) {$\cdots$};
      \node[var, right=of dots] (xd) {$d$};

      \node[target, below=.4cm of dots] (y) {$Y$};

      \draw[->] (x1) -- (y);
      \draw[->] (x2) -- (y);
      \draw[->] (dots) -- (y);
      \draw[->] (xd) -- (y);
    \end{tikzpicture}
    \caption{Since there are $2^d$ sets total, there are at most $2^d$ irreducible sets. The figure depicts a graph with $2^d$ irreducible sets.}
  \end{subfigure}
  \caption{The number of irreducible sets ranges between $d$ and $2^d$.}
    \label{fig:num_irreducible_sets}
\end{figure}

Although our exact algorithm scales linearly with $r$, the number of classes $r$ is unknown a priori. 
Since real-world applications demand strict resource limits, we must often operate within a fixed computational budget of $m$ value function queries. 
To address this, we propose a class of estimators explicitly tailored to this constrained setting.

\begin{figure}[b]
    \centering
    \includegraphics[width=0.9\linewidth]{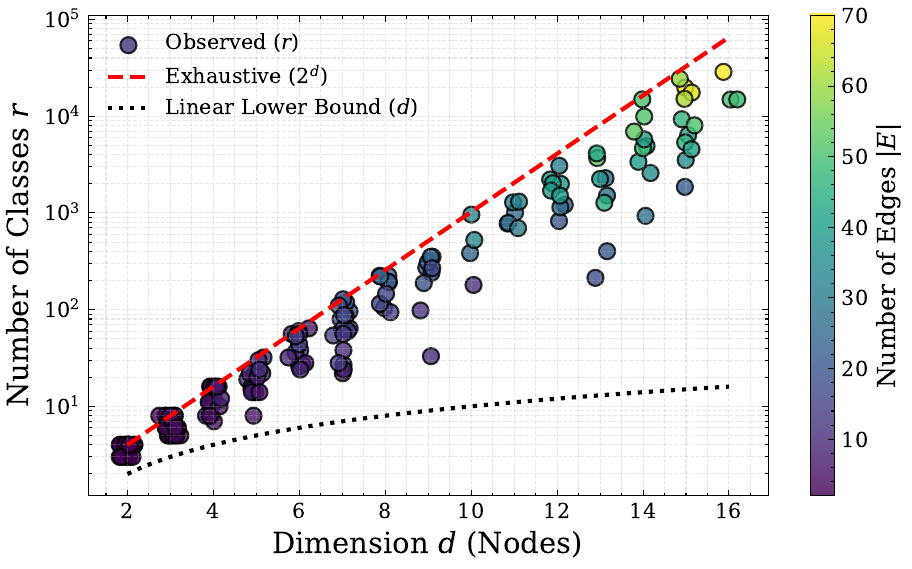}
    \caption{\textbf{Complexity Reduction.} The number of irreducible sets $r$ (colored points) versus the dimension $d$. While the theoretical worst-case complexity is $2^d$ (red dashed line), real-world causal structures are often sparse, resulting in $r$ scaling in between the exponential and the linear lower bound $d$ (black dotted line). }
    \label{fig:complexity}
\end{figure}

The fundamental limitation of prior caching-based approaches is \textit{sample redundancy}. Standard estimators sample coalitions without knowledge of the underlying causal structure, meaning they can (wastefully) query different coalitions that belong to the same large equivalence class.
Thus, a budget of $m$ queries often produces far fewer than $m$ unique values.
We resolve this inefficiency by introducing a \textit{boundary sampler}, a targeted exploration strategy guaranteed to identify $\min(m, r)$ distinct equivalence classes when run with $m$ queries.
By feeding these distinct values into a simulated estimator, we maximize the information extracted from the available budget. 
We find that this method can reduce estimation error by orders of magnitude compared to the best value-function-agnostic estimators run with the caching scheme of \citet{parafita2025practical}.
Furthermore, the estimator exhibits seamless convergence: as the budget $m$ approaches $r$, the approximation error vanishes, achieving exact computation (up to machine precision) once $m \geq r$.

In practice, SCMs are frequently learned from observational data and a hypothesized graph structure.
A critical prerequisite for this process is \textit{identifiability}: determining whether a causal query can be uniquely estimated from the observed probability distribution given the graph \citep{pearl2009causality}.
For instance, in the presence of unobserved confounding, different structural parameters could yield the exact same observational distribution but different interventional outcomes.
The gold standard for verifying non-parametric identifiability is the ID algorithm \citep{shpitser2006interventional, tian2002testable}, which determines if a specific query $\nu(S)$ is computable from the observational distribution.

This creates a practical bottleneck for Shapley value estimation. Since the Shapley value aggregates $r$ terms, a practitioner using prior methods would be forced to run the estimator and iteratively check identifiability for each irreducible set encountered.
If a single coalition proved unidentifiable, the entire estimation would be invalidated after significant computation \citep{parafita2025practical}.
In Section \ref{sec:identifiable}, we resolve this burden by proving a structural sufficiency theorem: to guarantee the identifiability of \textit{all} $2^d$ coalitions, it suffices to verify identifiability for only the $d$ singleton interventions $\{i\} \subset [d]$. This allows practitioners to run a rapid, $O(d)$ sanity check before model training begins, ensuring that the resulting Shapley values will be valid without the risk of costly mid-computation failures.

In summary, our contributions are three-fold:

\textbf{1. Exact Computation via Irreducible Sets:} We propose an algorithm that computes exact do-Shapley values by exploiting the graph's equivalence classes. By traversing the lattice of closed sets, the algorithm runs in time linear in the number of irreducible sets $r$ (and the graph size $e$), rather than the worst-case exponential complexity of $2^d$.

\textbf{2. Structure-Aware Estimation:} We introduce a class of boundary sampling estimators designed for fixed-budget settings. Unlike prior methods that sample blindly, our approach targets distinct equivalence classes. This yields error reductions of several orders of magnitude as $m$ approaches $r$, and seamlessly transitions to exact machine-precision computation once the budget satisfies $m \geq r$.

\textbf{3. Efficient Identifiability Check:} We prove that non-parametric identifiability of the full do-Shapley value is guaranteed if and only if the $d$ singleton interventions are non-parametrically identifiable. This result enables a rapid $O(d)$ sanity check, allowing practitioners to verify the feasibility of the explanation task before incurring the cost of model training or estimation.

While we focus on the Shapley value due to its widespread adoption, our theoretical insights generalize to the broader class of probabilistic values---e.g., Banzhaf values \cite{Banzhaf.1964}, beta Shapley values \cite{Kwon.2022b}, and weighted Banzhaf values \cite{li2024robust}---by simply changing the marginal weights in Equation \ref{eq:do_shapley}. 
Furthermore, in Section \ref{sec:generalization}, we demonstrate how to adapt our methods to compute Shapley Interaction Indices \cite{Grabisch.1997}, capturing the joint causal impact of coalitions.

\subsection{Additional Related Work}
Exactly computing Shapley values is generally feasible only when the underlying model possesses exploitable structure \citep{Rozemberczki.2022}. 
This has led to efficient, model-specific algorithms for decision trees and ensembles, including TreeSHAP \citep{lundberg2020fromlocal,yu2022linear}, interventional variants \citep{zern2023interventional}, and extensions like TreeSHAP-IQ \citep{muschalik2024beyond}. Similar exact methods exist for linear models \citep{strumbelj2014explaining}, product-kernel networks \citep{mohammadi2025computing}, Gaussian processes \citep{mohammadi2025exact}, graph neural networks \citep{muschalik2025exact}, and KNN-based data valuation \citep{jia2019towards,wang2023privacy,wang2024efficient}. When black-box access precludes exact methods, practitioners rely on model-agnostic estimators \citep{chen2023shap_survey,muschalik2024shapiq} such as Monte Carlo sampling \citep{Castro.2009,Kolpaczki.2024b,Kolpaczki.2024a,Fumagalli.2023,Wang.2023} or regression-based approaches like \texttt{KernelSHAP} \citep{lundberg2017shap}, which has been enhanced via leverage scores \citep{musco2025provably} and interaction support \citep{fumagalli2024kernelshapiq,tsai2022faith}.
Finally, \emph{exact computation} has recently merged with \emph{estimation} via surrogate modeling \citep{butler.2025}, where auxiliary models are fit to the value function to allow efficient extraction of Shapley values \citep{witter2025regressionadjusted}.

\section{Reformulating do-Shapley Values}

In this section, we leverage the underlying structure of an SCM to reformulate do-Shapley values in terms of equivalence classes. 

Firstly, we will assume that all nodes in $[d]$ are ancestors of the target node $Y$, since non-ancestors have null do-Shapley value. We will start with the definition of \smallest{}, derived from the concept of irreducible sets in \citet{parafita2025practical}.

\begin{definition}[\capfirst{\smallest{}}]
    The \textit{\smallest{}} of $S$, denoted $\underline{S}$, is the subset of nodes $j \in S$ such that there exists a directed path from $j$ to $Y$ that intersects $S$ only at $j$.
\end{definition}

A set $S \subseteq [d]$ is \textit{irreducible} if it is its own \smallest{}.

Similarly, we will define the \largest{} of $S$ as the superset of nodes that can be blocked from reaching $Y$ by $S$.

\begin{definition}[\capfirst{\largest{}}]
    The \textit{\largest{}} of $S$, denoted $\bar{S}$, is the set of all nodes $j \in [d]$ such that every directed path from $j$ to $Y$ intersects $S$.
\end{definition}

We say a set $S \subseteq [d]$ is \textit{closed} if it is its own \largest{}.

Let $\underline{S}$ and $\bar{S}$ be the \smallest{} and \largest{} of a coalition $S$, respectively.
$S$ belongs to an \textit{equivalence class} with all $T$ such that $\underline{S} \subseteq T \subseteq \bar{S}$.
By definition, all nodes in $T \setminus \underline{S}$ are \textit{blocked} from reaching $Y$ by $\underline{S}$, so, by the third rule of do-Calculus \cite{pearl2009causality},
\begin{align*}
    \nu(\underline{S})  
    = \nu(T)
    = \nu(\bar{S}).
\end{align*}
It is easy to see that the classes form a partition of all $2^d$ coalitions.

Let $r$ be the number of irreducible sets,
and denote the classes by $c_1, \ldots, c_r$.
In an abuse of notation, we will define $\nu(c_j)=\nu(S)$ where $S$ is any set in class $c_j$.
We will use this structure to rewrite the Shapley values:

\begin{align}
    \phi_i 
    &= \sum_{S \subseteq [d] \setminus \{i\}}
    p_{|S|} \left[ \nu(S \cup \{i\}) - \nu(S)\right]
    \nonumber \\&= \sum_{S \subseteq [d]} \nu(S) \left[
    \mathbbm{1}[i \in S] p_{|S|-1}
    -\mathbbm{1}[i \notin S] p_{|S|}
    \right]
    \nonumber \\&= \sum_{j=1}^r \nu(c_j) \cdot w_i(c_j)
    \label{eq:do_shapley}
\end{align}
where, with $\bar{S}$ as the \largest{} of class $c$ and $\underline{S}$ as the \smallest{} of class $c$, we define
\begin{align}
    w_i(c) = \sum_{T: \underline{S} \subseteq T \subseteq \bar{S}}
    \left[
    \mathbbm{1}[i \in T] p_{|T|-1}
    -\mathbbm{1}[i \notin T] p_{|T|}
    \right].
    \nonumber
\end{align}

Even though there could be exponentially many subsets in a class, we can compute $w_i(c)$ in $O(d)$ time.
In particular, it is easy to show that

\begin{align}
w_i(c) = \begin{cases}
    \sum_{\ell=|\underline{S}|}^{|\bar{S}|} p_{\ell-1}
    \binom{|\bar{S}| - |\underline{S}|}{\ell-|\underline{S}|} & i \in \underline{S} \\[10pt]
    - \sum_{\ell=|\underline{S}|}^{|\bar{S}|} p_\ell
    \binom{|\bar{S}| - |\underline{S}|}{\ell - |\underline{S}|} & i \notin \bar{S} \\[10pt]
    0 & \text{else}.
\end{cases}
\label{eq:weights}
\end{align}

We will use this structure to compute Shapley values.
If we have $O(r(d+e))$ time, then we can exactly compute Shapley values as described in Section \ref{sec:exact}.
Since $r$ is initially unknown, we may also want to estimate Shapley values given a fixed query budget $m$.
In Section \ref{sec:estimate}, we describe estimators that run in $O(m(d+e))$ time.

\begin{algorithm}[tb]
  \caption{\texttt{FindClass}}
  \label{alg:find_class}
  \begin{algorithmic}
    \STATE {\bfseries Input:} Set $S \subseteq [d]$, graph $G$
    \STATE {\bfseries Output:} \smallest~ $\underline{S} \subseteq S$, and \largest~$\bar{S} \supseteq S$
    \STATE $G' \gets G$ with all incoming edges to nodes in $S$ removed
    \STATE $N_\text{anc} \gets$ ancestor nodes i.e., a directed path to $Y$ in $G'$
    \STATE $\bar{S} \gets S \cup ([d]\setminus N_\text{anc})$ \hfill {\color{gray!90} $\triangleright$ $S$ and nodes not connected to $Y$}
    \STATE $\underline{S} \gets S \cap N_\text{anc}$ \hfill {\color{gray!90}$\triangleright$ Subset of $S$ still connected to $Y$}
    \STATE \textbf{return} $\underline{S}, \bar{S}$
  \end{algorithmic}
\end{algorithm}

\section{Exactly Computing do-Shapley Values}
\label{sec:exact}

By Equation \ref{eq:do_shapley}, computing all irreducible sets is sufficient to exactly compute the Shapley value.
It remains to find all irreducible sets.

A naive strategy is to traverse the set lattice by subset size, determining the class of each set via Algorithm \ref{alg:find_class}.
Of course, there are $2^d$ sets on the lattice, so even constant work per set is infeasible. Instead, we can \textit{efficiently} traverse the set lattice by only generating sets which are guaranteed to be closed, amortizing the work to each class, rather than each subset.
The key tool is a structural lemma on alternate definitions of closed sets. 

\begin{algorithm}[b]
  \caption{\texttt{AllClasses}}
  \label{alg:find_all_classes}
  \begin{algorithmic}
    \STATE {\bfseries Input:} Number of elements $d$, graph $G$
    \STATE {\bfseries Output:} All closed sets $\mathcal{C}$
    \STATE $\mathcal{C}_0 \gets  \ldots \gets \mathcal{C}_d \gets \emptyset $ \hfill {\color{gray!90}$\triangleright$ Closed sets of each size}
    \STATE $\mathcal{C}_d \gets \{\{1, \dots, d\}\}$  \hfill {\color{gray!90} $\triangleright$ Only closed set of size $d$ }
    \FOR{$\ell=d, \ldots, 1$}
        \FOR{$\bar{S} \in \mathcal{C}_{\ell}$}
            \STATE $\underline{S}, \bar{S} \gets \texttt{FindClass}(\bar{S}, G)$
            \STATE {\color{gray!90}$\triangleright$ Use Lemma \ref{lemma:closed_plus}}
            \FOR{$j \in \underline{S}$}                
                \STATE Add $\bar{S} \setminus \{j\}$ to $\mathcal{C}_{\ell - 1}$
            \ENDFOR
        \ENDFOR
    \ENDFOR
    \STATE {\bfseries return} $\mathcal{C}_0 \cup \ldots \cup \mathcal{C}_d$
  \end{algorithmic}
\end{algorithm}

\begin{lemma}\label{lemma:closed_plus}
    Let $\bar{S} \subset [d]$ be a closed set with \smallest{} $\underline{S}$.
    Then
    \begin{enumerate}
        \item For all $j \in \underline{S}$, $\bar{S} \setminus \{j\}$ is closed.
        \item If $\bar{S} \neq [d]$, there exists $j \in [d] \setminus \bar{S}$ so that $\bar{S} \cup \{j\}$ is closed and $j$ is in the \smallest{} of $\bar{S} \cup \{j\}$.
    \end{enumerate}
\end{lemma}

We defer the proof of the lemma to Appendix \ref{app:delayedproof}.

Algorithm \ref{alg:find_all_classes} describes our method.
We efficiently find each class by iterating over closed sets, in decreasing order of size.
We start with the full set $[d]$.
For each closed set $\bar{S}$ of size $\ell$, we compute its \smallest{} $\underline{S}$.
By Lemma \ref{lemma:closed_plus}, $\bar{S} \setminus \{j\}$ is closed for all $j \in \underline{S}$.
We then add each of these closed sets, and further explore them when we reach size $\ell-1$.

We can see that Algorithm \ref{alg:find_all_classes} correctly returns all closed sets by an inductive argument.
Suppose that we have identified all closed sets of size $\ell$; the base case is trivial since there is only one set, and it must be closed.
By Lemma \ref{lemma:closed_plus}, for every closed set $\bar{S}$ of size $\ell-1$, there is a closed set $\bar{S} \cup \{j\}$ for some $j$ in the \smallest{} of $\bar{S} \cup \{j\}$.
By the inductive assumption, we must have identified this set, and also found $\bar{S}$ by removing $j$ from $\bar{S} \cup \{j\}$.
It follows that every closed set of size $\ell-1$ gets generated by some closed set of size $\ell$.

Algorithm \ref{alg:find_all_classes} runs in $O(r(d+e))$ time:
for each closure---there is one closure for each of the $r$ classes---the algorithm runs a graph exploration in time $d + e$, and then adds at most $d$ closed sets of size $\ell-1$ to explore.

\paragraph{Simple Class Optimization}
Sometimes running Algorithm \ref{alg:find_class} as a subroutine in Algorithm \ref{alg:find_all_classes} can be avoided.
For a closed set $\bar{S}$ that is \textit{simple}---$\bar{S}$ is both its own \largest{} and \smallest{}---all of its subsets are also simple.
To see why, observe that all nodes $j \in \bar{S}$ have a directed path to $Y$ that does not intersect any other node in $\bar{S}$, and hence any other node in a subset of $\bar{S}$.
It follows that all subsets are irreducible; with Lemma \ref{lemma:closed_plus}, we have that all such subsets are also closed.
For a simple set, we add an optimization to Algorithm \ref{alg:find_all_classes} in our implementation so that all the bases of its candidates are cached, avoiding the $O(d+e)$ cost of Algorithm \ref{alg:find_class} for all subsets of a simple set.

Together, Equation \ref{eq:do_shapley} and Algorithm \ref{alg:find_all_classes} can compute Shapley values in time linear in $r$.

\begin{proposition}\label{prop:exact_do_shapley}    Shapley values of the intervention value function can be exactly computed in $O(r(d + e + T))$ time, where $T$ is the time to evaluate the game $\nu$ on a given coalition, and $r$ is the number of irreducible sets.
\end{proposition}

\section{Approximating do-Shapley Values}
\label{sec:estimate}

While Algorithm \ref{alg:find_all_classes} allows for exact computation in $O(r(d+e))$ time, the number of irreducible sets $r$ is unknown \textit{a priori}. In resource constrained settings where $r$ may be too large, we would like an approximation technique that operates within a fixed computational budget of $m$ value function queries.

Standard value-function-agnostic estimators
sample coalitions from a fixed distribution.
While caching can prevent re-evaluation of the SCM for known classes, these estimators suffer from sample redundancy: they blindly generate coalitions that may belong to equivalence classes already queried. As a result, a budget of $m$ queries often yields far fewer than $m$ distinct class values, wasting computational resources on redundant parts of the lattice.

To address this, we seek an estimator that guarantees the discovery of $\min(m, r)$ \textit{distinct} equivalence classes. If $m \geq r$, the method should naturally recover the exact Shapley values to machine precision. If $m < r$, it should prioritize classes with weight to minimize estimation error.

Recall that the Shapley value can be expressed as a weighted sum over equivalence classes:
\begin{align}
\label{eq:doshap2}
\phi_i = \sum_{j=1}^r \nu(c_j) \cdot w_i(c_j).
\end{align}
Directly estimating this sum presents a challenge: we cannot sample classes proportional to their weights $w_i(c_j)$ because the classes (and thus their weights) are unknown without graph exploration.
Furthermore, a stratified sampling approach—sampling coalitions without replacement to ensure unique classes—proves computationally expensive.
As detailed in Appendix \ref{app:stratified}, such a stratified sampling method incurs a cost quadratic in $m$, which defeats the purpose of fast approximation.

\paragraph{Boundary Sampling}
We propose a \textit{boundary sampler}, a targeted graph exploration strategy that achieves sample efficiency in time linear in $m$. Instead of sampling blindly from the powerset, we maintain candidate classes adjacent to those we have already visited.

The algorithm, described in Algorithm \ref{alg:boundary_estimator}, proceeds by maintaining a priority queue of candidate classes, ordered by their expected weight magnitude $\mathbb{E}_i[|w_i(c)|]$. In each iteration, we sample a class $c$ proportional to the expected magnitude of its weight, evaluate $\nu(c)$, and add it to our sampled set. We then expand the candidate by generating the neighbors of $c$ in the lattice. Specifically, for a class $c$ with \smallest{} $\underline{S}$ and \largest{} $\bar{S}$, the neighbors are defined as either lower neighbors $\{ \bar{S} \setminus \{j\} \mid j \in \underline{S} \}$ or upper neighbors $\{ \bar{S} \cup \{j\} \mid j \notin \bar{S} \}$.
For each neighbor, we run \texttt{FindClass} to determine its canonical representation and weight, adding it to the queue if it has not been seen. This ensures that every query results in a new class.
By starting at each level in the lattice, and adding neighbors above and below, we explore the lattice in a balanced way.

We state Algorithm \ref{alg:boundary_estimator}, and prove the following upper bound on its runtime, in Appendix \ref{app:delayedproof}.

\begin{proposition}\label{prop:boundary_sampler_time}
Algorithm \ref{alg:boundary_sampler} runs in $O(m \cdot d(d+e))$ time, where $m$ is the query budget, $d$ is the number of features, and $e$ is the number of edges in the graph.
\end{proposition}

\paragraph{Estimation via Simulation}
After running the boundary sampler, we utilize the set of sampled classes $\mathcal{C}$ to compute the Shapley values, as shown in Algorithm \ref{alg:boundary_estimator}.

\textit{Case $m \geq r$:} If the queue empties before the budget is reached, we have identified all irreducible sets. We proceed to compute the exact Shapley values using Equation \ref{eq:doshap2}.

\textit{Case $m < r$:} We leverage the fact that we have paid the cost to evaluate $\nu(c)$ for all $c \in \mathcal{C}$. We construct a \textit{simulated} estimator, described in Appendix \ref{app:simulatedsampling}. This allows us to generate many samples at no additional query cost, reducing the variance of the estimator while balancing time complexity.

We run Algorithm \ref{alg:boundary_estimator} where the base estimators are the current state-of-the-art value-function-agnostic Shapley value estimators \texttt{LeverageSHAP} \cite{musco2025provably} and \texttt{RegressionMSR} \cite{witter2025regressionadjusted}.

\begin{algorithm}[tb]
\caption{\texttt{doEstimator}}
\label{alg:boundary_estimator}
\begin{algorithmic}
\STATE {\bfseries Input:} Budget $m$, Game $v$, \texttt{BaseEstimator}, Sampling multiplier $k$
\STATE {\bfseries Output:} Shapley values estimates $\phi \in \mathbb{R}^d$
\STATE {\color{gray!90} $\triangleright$ Step 1: Sampling Phase}
\STATE $\mathcal{C}, \texttt{allSampled} \gets \texttt{BoundarySampler}(m, G)$
\STATE Query $\nu(c)$ for all $c \in \mathcal{C}$ \hfill {\color{gray!90} $\triangleright$ $\min(r, m)$ queries}

\STATE {\color{gray!90} $\triangleright$ Step 2A: Exact Computation}
\IF{\texttt{allSampled}}%
    \STATE $\phi \gets \mathbf{0}$
    \FOR{each class $c \in \mathcal{C}$}
        \STATE $\phi_i \gets \phi_i + \nu(c) \cdot w_i(c)$ \textbf{ for all } $i$
    \ENDFOR
    \textbf{return} $\phi$
\ENDIF

\STATE {\color{gray!90} $\triangleright$ Step 2B: Run \texttt{BaseEstimator}}
\STATE $\mathcal{D} \gets \texttt{SimulatedSampler}(\mathcal{C}, k \cdot m)$ {\color{gray!90} $\triangleright$ Algorithm \ref{alg:simulated_sampler}}
\STATE {\bfseries return} $\texttt{BaseEstimator}(\mathcal{D})$
\end{algorithmic}\end{algorithm}

\section{Identifiability}
\label{sec:identifiable}

Whenever a causal query $P_S(T)$ is uniquely determined by the graph and dataset, we say the query is (non-parametrically) \textit{identifiable}. However, this is not always the case; we include an illustrative example in Appendix \ref{appendix:identifiability}.

Computing the do-Shapley value requires verifying the identifiability of component queries $\nu(S)$ via the ID algorithm \cite{shpitser2006interventional}. 
Even with class-based grouping, this necessitates $r$ separate tests. 
Conducting these checks sequentially during estimation is risky, as a late discovery of non-identifiability renders all prior computation wasted. 
To eliminate this overhead, we present a theorem establishing a linear-time check for global identifiability.

\begin{theorem}
    The do-Shapley value $\phi_i$ is identifiable if, and only if, $\forall j \in [d]$, $\nu(\{j\})$ is identifiable.
\end{theorem}

The necessary background knowledge and proof is left for Appendix \ref{appendix:identifiability}.

Consequently, we can execute the ID algorithm on the $d$ singleton coalition queries and, if all are identifiable, the do-Shapley value will be identifiable and we can proceed with its estimation. 
If not, further parametric assumptions, or the inclusion of instrumental variables, will be required to ensure identifiability. 
Regardless, this result prevents practitioners from training estimating do-Shapley values only to learn the outcome was not identifiable to begin with.
Additionally, it reduces the number of calls to the ID algorithm from $r$ to $d$.

\section{Experiments}
\label{sec:experiments}

We evaluate the performance of our exact algorithm and boundary sampling estimators on a diverse set of real-world datasets. Our experiments are designed to investigate three key questions: (1) To what extent does the number of irreducible sets $r$ reduce the complexity compared to the worst-case $2^d$ in real-world dependencies? (2) Do our structure-aware estimators outperform state-of-the-art model-agnostic baselines under fixed query budgets? (3) How does the learned causal structure qualitatively influence feature attribution error?

\textbf{Data.}
We utilize the TALENT benchmark \cite{JMLR:v26:25-0512}, a large-scale repository of tabular datasets.
To enable reliable evaluation against exhaustive baselines, we restrict to datasets whose \textit{post-pruning} dimension (after restricting to the ancestors of the target $Y$) permits exact computation of $\nu(S)$ over all coalitions $S\subseteq[d]$ within our computational budget. For each dataset, we select a set of test instances $\mathbf{x}$ and report errors aggregated across instances and datasets. Additional dataset-level details (including the resulting dimensions after pruning) are provided in the appendix. 

\textbf{SCM Generation.} 
Since real-world datasets lack ground-truth causal graphs, we learn SCMs from data to serve as the ground-truth games $\nu$. For each dataset, we employ the Greedy Relaxed Search Procedure (GRaSP) \citep{lam2022greedy} with a BIC score to learn a Completed Partially Directed Acyclic Graph (CPDAG), which is converted to a DAG greedily. Not all features in a dataset causally influence the target. Following the definition of do-Shapley, we prune the learned graph to the ancestral set of the target variable $Y$; nodes with no directed path to $Y$ have a null Shapley value and are removed. We then fit non-linear structural equations using Gradient Boosting Regressors \cite{Friedman.2001} to model the conditional distributions $\mathbb{E}[X_i \mid \text{Pa}(X_i)]$. This learned SCM acts as our oracle $\nu(S)$.

\textbf{Methods.} 
We evaluate our framework against two state-of-the-art model-agnostic estimators: \texttt{RegressionMSR} \citep{witter2025regressionadjusted} and \texttt{LeverageSHAP} \citep{musco2025provably}. These serve as our structure-agnostic baselines, estimating Shapley values by directly sampling and querying coalitions from the full powerset. To explicitly isolate the gains attributable to our graph-theoretic insights, our proposed estimators \texttt{doRegressionMSR} and \texttt{doLeverageSHAP} are not fundamentally new regression techniques. Rather, they repurpose the \textit{same} estimation machinery used in the baselines. The main difference lies in the data generation process: whereas the baselines query random coalitions, our methods train the estimators on the distinct equivalence classes recovered by the boundary sampler (as described in Algorithm \ref{alg:boundary_estimator}). This allows us to compare ``structure-aware'' versus ``structure-agnostic'' sampling while holding the estimation logic constant.

\paragraph{Lattice Complexity Reduction} The efficiency of our exact algorithm relies on $r \ll 2^d$. In Figure \ref{fig:complexity}, we plot the number of irreducible sets $r$ against the dimension $d$ for 156 datasets from the TALENT benchmark. We observe that for real-world data, the number of irreducible sets $r$ often remains significantly below $2^d$.

\begin{figure}[t!]
    \centering
    \includegraphics[width=0.8\linewidth]{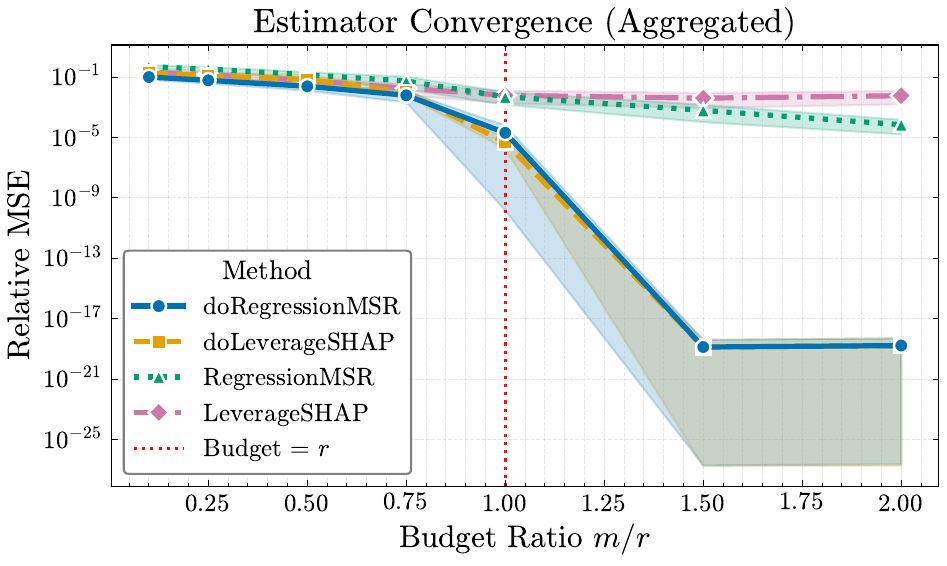}
    \caption{\textbf{Estimator Convergence (Aggregated).} The relative MSE of Shapley value estimates versus the budget ratio $m/r$, aggregated across all datasets. Shaded regions indicate 95\% confidence intervals. Our structure-aware estimators consistently outperform the baseline variants. Notably, as the budget exceeds the number of classes ($m > r$, red dotted line), our error vanishes to machine precision, whereas baselines continue to exhibit variance.}
    \label{fig:convergence}
\end{figure}

\textbf{Estimation Efficiency} We evaluate estimation error by query budget $m$, defined relative to the number of irreducible sets $r$. Figure \ref{fig:convergence} presents Relative MSE for all estimators.

In the sparse budget regime where we cannot fully explore the lattice, \texttt{doRegressionMSR} (blue) demonstrates superior sample efficiency, consistently achieving the lowest error. This advantage becomes increasingly pronounced as the budget approaches $r$. By prioritizing the discovery of distinct equivalence classes via boundary sampling, our method minimizes redundant queries that plague the standard samplers.

A phase transition occurs at $m=r$ (indicated by the red vertical line). Once the budget allows for full lattice exploration, our boundary sampler identifies that all irreducible sets have been found. At this point, the algorithm switches to the exact computation described in Section \ref{sec:exact}. Consequently, the MSE for both \texttt{doRegressionMSR} and \texttt{doLeverageSHAP} drops precipitously to machine precision. In contrast, the structure-agnostic baselines (green and purple) continue to sample coalitions with replacement, exhibiting a slow convergence rate and failing to achieve exactness even with double the necessary budget ($m=2r$).

\section{Generalizations}
\label{sec:generalization}

Our result immediately generalizes in two ways:
\begin{enumerate}[leftmargin=1.5em,itemsep=1pt, topsep=2pt]
    \item[\textbf{a)}] The weighting function $p$ is not restricted to Shapley values and may correspond to alternative probabilistic value concepts, such as the Banzhaf index, weighted Banzhaf values, or Beta Shapley values.
    \item[\textbf{b)}] The framework is not limited to single-feature attributions but naturally extends to interaction values, i.e., changes in the value function with respect to subsets $T$ rather than singletons $i$, as described in Appendix \ref{sec:appx_shapley_interactions}.
\end{enumerate}

While the Shapley values provides a principled framework for attributing value to individuals, they are limited in expressivity. For instance, assessing interactions like synergies or redundancies across multiple features, is not possible based on individual Shapley values. $n$-Shapley values \citep{lundberg2018treeshap,Bord.2023} enrich this explanation by adding interactions up to order $n$. These interactions are based on the Shapley interaction index \citep{Grabisch.1999}, an axiomatic extension of the Shapley value. Similar to the Shapley value, $n$-Shapley values satisfy an extended efficiency axiom \citep{muschalik2024shapiq}, and provide a more granular and expressive explanation of $\nu$ (see~\cref{fig:shapiq_plots}). Their explicit form is given in Appendix \ref{sec:appx_shapley_interactions}. With $n=d$, $n$-Shapley values yield the Möbius transform \citep{Rota.1964}, which provides the exact additive decomposition of $\nu$ \citep{Bord.2023}. Importantly, they satisfy the \emph{linearity} axiom and their weights only depend on the cardinality of the coalition and the interaction, which allows a similar decomposition as \cref{eq:do_shapley}. Consequently, once the irreducible sets are known, we can efficiently compute the Shapley interaction index and $n$-Shapley values, as shown in Appendix \ref{sec:appx_shapley_interactions}.

\begin{figure}[t]
    \centering
    \includegraphics[width=\linewidth]{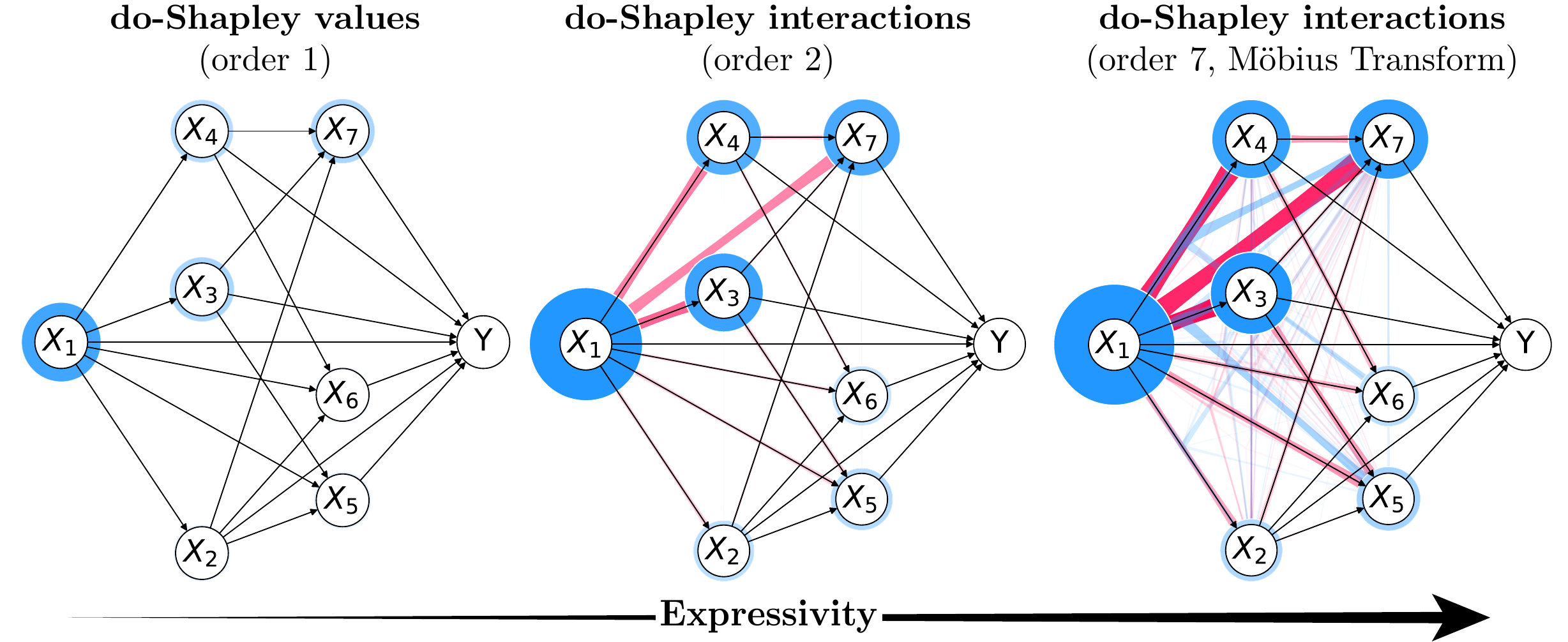}
    \caption{Interactions offer an expressive explanation framework.}
    \label{fig:shapiq_plots}
\end{figure}

\section*{Conclusion}

We address the computational challenges of do-Shapley values by reformulating the estimation via structural equivalence classes. Our proposed algorithm, based on \emph{irreducible sets}, scales with causal complexity ($r$) rather than dimension ($2^d$) and employs \emph{boundary sampling} to reduce redundancy. Additionally, we show that verifying identifiability requires only a linear $O(d)$ check of singletons. By prioritizing structure-aware exploration, this framework facilitates attribution in complex systems, offering a practical step toward scalable causal explainability. Future work could extend these methods to larger graphs, such as in genomics, or investigate sensitivity to graph misspecification.

\clearpage
\section*{Impact Statement}

This paper presents work whose goal is to advance the field of machine learning. There are many potential societal consequences of our work, none of which we feel must be specifically highlighted here.

\bibliography{references}

@article{strumbelj2014explaining,
  title={Explaining prediction models and individual predictions with feature contributions},
  author={{\v{S}}trumbelj, Erik and Kononenko, Igor},
  journal={Knowledge and information systems},
  volume={41},
  pages={647--665},
  year={2014},
  publisher={Springer}
}

@article{chen2023shap_survey,
  title={Algorithms to estimate {S}hapley value feature attributions},
  author={Chen, Hugh and Covert, Ian C and Lundberg, Scott M and Lee, Su-In},
  journal={Nature Machine Intelligence},
  pages={1--12},
  year={2023},
  publisher={Nature Publishing Group UK London}
}

@article{frye2020asymmetric_shap,
  title={Asymmetric {S}hapley values: incorporating causal knowledge into model-agnostic explainability},
  author={Frye, Christopher and Rowat, Colin and Feige, Ilya},
  journal={Advances in Neural Information Processing Systems (NeurIPS)},
  volume={33},
  pages={1229--1239},
  year={2020}
}

@article{heskes2020causal_shap,
  title={Causal {S}hapley values: exploiting causal knowledge to explain individual predictions of complex models},
  author={Heskes, Tom and Sijben, Evi and Bucur, Ioan Gabriel and Claassen, Tom},
  journal={Advances in Neural Information Processing Systems (NeurIPS)},
  volume={33},
  pages={4778--4789},
  year={2020}
}

@inproceedings{jung2022do_shap,
  title={On measuring causal contributions via do-interventions},
  author={Jung, Yonghan and Kasiviswanathan, Shiva and Tian, Jin and Janzing, Dominik and Bl{\"o}baum, Patrick and Bareinboim, Elias},
  booktitle={International Conference on Machine Learning},
  pages={10476--10501},
  year={2022},
  organization={PMLR}
}

@article{lundberg2017shap,
  title={A unified approach to interpreting model predictions},
  author={Lundberg, Scott M and Lee, Su-In},
  journal={Advances in neural information processing systems},
  volume={30},
  year={2017}
}

@article{parafita2022dcg,
  title={Estimand-Agnostic Causal Query Estimation With {Deep} {Causal} {Graphs}},
  author={Parafita, {\'A}lvaro and Vitri{\`a}, Jordi},
  journal={IEEE Access},
  volume={10},
  pages={71370--71386},
  year={2022},
  publisher={IEEE}
}

@incollection{shapley1953shap,
  title={A Value for n-Person Games},
  author={Shapley, LS},
  booktitle={Contributions to the Theory of Games (AM-28), Volume II},
  pages={307--317},
  year={1953},
  publisher={Princeton University Press}
}

@article{lundberg2018treeshap,
  title={Consistent individualized feature attribution for tree ensembles},
  author={Lundberg, Scott M and Erion, Gabriel G and Lee, Su-In},
  journal={arXiv preprint arXiv:1802.03888},
  year={2018}
}

@inproceedings{shpitser2006interventional,
  title={Identification of joint interventional distributions in recursive semi-{Markovian} causal models},
  author={Shpitser, Ilya and Pearl, Judea},
  booktitle={Proceedings of 21st National Conference on Artificial Intelligence (AAAI)},
  address="Boston, MA, USA",
  pages={1219--1226},
  year={2006}
}

@inproceedings{tian2002testable,
  title={On the testable implications of causal models with hidden variables},
  author={Tian, Jin and Pearl, Judea},
  booktitle={Proceedings of the 18th Conference on Uncertainty in Artificial Intelligence (UAI)},
  pages={519--527},
  year={2002},
  address="Edmonton, Canada"
}

@inproceedings{
    parafita2025practical,
    title={Practical do-Shapley Explanations with Estimand-Agnostic Causal Inference},
    author={Parafita, {\'A}lvaro and Garriga, Tomas and Brando, Axel and Cazorla, Francisco J},
    booktitle={The Thirty-ninth Annual Conference on Neural Information Processing Systems},
    year={2025},
    url={https://openreview.net/forum?id=Qabko39AS5}
}

@article{lundberg2020fromlocal,
 title={{From local explanations to global understanding with explainable AI for trees}},
  author={Scott M. Lundberg and Gabriel G. Erion and Hugh Chen and Alex J. DeGrave and Jordan M. Prutkin and Bala Nair and Ronit Katz and Jonathan Himmelfarb and Nisha Bansal and Su{-}In Lee},
  year=2020,
  journal={Nature Machine Intelligence},
  volume=2,
  number=1,
  pages={56--67}
}

@inproceedings{yu2022linear,
  title={Linear tree shap},
  author={Peng Yu and Albert Bifet and Jesse Read and Chao Xu},
  booktitle={Advances in Neural Information Processing Systems (NeurIPS)},
  volume={35},
  year={2022}
}

@inproceedings{zern2023interventional,
  title={Interventional {SHAP} Values and Interaction Values for Piecewise Linear Regression Trees},
  author={Artjom Zern and Klaus Broelemann and Gjergji Kasneci},
  year={2023},
  booktitle={Proceedings of the 37th AAAI Conference on Artificial Intelligence},
  volume={37},
  pages={11164--11173}
}

@inproceedings{muschalik2024beyond,
  title={{Beyond TreeSHAP: Efficient Computation of Any-Order Shapley Interactions for Tree Ensembles}},
  author={Maximilian Muschalik and Fabian Fumagalli and Barbara Hammer and Eyke H{\"{u}}llermeier},
  year=2024,
  booktitle={Proceedings of the 38th AAAI Conference on Artificial Intelligence},
  volume={38},
  pages={14388--14396}
}

@inproceedings{muschalik2025exact,
  title={{Exact Computation of Any-Order Shapley Interactions for Graph Neural Networks}},
  author={Maximilian Muschalik and Fabian Fumagalli and Paolo Frazzetto and Janine Strotherm and Luca Hermes and Alessandro Sperduti and Eyke H{\"u}llermeier and Barbara Hammer},
  booktitle={The Thirteenth International Conference on Learning Representations},
  year={2025}
}

@article{mohammadi2025computing,
  title={Computing Exact Shapley Values in Polynomial Time for Product-Kernel Methods},
  author={Mohammadi, Majid and Chau, Siu Lun and Muandet, Krikamol},
  journal={arXiv preprint arXiv:2505.16516},
  year={2025}
}

@article{mohammadi2025exact,
  title={Exact Shapley Attributions in Quadratic-time for FANOVA Gaussian Processes}, 
  author={Majid Mohammadi and Krikamol Muandet and Ilaria Tiddi and Annette Ten Teije and Siu Lun Chau},
  year={2025},
  journal={arXiv preprint arXiv:2508.14499}
}

@inproceedings{jia2019towards,
  author={Ruoxi Jia and David Dao and Boxin Wang and Frances Ann Hubis and Nick Hynes and Nezihe Merve G{\"{u}}rel and Bo Li and Ce Zhang and Dawn Song and Costas J. Spanos},
  title={Towards Efficient Data Valuation Based on the Shapley Value},
  booktitle={Proceedings of the 22nd International Conference on Artificial Intelligence and Statistics (AIStats)},
  pages        = {1167--1176},
  publisher    = {PMLR},
  year         = {2019}
}

@inproceedings{wang2023privacy,
  title={A Privacy-Friendly Approach to Data Valuation},
  author={Jiachen T. Wang and Yuqing Zhu and Yu-Xiang Wang and Ruoxi Jia and Prateek Mittal},
  booktitle={Advances in Neural Information Processing Systems (NeurIPS)},
  volume={37},
  year={2023},
}

@inproceedings{wang2024efficient,
  author={Jiachen T. Wang and Prateek Mittal and Ruoxi Jia},
  title={Efficient Data Shapley for Weighted Nearest Neighbor Algorithms},
  booktitle={Proceedings of the 27th International Conference on Artificial Intelligence and Statistics (AIStats)},  
  pages={2557--2565},
  publisher={PMLR},
  year={2024},
}

@inproceedings{muschalik2024shapiq,
  title={{shapiq: Shapley Interactions for Machine Learning}},
  author= {Maximilian Muschalik and Hubert Baniecki and Fabian Fumagalli and Patrick Kolpaczki and Barbara Hammer and Eyke H{\"{u}}llermeier},
  year=2024,
  booktitle={Advances in Neural Information Processing Systems (NeurIPS)},
  volume=37,
  pages={130324--130357}
}

@inproceedings{covert2021improving,
  title={Improving kernelshap: Practical shapley value estimation using linear regression},
  author={Covert, Ian and Lee, Su-In},
  booktitle={International conference on artificial intelligence and statistics},
  pages={3457--3465},
  year={2021},
  organization={PMLR}
}

@inproceedings{
butler.2025,
title={Proxy-{SPEX}: Sample-Efficient Interpretability via Sparse Feature Interactions in {LLM}s},
author={Landon Butler and Abhineet Agarwal and Justin Singh Kang and Yigit Efe Erginbas and Bin Yu and Kannan Ramchandran},
booktitle={The Thirty-ninth Annual Conference on Neural Information Processing Systems},
year={2025},
url={https://openreview.net/forum?id=KI8qan2EA7}
}

@inproceedings{
witter2025regressionadjusted,
title={Regression-adjusted Monte Carlo Estimators for Shapley Values and Probabilistic Values},
author={R. Teal Witter and Yurong Liu and Christopher Musco},
booktitle={The Thirty-ninth Annual Conference on Neural Information Processing Systems},
year={2025},
url={https://openreview.net/forum?id=Qabko39AS5}
}

@article{Banzhaf.1964,
	title        = {Weighted voting doesn't work: A mathematical analysis},
	author       = {Banzhaf III, John F},
	year         = 1964,
	journal      = {Rutgers Law Review},
	publisher    = {HeinOnline},
	volume       = 19,
	pages        = 317
}

@inproceedings{Bord.2023,
	title        = {{From Shapley Values to Generalized Additive Models and back}},
	author       = {Sebastian Bordt and Ulrike von Luxburg},
	year         = 2023,
	booktitle    = {Proceedings of the International Conference on Artificial Intelligence and Statistics {(AISTATS)}},
	pages        = {709--745}
}

@article{Castro.2009,
	title        = {{Polynomial calculation of the Shapley value based on sampling}},
	author       = {Javier Castro and Daniel G{\'{o}}mez and Juan Tejada},
	year         = 2009,
	journal      = {Computers \& Operations Research},
	volume       = 36,
	number       = 5,
	pages        = {1726--1730},
	doi          = {10.1016/j.cor.2008.04.004}
}

@inproceedings{Rozemberczki.2022,
	title        = {The Shapley Value in Machine Learning},
	author       = {Benedek Rozemberczki and Lauren Watson and P{\'{e}}ter Bayer and Hao{-}Tsung Yang and Oliver Kiss and Sebastian Nilsson and Rik Sarkar},
	year         = 2022,
	booktitle    = {Proceedings of International Joint Conference on Artificial Intelligence {(IJCAI)}},
	pages        = {5572--5579}
}

@article{Dubey.1981,
	title        = {Value Theory Without Efficiency},
	author       = {Pradeep Dubey and Abraham Neyman and Robert James Weber},
	year         = 1981,
	journal      = {Mathematics of Operations Research},
	volume       = 6,
	number       = 1,
	pages        = {122--128},
	doi          = {10.1287/MOOR.6.1.122}
}

@article{Friedman.2001,
	title        = {{Greedy function approximation: A gradient boosting machine.}},
	author       = {Jerome H. Friedman},
	year         = 2001,
	journal      = {The Annals of Statistics},
	publisher    = {Institute of Mathematical Statistics},
	volume       = 29,
	number       = 5,
	pages        = {1189--1232},
	doi          = {10.1214/aos/1013203451}
}

@article{Fujimoto.2006,
	title        = {{Axiomatic characterizations of probabilistic and cardinal-probabilistic interaction indices}},
	author       = {Katsushige Fujimoto and Ivan Kojadinovic and Jean{-}Luc Marichal},
	year         = 2006,
	journal      = {Games and Economic Behavior},
	volume       = 55,
	number       = 1,
	pages        = {72--99},
	doi          = {10.1016/j.geb.2005.03.002}
}

@inproceedings{Fumagalli.2023,
	title        = {{SHAP-IQ: Unified Approximation of any-order Shapley Interactions}},
	author       = {Fabian Fumagalli and Maximilian Muschalik and Patrick Kolpaczki and Eyke H{\"u}llermeier and Barbara Hammer},
	year         = 2023,
	booktitle    = {Proceedings of Advances in Neural Information Processing Systems {(NeurIPS)}},
}

@article{Grabisch.1997,
	title        = {k-order additive discrete fuzzy measures and their representation},
	author       = {Grabisch, Michel},
	year         = 1997,
	journal      = {Fuzzy Sets and Systems},
	volume       = 92,
	number       = 2,
	pages        = {167–189},
	doi          = {10.1016/S0165-0114(97)00168-1}
}

@article{Grabisch.1999,
	title        = {{An axiomatic approach to the concept of interaction among players in cooperative games}},
	author       = {Michel Grabisch and Marc Roubens},
	year         = 1999,
	journal      = {International Journal of Game Theory},
	volume       = 28,
	number       = 4,
	pages        = {547--565},
	doi          = {10.1007/s001820050125}
}

@inproceedings{Kolpaczki.2024a,
	title        = {Approximating the Shapley Value without Marginal Contributions},
	author       = {Patrick Kolpaczki and Viktor Bengs and Maximilian Muschalik and Eyke H{\"{u}}llermeier},
	year         = 2024,
	booktitle    = {Proceeedings of the {AAAI} Conference on Artificial Intelligence {(AAAI)}},
	pages        = {13246--13255}
}

@inproceedings{Kolpaczki.2024b,
	title        = {{SVARM-IQ:} Efficient Approximation of Any-order Shapley Interactions through Stratification},
	author       = {Patrick Kolpaczki and Maximilian Muschalik and Fabian Fumagalli and Barbara Hammer and Eyke H{\"{u}}llermeier},
	year         = 2024,
	booktitle    = {Proceedings of the International Conference on Artificial Intelligence and Statistics {(AISTATS)}},
	pages        = {3520--3528},
}

@incollection{Rota.1964,
	title        = {On the foundations of combinatorial theory: I. Theory of M{\"o}bius functions},
	author       = {Rota, Gian-Carlo},
	year         = 1964,
	booktitle    = {Classic Papers in Combinatorics},
	publisher    = {Springer},
	pages        = {332--360}
}

@inproceedings{Zern.2023,
	title        = {Interventional {SHAP} Values and Interaction Values for Piecewise Linear Regression Trees},
	author       = {Artjom Zern and Klaus Broelemann and Gjergji Kasneci},
	year         = 2023,
	booktitle    = {Proceeedings of the {AAAI} Conference on Artificial Intelligence {(AAAI)}},
	pages        = {11164--11173}
}

@inproceedings{Muschalik.2024a,
  title     = {{shapiq: Shapley Interactions for Machine Learning}},
  author    = {Maximilian Muschalik and Hubert Baniecki and Fabian Fumagalli   and Patrick Kolpaczki and Barbara Hammer and Eyke H\"{u}llermeier},
	booktitle    = {Proceedings of Advances in Neural Information Processing Systems {(NeurIPS)}},
  year      = {2024},
 pages = {130324--130357},
}

@inproceedings{Wang.2023,
  author       = {Jiachen T. Wang and
                  Ruoxi Jia},
  editor       = {Francisco J. R. Ruiz and
                  Jennifer G. Dy and
                  Jan{-}Willem van de Meent},
  title        = {Data Banzhaf: {A} Robust Data Valuation Framework for Machine Learning},
  booktitle    = {Proceedings of the International Conference on Artificial Intelligence and Statistics ({AISTATS})},
  pages        = {6388--6421},
  year         = {2023},
}

@InProceedings{Kwon.2022b,
  title = 	 { Beta Shapley: a Unified and Noise-reduced Data Valuation Framework for Machine Learning },
  author =       {Kwon, Yongchan and Zou, James},
  booktitle = 	 {{Proceedings of the International Conference on Artificial Intelligence and Statistics {(AISTATS)}}},
  pages = 	 {8780--8802},
  year = 	 {2022},
}

@article{lundberg2018consistent,
  title={Consistent individualized feature attribution for tree ensembles},
  author={Lundberg, Scott M and Erion, Gabriel G and Lee, Su-In},
  journal={arXiv preprint arXiv:1802.03888},
  year={2018}
}

@inproceedings{lam2022greedy,
  title={Greedy relaxations of the sparsest permutation algorithm},
  author={Lam, Wai-Yin and Andrews, Bryan and Ramsey, Joseph},
  booktitle={Uncertainty in Artificial Intelligence},
  pages={1052--1062},
  year={2022},
  organization={PMLR}
}

@inproceedings{musco2025provably,
  author       = {Christopher Musco and
                  R. Teal Witter},
  title        = {{Provably Accurate Shapley Value Estimation via Leverage Score Sampling}},
  booktitle = {Proccedings of the International Conference on Learning Representations (ICLR)},
  year         = {2025},
}

@inproceedings{fumagalli2024kernelshapiq,
	title        = {{KernelSHAP-IQ: Weighted Least Square Optimization for Shapley Interactions}},
	author       = {Fabian Fumagalli and Maximilian Muschalik and Patrick Kolpaczki and Eyke H{\"{u}}llermeier and Barbara Hammer},
	year         = 2024,
	booktitle    = {Proceedings of the International Conference on Machine Learning ({ICML})},
    pages =  {14308--14342},
}

@article{tsai2022faith,
	title        = {{Faith-Shap: The Faithful Shapley Interaction Index}},
	author       = {Che{-}Ping Tsai and Chih{-}Kuan Yeh and Pradeep Ravikumar},
	year         = 2023,
	journal      = {Journal of Machine Learning Research},
	volume       = 24,
	number       = 94,
	pages        = {1--42}
}

@article{holland1986statistics,
  title={Statistics and causal inference},
  author={Holland, Paul W},
  journal={Journal of the American statistical Association},
  volume={81},
  number={396},
  pages={945--960},
  year={1986},
  publisher={Taylor \& Francis}
}

@article{rubin1974estimating,
  title={Estimating causal effects of treatments in randomized and nonrandomized studies.},
  author={Rubin, Donald B},
  journal={Journal of educational Psychology},
  volume={66},
  number={5},
  pages={688},
  year={1974},
  publisher={American Psychological Association}
}

@book{peters2017elements,
  title={Elements of causal inference: foundations and learning algorithms},
  author={Peters, Jonas and Janzing, Dominik and Sch{\"o}lkopf, Bernhard},
  year={2017},
  publisher={The MIT press}
}

@article{bareinboim2016causal,
  title={Causal inference and the data-fusion problem},
  author={Bareinboim, Elias and Pearl, Judea},
  journal={Proceedings of the National Academy of Sciences},
  volume={113},
  number={27},
  pages={7345--7352},
  year={2016},
  publisher={National Academy of Sciences}
}

@article{li2024robust,
  title={Robust data valuation with weighted banzhaf values},
  author={Li, Weida and Yu, Yaoliang},
  journal={Advances in Neural Information Processing Systems},
  volume={36},
  year={2024}
}

@book{pearl2009causality,
  title={Causality: Models, Reasoning and Inference},
  author={Pearl, Judea},
  year={2009},
  publisher={Cambridge University Press},
  edition="Second"
}

@article{JMLR:v26:25-0512,
  author  = {Si-Yang Liu and
			 Hao-Run Cai and
 			 Qi-Le Zhou and
			 Huai-Hong Yin and
			 Tao Zhou and
			 Jun-Peng Jiang and
			 Han-Jia Ye},
  title   = {Talent: A Tabular Analytics and Learning Toolbox},
  journal = {Journal of Machine Learning Research},
  year    = {2025},
  volume  = {26},
  number  = {226},
  pages   = {1--16},
  url     = {http://jmlr.org/papers/v26/25-0512.html}
}
\bibliographystyle{icml2026}

\newpage
\appendix
\onecolumn

\clearpage
\section{Delayed Proofs}
\label{app:delayedproof}

\noindent \textbf{Lemma \ref{lemma:closed_plus}} \textit{
    Let $\bar{S} \subset [d]$ be a closed set with \smallest{} $\underline{S}$.
    Then
    \begin{enumerate}
        \item For all $j \in \underline{S}$, $\bar{S} \setminus \{j\}$ is closed.
        \item If $\bar{S} \neq [d]$, there exists $j \in [d] \setminus \bar{S}$ so that $\bar{S} \cup \{j\}$ is closed and $j$ is in the \smallest{} of $\bar{S} \cup \{j\}$.
    \end{enumerate}
}

\begin{proof}[Proof of Lemma \ref{lemma:closed_plus}]
    We will first show that $\bar{S} \setminus \{j\}$ is closed for all $j \in \underline{S}$.
    To do so, it suffices to show that, for all nodes \textit{not} in $\bar{S} \setminus \{j\}$, there is a directed path to $Y$ that does not intersect $\bar{S} \setminus \{j\}$.
    This is clearly true for all nodes not in $\bar{S}$ since $\bar{S}$ is itself closed.
    It must also be true for $j$ since $j \in \underline{S}$, i.e., there is a directed path from $j$ to $Y$ that does not intersect $S$.
    The first statement follows.

    Next, we will show that for $S \neq [d]$, there exists $j \in [d] \setminus \bar{S}$ so that $\bar{S} \cup \{j\}$ is closed.
    Let $T = [d] \setminus \bar{S}$.
    Since $G$ is a finite DAG, there is a topological order on the nodes.
    Consider a node $j \in T$ with no ancestors in $T$, either because they are all in $\bar{S}$, or it has no ancestors.
    To prove that $\bar{S} \cup \{j\}$ is closed,
    it suffices to show that, for all nodes \textit{not} in $\bar{S} \cup \{j\}$, there is a directed path to $Y$ that does not intersect $\bar{S} \cup \{j\}$.
    Since $\bar{S}$ is closed, all of the nodes in $T \setminus \{j\}$ have paths to $Y$ that do not intersect $\bar{S}$.
    Note that $j$ cannot be in any of these paths, since $j$ has no ancestors in $T$, hence all nodes in $T \setminus \{j\}$ must actually have paths that do not intersect $\bar{S}$ or $j$.
    Furthermore, since $j \in T$, it must have a directed path to $Y$ that does not intersect $\bar{S}$.
    Therefore, $j$ must be in the \smallest{} of $\bar{S} \cup \{j\}$.
    The second statement follows.
\end{proof}

\noindent \textbf{Proposition \ref{prop:boundary_sampler_time}}. 
\textit{
Algorithm \ref{alg:boundary_sampler} runs in $O(m \cdot d(d+e))$ time, where $m$ is the query budget, $d$ is the number of features, and $e$ is the number of edges in the graph.
}

\begin{proof}[Proof of Proposition \ref{prop:boundary_sampler_time}]
The algorithm performs exactly $m$ iterations of the main \texttt{while} loop. In each iteration, we process one class $c$. The cost of processing a class is dominated by generating its neighbors and invoking \texttt{FindClass} for each. A class with closure $\bar{S}$ and basis $\underline{S}$ has $|\underline{S}|$ lower neighbors and $d - |\bar{S}|$ upper neighbors; thus, the total number of neighbors is bounded by $d$. For each neighbor, we execute \texttt{FindClass}, which requires a graph traversal taking $O(d+e)$ time. Therefore, the work per iteration is $O(d(d+e))$.
Over $m$ iterations, the total time complexity is $O(m \cdot d(d+e))$.
This is linear in the budget $m$, ensuring the method is scalable for anytime estimation.
\end{proof}

\begin{algorithm}[tb]\caption{\texttt{BoundarySampler}}
\label{alg:boundary_sampler}
\begin{algorithmic}
    \STATE {\bfseries Input:} Budget $m$, Graph $G$
    \STATE {\bfseries Output:} Sampled classes $\mathcal{C}$, flag \texttt{allSampled}
    \STATE $\mathcal{C} \gets \emptyset$ \hfill {\color{gray!90} $\triangleright$ Sampled classes}
    \STATE $\mathcal{Q} \gets \emptyset$ \hfill {\color{gray!90} $\triangleright$ Queue mapping classes to weights}
    \STATE $\mathcal{C}_{\text{seen}} \gets \emptyset$
    \STATE {\color{gray!90} $\triangleright$ Helper to calculate weight and enqueue}\FUNCTION{\texttt{Enqueue}($S$)}\STATE $c \gets \texttt{FindClass}(S, G)$ \hfill {\color{gray!90} $\triangleright$ Get basis and closure}\IF{$c \in \mathcal{C}_{\text{seen}}$}
    \STATE \textbf{return}\ENDIF\STATE Add $c$ to $\mathcal{C}_{\text{seen}}$\STATE $\mathcal{Q}[c] \gets \mathbb{E}_i[|w_i(c)|]+ \epsilon$ \hfill {\color{gray!90} $\triangleright$ Add $\epsilon$ to ensure valid dist.}\ENDFUNCTION
    \STATE {\color{gray!90} $\triangleright$ Phase 1: Warm-start}   
    \FOR{$\ell = 1$ \textbf{to} $d$}
    \STATE Sample random set $S \subset [d]$ where $|S|=\ell$
    \STATE $\texttt{Enqueue}(S)$
    \ENDFOR
    \STATE {\color{gray!90} $\triangleright$ Phase 2: Weighted Graph Traversal}
    \WHILE{$|\mathcal{C}| < m$ \textbf{and} $\mathcal{Q} \neq \emptyset$}
    \STATE Sample $c$ from $\mathcal{Q}$ with prob $\propto \mathcal{Q}[c]$
    \STATE Remove $c$ from $\mathcal{Q}$ and add to $\mathcal{C}$
    \STATE $\underline{S}, \bar{S} \gets$ \smallest{} and \largest{} of $c$
    \STATE $N_{\text{below}} \gets \{ \bar{S} \setminus \{j\} \mid j \in \underline{S} \}$
    \STATE $N_{\text{above}} \gets \{ \bar{S} \cup \{j\} \mid j \notin \bar{S} \}$
    \FOR{candidate set $S' \in N_{\text{below}} \cup N_{\text{above}}$}
        \STATE $\texttt{Enqueue}(S')$
    \ENDFOR
    \ENDWHILE
\STATE {\bfseries return} $\mathcal{C}$, $\mathcal{Q} == \emptyset$
\end{algorithmic}
\end{algorithm}

\clearpage
\section{Background on Structural Causal Models}
\label{app:backgroundscm}

A \textbf{Causal graph} is usually described as a Directed Acyclic Graph (DAG) $G$, where every node represents a measured random variable and every directed edge represents a relationship cause $\rightarrow$ effect. These graphs often include dashed bidirected edges between pairs of nodes ($X \leftrightarrow Y$) as a shorthand for the existence of an unobserved latent variable $U$ that acts as a confounder between them ($X \leftarrow U \rightarrow Y$). Additionally, it is assumed that every measured node $X$ has a latent exogenous noise variable, denoted $E_X$, with the associated edge $E_X \rightarrow X$.

A \textbf{Structural Causal Model} (SCM) is a probabilistic model based on such a causal graph $G$ with a probability distribution for all latent nodes (i.e., all confounders and $E_X$) and, for each measured node $X$, functions $f_X$ such that $X := f_X(Pa_G(X))$, taking the values of all parents of $X$ in $G$ including latent variables. From this model, a probability distribution over the measured variables $V$ emerges, $P(V)$, as well as the intervened model $M_{do(S=s)}$, where the $do$ operator conveys an intervention on all nodes $X \in S$ replacing their functions $f_X$ by the assignment $x := s_X$. This effectively removes all incoming edges to the intervened nodes from the graph, and results in a new probability distribution for the intervened model, $P(V \mid do(S=s))$, also denoted $P_s(V)$ or, for arbitrary intervention values, $P_S(V)$.

Given a dataset and its assumed underlying causal structure, we can train an SCM following that graph to learn the distribution of the dataset. Afterwards, one can employ procedures on the SCM to estimate \textbf{causal queries} of the form $P_X(Y)$. Additionally, if the causal query is non-parametrically \textit{identifiable} (more details in \cref{sec:identifiable}), these estimations resulting from the SCM are necessarily equivalent to what the true data generating process would return if we had access to it. Therefore, we can employ these trained SCMs to estimate the do-SHAP value functions $\nu(S)$ as long as they are identifiable. For more details about this approach, please refer to \cite{parafita2022dcg}.

\clearpage
\section{Simulated Sampling from Irreducible Sets}
\label{app:simulatedsampling}

In this appendix, we detail the \texttt{SimulatedSampler}, the engine behind Step 2B of the \texttt{doEstimator}. This component allows us to sample coalitions efficiently from the specific sub-lattice defined by the discovered equivalence classes.

\subsection{Algorithm and Explanation}

The challenge in simulating samples from the discovered classes is that the union of these classes does not form a simple structure (like a full powerset). A naive rejection sampling approach—sampling from the full powerset and keeping only those in $\mathcal{C}$—would be inefficient if the volume of $\mathcal{C}$ is small relative to $2^d$.

Instead, our sampler (Algorithm \ref{alg:simulated_sampler}) adopts a constructive approach:
\begin{itemize}
    \item Normalization: We first calculate the total ``volume" of available coalitions within the known classes. For each class $c$ with basis $\underline{S}$ and closure $\bar{S}$, the number of subsets of size $s$ contained in $c$ is given by $\binom{|\bar{S}| - |\underline{S}|}{s - |\underline{S}|}$.
    \item Scale Calibration: Standard Shapley weights are defined for the entire powerset. To sample validly from our restricted support, we solve for a scaling factor $\gamma$ such that the expected number of samples drawn matches our target budget $B_{\text{sim}}$. This is achieved via binary search (Lines 7-15).
    \item Stratified Generation: We iterate through each class $c \in \mathcal{C}$ and each valid subset size $s$. For each size, we compute the expected number of samples $N_{c,s}$. We use probabilistic rounding to convert this expectation into an integer count, and then generate that many unique subsets from class $c$ using a combinatorial number system (Lines 22-26).
\end{itemize}

\begin{algorithm}[h!]\caption{\texttt{SimulatedSampler}}\label{alg:simulated_sampler}\begin{algorithmic}[1]\STATE {\bfseries Input:} Discovered classes $\mathcal{C}$, Simulation Budget $B_{\text{sim}}$, Weights $w$ (per size)\STATE {\bfseries Output:} Dataset of coalitions $\mathcal{D} = \{(S_k, \nu(c_k), p_k)\}$\STATE $\mathcal{D} \gets \emptyset$\STATE {\color{gray!90} $\triangleright$ Step 1: Count available coalitions per size across all classes}\STATE $N_{\text{avail}}[s] \gets 0$ for $s \in 0\dots d$\FOR{each class $c \in \mathcal{C}$}\STATE Let $n_{\text{free}} = |\bar{S}| - |\underline{S}|$\FOR{$j = 0$ \textbf{to} $n_{\text{free}}$}\STATE $N_{\text{avail}}[|\underline{S}| + j] \gets N_{\text{avail}}[|\underline{S}| + j] + \binom{n_{\text{free}}}{j}$\ENDFOR\ENDFOR\STATE {\color{gray!90} $\triangleright$ Step 2: Calibrate sampling scale $\gamma$}\STATE Define $E[\text{samples}](\gamma) = \sum_{s=0}^d N_{\text{avail}}[s] \cdot \min\left( \gamma \frac{w_s}{\binom{d}{s}}, 1 \right)$\STATE Find $\gamma^*$ via binary search such that $E[\text{samples}](\gamma^*) \approx B_{\text{sim}}$\STATE {\color{gray!90} $\triangleright$ Step 3: Constructive Sampling}\FOR{each class $c \in \mathcal{C}$}\STATE Let $I_{\text{free}}$ be indices in $\bar{S} \setminus \underline{S}$\FOR{$j = 0$ \textbf{to} $|I_{\text{free}}|$}\STATE Size $s \gets |\underline{S}| + j$\STATE Prob $p \gets \min\left( \gamma^* \frac{w_s}{\binom{d}{s}}, 1 \right)$\STATE Count $K \gets \binom{|I_{\text{free}}|}{j}$\STATE Expected count $\mu \gets K \cdot p$\STATE $N_{\text{draw}} \gets \lfloor \mu \rfloor + \text{Bernoulli}(\mu - \lfloor \mu \rfloor)$ \hfill {\color{gray!90} $\triangleright$ Probabilistic rounding}    \IF{$N_{\text{draw}} > 0$}
        \STATE Generate $N_{\text{draw}}$ unique combinations $C_{\text{sub}} \subseteq I_{\text{free}}$ of size $j$
        \FOR{each combination $\sigma \in C_{\text{sub}}$}
            \STATE $S \gets \underline{S} \cup \sigma$
            \STATE Add $(S, \nu(c), p)$ to $\mathcal{D}$
        \ENDFOR
    \ENDIF
\ENDFOR
\ENDFOR\STATE \textbf{return} $\mathcal{D}$\end{algorithmic}\end{algorithm}

\subsection{Runtime Analysis}

\begin{proposition}
The \texttt{SimulatedSampler} generates $N_{\text{sim}}$ samples in $O(N_{\text{sim}} \cdot d + |\mathcal{C}| \cdot d)$ time.
\end{proposition}

\begin{proof}
The algorithm consists of three main parts:
\begin{enumerate}
    \item Counting ($O(|\mathcal{C}| \cdot d)$): We iterate over each class once. For each class, we perform a loop over its "free" size range, which is at most $d$. The binomial coefficient calculations can be done in $O(1)$ with pre-computation.
    \item Calibration ($O(d \cdot \log(1/\epsilon))$): The binary search evaluates the expected sample sum a constant number of times. Each evaluation sums over $d$ sizes.
    \item Generation ($O(N_{\text{sim}} \cdot d)$): The outer loops iterate over classes and sizes, but the inner generation logic (Lines 24-28) executes exactly $N_{\text{sim}}$ times in total (by definition of the calibrated budget). Generating a combination of size $k$ using the combinatorial number system or direct sampling takes $O(d)$. Thus, the generation phase scales linearly with the number of output samples.
\end{enumerate}
Dominating terms yield a total complexity of $O(N_{\text{sim}} \cdot d + |\mathcal{C}| \cdot d)$, which is highly efficient given that no SCM evaluations are performed.
\end{proof}

\clearpage
\section{Stratified Sampling}
\label{app:stratified}

Ideally, to estimate the Shapley value for player $i$, we would sample each class $c$ proportional to the magnitude of its weight $w_i(c)$. However, the underlying class structure (the mapping of sets to values) is unknown prior to sampling, making the direct computation of $w_i(c)$ impossible.

To address this, we consider an adaptive sampling scheme based on the natural distribution suggested by the Shapley weights. We sample a set $S$ containing $i$ with probability $p_{\ell-1}$, and a set $S$ excluding $i$ with probability $p_\ell$. Note that while these probabilities are derived from the Shapley weights, they do not perfectly correspond to the final importance weights because the contribution of $i$ cancels to zero when $i$ is effectively a ``null'' player (i.e., in the \largest{} but not in the \smallest{}).

The proposed \textbf{do-Good estimator} samples according to these distributions \textit{without replacement}. Sampling classes without replacement is crucial for the estimator to achieve exactness when the sampling budget covers the effective support of the game ($m \geq r$). To achieve this, we maintain weighted counts of the remaining ``mass'' of the distributions for each player $i$. We define the remaining mass for sets including $i$ ($\mu^{(+)}$) and excluding $i$ ($\mu^{(-)}$) as:
\begin{align}
    \mu_{\ell,i}^{(+)} &= p_{\ell-1} \sum_{S: |S|=\ell, i \in S} \mathbbm{1}[\text{$S$ not seen}], \\
    \mu_{\ell,i}^{(-)} &= p_{\ell} \sum_{S: |S|=\ell, i \notin S} \mathbbm{1}[\text{$S$ not seen}].
\end{align}
Initially, before any classes have been sampled, these initialize to the full binomial sums: $\mu_{\ell,i}^{(+)} = p_{\ell-1} \binom{d-1}{\ell-1}$ and $\mu_{\ell,i}^{(-)} = p_{\ell} \binom{d-1}{\ell}$.

\subsection*{Sampling Procedure}
The sampling process proceeds hierarchically to determine whether to sample proportional to $\mu_{\ell,i}^+$ or $\mu_{\ell,i}^-$:
\begin{enumerate}
    \item We first select an index $i$ uniformly from $[d]$.
    \item We sample an inclusion indicator $z \in \{-,+\}$, where
    \begin{align*}
        \Pr(z =+) = \frac{\sum_{\ell=1}^d \mu_{\ell,i}^{(+)}}{\sum_{\ell=1}^d \mu_{\ell,i}^{(+)} + \sum_{\ell=0}^{d-1} \mu_{\ell,i}^{(-)}}.
    \end{align*}
    \item We sample a set size $\ell$ with probability proportional to the remaining mass $\mu_{\ell,i}^{(z)}$.
\end{enumerate}

Once $\ell$ and $z$ are determined, we must sample a specific set $S$ of size $\ell$ (containing $i$ if and only if $z=+$) \textit{uniformly} from the collection of all such sets that belong to currently unseen classes. This step is non-trivial; it depends on the number of valid completions available in the unseen space. As detailed in Algorithm~\ref{alg:sample_unseen}, computing these counts requires a linear pass through the history of previously discovered classes.

We analyze the computational complexity of the \texttt{ClassSampler} (Algorithm~\ref{alg:do_good}). The runtime is dominated by the requirement to sample uniformly from unseen sets, which necessitates checking consistency against all previously discovered classes.

\begin{proposition}[Stratified Sampling Complexity]
\label{prop:stratifiedcomplexity}
    Let $m$ be the sampling budget (number of iterations) and $d$ be the number of elements (dimension). Assuming set operations (union, intersection, subset checks) take $O(d)$ time, the total time complexity of the \texttt{ClassSampler} is $O(m^2 d^2)$.
\end{proposition}

Because the algorithm is quadratic in $m$, we turn to the \textit{boundary sampling} method described in Section \ref{sec:estimate}.

\begin{proof}[Proof of Proposition \ref{prop:stratifiedcomplexity}]
    The analysis proceeds by examining the cost of the helper functions from the bottom up.
    
    \textbf{1. Cost of \texttt{CountSeen}:}
    This function iterates through the set of seen classes $\mathcal{C}$. In the $k$-th iteration of the main loop, $|\mathcal{C}| \leq k$. Inside the loop, we perform standard set operations (checking $\underline{S} \subseteq S \cup R$, etc.).
    \begin{equation}
        T_{\text{count}}(k) = O(|\mathcal{C}| \cdot d) = O(k \cdot d).
    \end{equation}
    
    \textbf{2. Cost of \texttt{SampleUnseenBySize}:}
    This function constructs a set $S$ of size $\ell$ element-by-element. The \texttt{while} loop runs at most $\ell \leq d$ times. In every iteration, it calls \texttt{CountSeen} twice to calculate $N_{\text{in}}$ and $N_{\text{out}}$.
    \begin{align}
        T_{\text{sample}}(k) &= \sum_{j=1}^{\ell} 2 \cdot T_{\text{count}}(k) \nonumber \\
        &= O(d \cdot (k \cdot d)) = O(k d^2).
    \end{align}
    
    \textbf{3. Total Cost of \texttt{ClassSampler}:}
    The main algorithm runs for $m$ iterations. In iteration $k$, it calls \texttt{SampleUnseenBySize}, performs a graph lookup (\texttt{FindClass}), and updates weights.
    
    The function \texttt{FindClass} runs in $O(d+e)$ time where $e$ is the number of edges in the graph. Since $e \leq d^2$, the dominant cost remains the combinatorial counting step. The update of $\mu$ values takes $O(d^2)$ but is repeated only $m$ times.
    
    Summing over $m$ iterations:
    \begin{align}
        T_{\text{total}} &= \sum_{k=1}^{m} \left( T_{\text{sample}}(k) + O(d^2) \right) \nonumber \\
        &= \sum_{k=1}^{m} O(k d^2) \nonumber \\
        &= O(d^2) \sum_{k=1}^{m} k \approx O(d^2 m^2).
    \end{align}
    Thus, the total complexity is quadratic in both the dimension and the sample budget.
\end{proof}

\begin{algorithm}[tb]
\caption{\texttt{ClassSampler}}
\label{alg:do_good}
\begin{algorithmic}
    \STATE {\bfseries Input:} Number of elements $d$, budget $m$, value function $\nu$, graph $G$
    \STATE {\bfseries Output:} Set of seen classes $\mathcal{C}$
    \STATE Initialize seen classes $\mathcal{C} \gets \emptyset$
    \STATE Set $\mu_{\ell,i}^{(1)} \gets p_{\ell-1} \binom{d-1}{\ell-1}$ for all $\ell \in \{1,\ldots, d\}$ and $i \in [d]$
    \STATE Set $\mu_{\ell,i}^{(0)} \gets p_{\ell} \binom{d-1}{\ell}$ for all $\ell \in \{0,\ldots, d-1\}$ and $i \in [d]$
    \FOR{$\text{idx} = 1$ \textbf{to} $m$}
        \STATE {\color{gray!90} $\triangleright$ Sample set from unseen classes}
        \STATE Sample $i \sim \text{Uniform}([d])$
        \STATE Sample $z \sim \text{Bernoulli}\left(\frac{\sum_{\ell=1}^d \mu_{\ell,i}^{(1)}}{\sum_{\ell=1}^d \mu_{\ell,i}^{(1)} + \sum_{\ell=0}^{d-1} \mu_{\ell,i}^{(0)}}\right)$
        \STATE Sample $\ell \propto \mu_{\ell,i}^{(z)}$
        \STATE $S \gets \{i\}$ \textbf{ if } $z=1$ \textbf{ else } $\emptyset$
        \STATE $S \gets \texttt{SampleUnseenBySize}(\ell, S, \mathcal{C})$
        \STATE $\underline{S}, \bar{S} \gets \texttt{FindClass}(S, G)$
        \hfill {\color{gray!90}$\triangleright$ Find class $c$}
        \STATE Update all $\mu_{\ell,i}^{(1)}$ and $\mu_{\ell,i}^{(0)}$ \hfill {\color{gray!90}$\triangleright$ Constant time per $\ell$ and $i$}
        \STATE Add class $c$ to $\mathcal{C}$
    \ENDFOR
\STATE {\bfseries return} $\mathcal{C}$
\end{algorithmic}
\end{algorithm}

\begin{algorithm}[tb]
  \caption{\texttt{SampleUnseenBySize}}
  \label{alg:sample_unseen}
  \begin{algorithmic}
    \STATE {\bfseries Input:} Size $\ell$, Initial set $S$ (i.e., $\{i\}$ or $\emptyset$), seen classes $\mathcal{C}$
    \STATE {\bfseries Output:} A set $S$ of size $\ell$ sampled uniformly from unseen sets
    \STATE $R \gets [d] \setminus S$ \hfill $\triangleright$ Available candidates    
    \STATE $\mathcal{C}' \gets \{c : |\underline{S}| \leq \ell \leq |\bar{S}| \}$ \hfill $\triangleright$ Relevant classes
    \WHILE{$|S| < \ell$}
        \STATE Pick $j$ uniformly from $R$        
        \STATE {\color{gray!90} $\triangleright$ Completions in seen classes with and without $j$}
        \STATE $N_{\text{in}}, \mathcal{C}'_\text{in} \gets \texttt{CountSeen}(\ell, S \cup \{j\}, R \setminus \{j\}, \mathcal{C}')$
        \STATE $N_{\text{out}}, \mathcal{C}'_\text{out} \gets \texttt{CountSeen}(\ell, S, R \setminus \{j\}, \mathcal{C}')$
        
        \STATE {\color{gray!90} $\triangleright$ Calculate count of \textit{unseen} completions}
        \STATE $U_{\text{in}} \gets \binom{|R|-1}{\ell - |S| - 1} - N_{\text{in}}$
        \STATE $U_{\text{out}} \gets \binom{|R|-1}{\ell - |S|} - N_{\text{out}}$
        
        \STATE {\color{gray!90} $\triangleright$ Pick $j$ proportional to unseen completions}
        \STATE Sample $b \sim \text{Bernoulli}\left( \frac{U_{\text{in}}}{U_{\text{in}} + U_{\text{out}}} \right)$
        \STATE $\mathcal{C'} \gets \mathcal{C}'_\text{out}$
        \IF{$b = 1$}
            \STATE $\mathcal{C'} \gets \mathcal{C}'_\text{in}$
            \STATE $S \gets S \cup \{j\}$
        \ENDIF
        \STATE $R \gets R \setminus \{j\}$
    \ENDWHILE
    \STATE {\bfseries return} $S$
  \end{algorithmic}
\end{algorithm}

\begin{algorithm}[tb]
  \caption{\texttt{CountSeen}}
  \label{alg:count_seen}
  \begin{algorithmic}
    \STATE {\bfseries Input:} Size $\ell$, current set $S$, remaining set $R$, seen classes $\mathcal{C}$
    \STATE {\bfseries Output:} Number of completions $N$, applicable classes $\mathcal{C}'$
    \STATE $N \gets 0$
    \STATE $\mathcal{C'} \gets \emptyset$ \hfill {\color{gray!90} $\triangleright$ Initialize applicable classes}
    \FOR{class $c \in \mathcal{C}$}
        \STATE Let $\underline{S}$ be the basis and $\bar{S}$ be the closure of $c$
        \STATE {\color{gray!90}$\triangleright$ Check if $S$ and $R$ are consistent with class}
        \STATE \textbf{if} $\underline{S} \not\subseteq (S \cup R)$ or $S \not\subseteq \bar{S}$ \textbf{then continue}         
        \STATE Add $c$ to $\mathcal{C}'$ \hfill {\color{gray!90} $\triangleright$ Applicable class}
        \STATE { \color{gray!90} $\triangleright$ Calculate available optional elements in this class}
        \STATE $n_{\text{options}} \gets |R \cap (\bar{S} \setminus \underline{S})|$        
        \STATE { \color{gray!90}$\triangleright$ Calculate available, optional spaces }
        \STATE $n_\text{spaces} \gets \ell - |S| - |R \cap \underline{S}|$        
        \STATE $N \gets N + \binom{n_\text{options}}{n_\text{spaces}}$
    \ENDFOR
    \STATE {\bfseries return} $N$
  \end{algorithmic}
\end{algorithm}

\clearpage
\section{Identifiability Criterion}
\label{appendix:identifiability}

In this section, we operate exclusively with non-parametric identifiability, following \cite{pearl2009causality} (Definition 3.2.4), where an interventional causal query $Q[M]$, defined on a causal model $M$ with DAG $G = (V \cup U, E)$ with $V$ measured nodes and $U$ latent nodes, is (non-parametrically) identifiable from $G$ if the query can be computed uniquely from any positive probability of the observed variables Markov-relative\footnote{
    A probability distribution is Markov-relative to a DAG $G = (V \cup U, E)$ if $P(V, U) = \prod_{X \in V \cup U} P(X \mid Pa_G(X))$, where $Pa_G(X)$ is the set of parents of $X$ in $G$.
} to the graph. In other words, if two causal models $M_1$ and $M_2$ with positive probability distributions $P_{M_1}$ and $P_{M_2}$ are both Markov-relative to $G$ and observationally identical to each other ($P_{M_1}(V) = P_{M_2}(V)$), then both models must have the same value for the given query. Note that this definition does not include any assumptions about the functional forms of the causal models nor the probability distributions of the unobserved variables. In other words, non-parametric identifiability cannot assume anything other than the shape of the graph $G$ (including latent confounders).

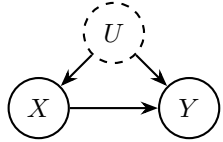
\begin{figure}[htb]
    \centering
    \begin{tikzpicture}[node distance=1cm, thick]
      \node[var] (X) at (0,0) {$X$};
      \node[var, dashed] (U) at (1, 1) {$U$};
      \node[var] (Y) at (2, 0) {$Y$};
    
      \draw[->] (X) -- (Y);
      \draw[->] (U) -- (X);
      \draw[->] (U) -- (Y);
    \end{tikzpicture}
    \caption{Bow-arc graph, with non-identifiable $P_X(Y)$.}
    \label{fig:bow_arc_graph}
\end{figure}

\paragraph{Non-identifiable example.} Consider the Directed Acyclic Graph (DAG) $G=(\{X, Y\} \cup \{U\}, E)$ in \cref{fig:bow_arc_graph}, where $U$ is an unmeasured variable acting as a latent confounder between $X$ and $Y$. In this graph, $P_X(Y)$ is not identifiable. Let us demonstrate by creating two equivalent positive distributions $P^1$ and $P^2$ Markov-relative to $G$, with different intervened distributions for $P_X(Y)$. We define two latent variables: $U$, the latent confounder between $X$ and $Y$, and $E_Y$, the exogenous noise variable for $Y$, both with the same distribution in both models, $\mathcal{N}(0, 1)$. In terms of their functional assignments, $x := f^1_X(u) = f^2_X(u) = u$, while $y := f^1_Y(x, u, e_y) = x - u + e_y$ and $y := f^2_Y(x, u, e_y) = x \cdot 0 + e_y$. Note that all distributions $P^1(X), P^2(X), P^1(Y \mid X), P^2(Y \mid X)$ are identical, $\mathcal{N}(0, 1)$, and so is $P^1(X, Y) = P^1(X) P^1(Y \mid X) = P^2(X) P^2(Y \mid X) = P^2(X, Y)$, and they are positive. However, $P^1(Y \mid do(X=0)) \sim N(0, 2)$ while $P^2(Y \mid do(X=0)) \sim N(0, 1)$. Therefore, $P_X(Y)$ is not identifiable.

Non-parametric identifiability can be demonstrated if we can find an estimand of the causal query only containing observational terms. In particular, the $\texttt{ID}$ algorithm \cite{shpitser2006interventional}, described in \cref{subsec:id_algorithm} and presented in simplified form in \cref{alg:id}, determines non-parametric identifiability w.r.t. a causal graph $G$ of any non-conditional, interventional query $P_X(Y)$, for any disjoint subsets of variables $X, Y$.

Note that any coalition value $\nu(S) = \mathbb{E}[Y \mid do(S=s)]$ is non-parametrically identifiable if, and only if, $P_S(Y)$ is non-parametrically identifiable. Trivially, identifiability in the probability terms guarantees identifiability in the expectation (since the expectation's summation or integral is an observational estimand only consisting of identifiable terms $P_s(y)$). Conversely, non-identifiability in the terms $P_S(Y)$ must imply non-identifiability for the expectation in the absence of further parametrical assumptions (e.g., assuming that $P_s(y) = P_s(-y)$, in which case the terms cancel each other out and make the expectation identifiable). For this reason, we will focus on non-parametric identifiability, and only on queries of the form $P_S(Y)$---which can be determined thanks to the ID algorithm---as a proxy for identifiability of $\nu(S)$.

\subsection{The ID algorithm}
\label{subsec:id_algorithm}

Let us provide the necessary background definitions and properties to understand the ID algorithm and follow the proofs of our identifiability criterion.

\begin{definition}
Let $G=(V \cup U, E)$ be a Directed Acyclic Graph (DAG) with nodes $V \cup U$ (measured and latent nodes, respectively) and edges $E$. Its \textbf{latent projection} is an \textbf{Acyclic Directed Mixed Graph} (ADMG) $G^*=(V, E^*)$, which is a graph containing directed and bidirected edges and no directed cycles. In particular, $G^*$ only contains the measured nodes in $G$ and: (1) any directed edge in $G$ between measured nodes; (2) an edge between measured nodes if both are connected in $G$ by a directed path where all intermediate nodes are in $U$; and (3) a bidirected edge $V_i \leftrightarrow V_j$ between measured nodes if both are connected by a path of the form $V_i \leftarrow \cdots \leftarrow \cdot \rightarrow \cdots \rightarrow V_j$ where all intermediate nodes are in $U$.
\end{definition}

In the following, when talking about a graph, we refer to the ADMG latent projection of the DAG under study.

\paragraph{Notation.} Let $G$ be an ADMG. We denote by $V(G)$ the set of nodes of $G$. We denote by $G[S]$ the induced subgraph of $G$ filtering to the nodes $S \subseteq V(G)$ and only preserving those edges connecting preserved nodes. For a node $Y \in V(G)$, we denote by $An_G(Y)$ the set of (directed) ancestors of node $Y$ in $G$ including $Y$, and by $De_G(Y)$ the set of (directed) descendants of node $Y$ in $G$ including $Y$. We overload the notation with the ancestors of a set of nodes $T \subseteq V(G)$, $An_G(T) := \bigcup_{t \in T} An_G(t)$, and equivalently for $De_G(T)$. We denote the intervened graph $G_{\overline{S}}$ as the subgraph of $G$ such that any incoming edge to the nodes in $S$ is removed. For this reason, in this section, we will avoid denoting the closure of a coalition $S$ by $\overline{S}$ to avoid ambiguity.

\begin{definition}
    The \textbf{root set} $R$ of an ADMG $G$ is the set of nodes in $G$ with no proper descendants: $R := \{X \in V(G) \mid De_G(X) = \{X\}\}$.
\end{definition}

\begin{definition}
    Let $G$ be an ADMG. We say that a set of nodes $V$ is a \textbf{C-component} within an ADMG $G$ if all nodes are connected via bidirected arcs. We denote the set of maximal C-components of a graph $G$ by $\mathcal{C}(G)$. In particular, $\mathcal{C}(G)$ constitutes a partition of $V(G)$. We also say $G$ is a C-component when $V(G)$ is a C-component in $G$.
\end{definition}

\begin{definition}
    We say an ADMG $G$ is an $R$-rooted \textbf{C-forest} if $R$ is its root set, $G$ is a C-component, and all its nodes have at most one (directed) child.
\end{definition}

\begin{definition}
    Let $G$ be an ADMG, and $S, T$ disjoint subsets of variables in $V(G)$. We say a pair $(F, F')$ is an $R$-rooted \textbf{hedge} for $P_S(T)$ if both $F$ and $F'$ are $R$-rooted C-forests such that $F' \subseteq F \subseteq G$, $F \cap S \neq \emptyset$, $F' \cap S = \emptyset$ and $R \subseteq An_{G_{\overline{S}}}(T)$.
\end{definition}

Hedges are the reason for non-identifiability of causal effects. See the following theorem from \citep{shpitser2006interventional} (Theorem 4). 
\begin{theorem}
    If $(F, F')$ is a hedge for $P_S(T)$ in an ADMG $G$, $P_S(T)$ is not identifiable in $G$.
\end{theorem}

We are now ready to study the ID algorithm \cite{shpitser2006interventional}, with which we can ascertain the identifiability of any query $P_S(T)$. We present a simplified version in \cref{alg:id}.

\begin{algorithm}[t]
\caption{$\texttt{ID}(T, S, G)$}
\label{alg:identifiability}
\begin{algorithmic}
    \STATE {\bfseries Input:} Disjoint sets $T, S \subseteq \mathcal{V}$, ADMG $G = (\mathcal{V}, E)$.
    \STATE {\bfseries Output:} Boolean indicating whether $P_S(T)$ is identifiable.
    \STATE {\bfseries 1:} If $S = \emptyset$, {\bfseries return} TRUE.
    \STATE {\bfseries 2:} If $\mathcal{V} \neq An_G(T)$, {\bfseries return} $\texttt{ID}(T, S \cap An_G(T), G[An_G(T)])$.
    \STATE {\bfseries 3:} Let $W := (\mathcal{V} \setminus S) \setminus An_{G_{\overline{S}}}(T)$. If $W \neq \emptyset$, {\bfseries return} $\texttt{ID}(T, S \cup W, G)$.
    \STATE {\bfseries 4:} If $\mathcal{V} \setminus S \not\in \mathcal{C}(G[\mathcal{V} \setminus S])$, {\bfseries return} $\wedge_{C_i \in \mathcal{C}(G[\mathcal{V} \setminus S])} \texttt{ID}(C_i, \mathcal{V} \setminus C_i, G)$.
    \STATE {\bfseries 5:} If $\mathcal{V} \in \mathcal{C}(G)$, {\bfseries return} FALSE.
    \STATE {\bfseries 6:} If $\mathcal{V} \setminus S \in \mathcal{C}(G)$, {\bfseries return} TRUE.
    \STATE {\bfseries 7:} Let $C \in \mathcal{C}(G)$ s.t. $C \supset \mathcal{V} \setminus S$, {\bfseries return} $\texttt{ID}(T, S \cap C, G[C])$.
\end{algorithmic}
\label{alg:id}
\end{algorithm}

\subsection{Identifiability criterion}

When evaluating do-SHAP, one must evaluate causal queries $\nu(S) := \mathbb{E}[Y \mid do(S=s)]$ for multiple coalitions $S \subseteq [d]$. If one such query is not identifiable, we cannot continue evaluating do-SHAP and must raise an error. However, this forces practitioners to execute do-SHAP without knowing if an estimate can be raised, while also running the ID algorithm for every single coalition that must be evaluated.

Instead, we present the following result, which allows us to determine the identifiability of all $2^d$ coalition values $\nu(S)$ just by determining the identifiability of the $d$ singleton coalitions $\{i\} \subseteq [d]$.

\begin{theorem}
    Let $G$ be an ADMG over $\mathcal{V}$, with a single $Y \in \mathcal{V}$ target node. If $\exists X \subseteq \mathcal{V} \setminus \{Y\}$ such that $P_X(Y)$ is not identifiable, then $\exists s \in X \cup (\mathcal{V} \setminus An_{G_{\overline{X}}}(Y))$ such that $P_{\{s\}}(Y)$ is not identifiable.
    \label{thm:doshapid}
\end{theorem}

We will devote the remainder of this section to proving this theorem. Let us first prove a set of lemmas that will progressively build towards the proof.

\begin{lemma}
    Let $G$ be an ADMG over $\mathcal{V}$, with a set $Y \subseteq \mathcal{V}$ the target nodes. If $\exists X \subseteq \mathcal{V} \setminus Y$ such that $P_X(Y)$ is not identifiable, then the $\texttt{ID}$ algorithm ends at line 5 in a recursive call $\texttt{ID}(T, S, G')$ for $S$ and $T$ non-empty disjoint subsets of $\mathcal{V}' \subseteq \mathcal{V}$, and $G'$ an induced subgraph $G' = G[\mathcal{V}'] \subseteq G$. Then, there exists $R$-rooted C-forests $(F, F')$, $R \subseteq \mathcal{V}'$ such that they constitute a hedge for $P_S(T)$.
    \label{lemma:hedge}
\end{lemma}

\begin{proof}
Firstly, the $\texttt{ID}$ algorithm \cite{shpitser2006interventional} always terminates (Lemma 3 in their paper), it is sound (Theorem 5) and complete (Corollary 2). Since we assume that $P_X(Y)$ is not identifiable, it must be that when we evaluate $\texttt{ID}(Y, X, G)$, we will eventually return FALSE from line 5 at a certain recursion level $\texttt{ID}(T, S, G')$. Both $X$ and $S$ must be non-empty (otherwise line 1 would have returned TRUE). Both $S$ and $T$ must be subsets of $\mathcal{V}$ and disjoint (since all recursive calls maintain them being disjoint given the initial assumption that $Y \cap X = \emptyset$). Finally, $G'$ is an induced subgraph $G' = G[\mathcal{V}'] \subseteq G$, since all recursive calls will at most remove nodes from $G$.

Let $R$ be the root set of $G'[\mathcal{V}' \setminus S]$. In particular, $R \cap S = \emptyset$. Let $F'$ be an edge subgraph of $G'[\mathcal{V}' \setminus S]$ such that its root set remains $R$ but all observable nodes have at most one child, and all confounding arcs in $G'[\mathcal{V}' \setminus S]$ are present. Note that $G'[\mathcal{V}' \setminus S]$ is a C-component (otherwise, line 4 would have triggered) and in creating $F'$ we did not remove any confounding arcs, so $F'$ is also a C-component, and therefore an $R$-rooted C-forest. Now, let us define an edge subgraph $F \subseteq G'$ by starting from $F'$, adding all nodes in $S$ and, for every such node adding only one edge to one of its children in $G'$, as well as all bidirected edges. Since every new node has a child in $V(F)$, the root set remains $R$, and $F$ is an $R$-rooted C-forest by the same reasoning as before because line 5 guarantees $\mathcal{C}(G') = \{\mathcal{V}'\}$. Additionally, $F' \subseteq F$, $V(F') \cap S = \emptyset$ and $V(F) \cap S = S \neq \emptyset$. All that remains is to prove that $R \subseteq An_{G'_{\overline{S}}}(T)$. First, note that $R \subseteq An_{G'}(T)$ by line 2, so there exist paths from nodes $r \in R$ to nodes $t \in T$. Let us assume that one such $r$ has all its directed paths blocked by $S$. Then, $r \not\in An_{G_{\overline{S}}}(T)$ and $r \not\in S$ (because $R \cap S = \emptyset$), which means that $W$ in line 3 would have contained $r$. By contradiction, $R \subseteq An_{G'_{\overline{S}}}(T)$ and, finally, $(F, F')$ is a hedge for $P_S(T)$ in $G'$.
\end{proof}

We now define some notation that will guide the proofs that follow. Let $G$ be an ADMG over $\mathcal{V}$, with a \textit{single} target node $Y \in \mathcal{V}$. If $\exists X \subseteq \mathcal{V} \setminus \{Y\}$ such that $P_X(Y)$ is not identifiable, let us focus on a single call-path returning FALSE, and denote the $i$-deep call $ID(T^{i}, S^{i}, G^{i})$ (with corresponding $\mathcal{V}^{i} := V(G^{i})$) along the call-path, starting on $ID(\{Y\}, X, G)$ (with $\{Y\} = T^{0}, X = S^{0}, G = G^{0}$) and ending at depth $d \geq 0$ in line 5 for the found hedge. Whenever line 4 or 7 are called at depth $i$, let us denote by $C^i$ the C-component that appears either: in line 4, for the recursive call to $\texttt{ID}$ in the particular branch following the call-path, or the one in line 7 that filters $S$. Let $d'$ be the first depth level at which line 4 is triggered or $d + 1$ if it never does.

\begin{lemma}
    If line 7 triggers at $\texttt{ID}(T^i, S^i, G^i)$, then $\mathcal{V}^i \setminus S^i \subsetneq C^i \subsetneq \mathcal{V}^i$. In particular, $S^{i+1} \subsetneq S^i$.
    \label{lemma:line7proper}
\end{lemma}

\begin{proof}
$\mathcal{V}^i \setminus S^i \subset C^i \subset \mathcal{V}^i$ and it is a maximal C-component in $G^i$ by construction. If $C^i = \mathcal{V}^i \setminus S^i$, then $\mathcal{V}^i \setminus S^i \in \mathcal{C}(G^i)$ and line 6 would have triggered. If $C^i = \mathcal{V}^i$, then $V^i \in \mathcal{C}(G)$ and line 5 would have triggered. Finally, since $C^i \neq \mathcal{V}^i$ and $C^i \supsetneq \mathcal{V}^i \setminus S^i$, then $S^i \cap C \subsetneq S^i$.
\end{proof}

\begin{lemma}
    Some observations that will become relevant throughout the following discussion are:
    \begin{itemize}
        \item Line 1 never triggers along the failing call-path, and $\forall i \leq d, S^{i} \neq \emptyset$; otherwise, the call-path would not end at line 5.
        \item Line 6 never triggers along the failing call-path; otherwise, the call-path would not end at line 5.
        \item Line 5 does not trigger along the failing call-path for depths $i < d$.
        \item $G$ can only have nodes removed: $G^{i+1} \subseteq G^i,\; \forall i < d$.
        \item We can only remove nodes from the local set of intervened variables $S$ on lines 2 and 7, while also removing them from the local graph $G$.
        \item We can only add nodes to the local set of intervened variables $S$ on lines 3 and 4, while also removing them from the set of unused nodes $\mathcal{U} := \mathcal{V} \setminus (T \cup S)$.
        \item The set of unused nodes can never grow: $\mathcal{U}^{i+1} \subseteq \mathcal{U}^{i},\; \forall 0 \leq i < d$.
        \item $T$ does not change until line 4 is called: $\forall i \leq d'$, $T^i = \{Y\}$.
    \end{itemize}
    \label{lemma:observations}
\end{lemma}

\begin{lemma}
    Under the previous assumptions, line 4 can trigger at most once.
    \label{lemma:line4once}
\end{lemma}

\begin{proof}
    Assuming that line 4 triggers at least once, let $d' < d$ be the depth of its first call, and $C^{d'}$ the C-component whose recursive call we follow along the call-path. Right after that, $T^{d'+1} = C^{d'}$, $S^{d'+1} = \mathcal{V}^{d'} \setminus C^{d'}$, and $U^{d'+1} = \mathcal{V}^{d'+1} \setminus (T^{d'+1} \cup S^{d'+1}) = \emptyset$. Assume there is a minimal $d'' > d'$ in which line 4 triggers again. Then, since only line 4 can change $T$, $T^{d''} = T^{d'+1} = C^{d'}$. Note that $\mathcal{V}^{d''} \setminus S^{d''} = T^{d''} = C^{d'}$ (otherwise $U^{d''}$ would increase). Therefore, $\mathcal{C}(G^{d''}[\mathcal{V}^{d''} \setminus S^{d''}]) = \mathcal{C}(G^{d''}[C^{d'}]) = \mathcal{C}(G^{d'}[C^{d'}]) = \{C^{d'}\}$ since $G^{d''} \subseteq G^{d'}$ and $C^{d'}$ is a maximal C-component in $G^{d'}[\mathcal{V}^{d'} \setminus S^{d'}] \supsetneq G^{d'}[C^{d'}]$, which means that line 4 cannot trigger again. In particular, $T^i = C^{d'},\; \forall i: d' < i \leq d$.
\end{proof}

\begin{lemma}
    Under the previous assumptions, if $\mathcal{V} = An_G(Y)$ and $X$ is closed (i.e., $X = X \cup (\mathcal{V} \setminus An_{G_{\overline{X}}}(Y))$), then line 3 never triggers.
    \label{lemma:line3once}
\end{lemma}

\begin{proof}
    By the fact that the closure of a closed set is itself, line 3 cannot trigger at depth $0$. Let us prove the statement by induction.
    
    Firstly, consider a depth $i$, $0 < i < d'$, in which line 3 is triggered, while assuming it has not been triggered before. From Lemma \ref{lemma:observations}, $T^i = \{Y\}$. For depth $i$, it must be that $W^i := (\mathcal{V}^i \setminus S^i) \setminus An_{G^i_{\overline{S^i}}}(Y) \neq \emptyset$, with nodes $w \not\in S^i$ and $w \not\in An_{G^i_{\overline{S^i}}}(Y)$, but $w \in An_{G^i}(Y)$ (since we skip line 2 to trigger 3), so there are directed paths from $w$ to $Y$, all blocked by $S^i$. Let us now study depth $i -1$. Since $w \not\in S^i$ and $S^{i-1} \supseteq S^i$, if $w \in S^{i-1}$ it must have been removed by lines 2 or 7, but that would also remove it from $\mathcal{V}^i$, which is not the case, so we know that $w \not\in S^{i-1}$. Additionally, if we assumed that $w \not\in An_{G^{i-1}_{\overline{S^{i-1}}}}(Y)$, since $w \not\in S^i$, then $w \in W^{i-1}$ but because line 3 did not trigger at depth $i-1$, $W^{i-1} = \emptyset$. Consequently, $w \in An_{G^{i-1}_{\overline{S^{i-1}}}}(Y)$ and $w \not\in S^{i-1}$, so there are directed paths from $w$ to $Y$ unblocked by $S^{i-1}$, going through nodes in a set $Z \subseteq \mathcal{V}^{i-1} \setminus S^{i-1}$, which all must have disappeared from $G^i$ to make it so $w \not\in An_{G^i_{\overline{S^i}}}(Y)$. Since line 2 preserves all directed paths to $Y$, this can only happen if line 7 was triggered at depth $i - 1$. However, this cannot be the case either: the C-component $C^{i-1} \supsetneq \mathcal{V}^{i-1} \setminus S^{i-1}$ preserves all non-intervened nodes, which is the case for $Z$. Therefore, it must be that our initial assumption was false, $W^{i}$ is empty, and so line 3 was not triggered at depth $i < d'$ after all.
    
    For depth $d'$, we naturally skip line 3 to reach line 4. Afterwards, $\mathcal{V}^{d'+1} = T^{d'+1} \sqcup S^{d'+1}$ so $\forall d' < i \leq d$, $\mathcal{U}^i := \mathcal{V}^i \setminus (T^i \cup S^i) = \emptyset$ (Lemma \ref{lemma:observations}), so $W^i = \emptyset$, and line 3 is never called either.
\end{proof}

\begin{lemma}
Let $G$ be an ADMG over $\mathcal{V}$, with a single $Y \in \mathcal{V}$ target node. If $\exists X \subseteq \mathcal{V} \setminus \{Y\}$ such that $P_X(Y)$ is not identifiable, let $\texttt{ID}(T^d, S^d, G^d)$ be the last call resulting in the hedge given by Lemma \ref{lemma:hedge}. Then $S^d \cap (X \cup (\mathcal{V} \setminus An_{G_{\overline{X}}}(Y))) \neq \emptyset$.
\end{lemma}

\begin{proof}
Let $X'$ be the closure of $X$, $X' := X \cup (\mathcal{V} \setminus An_{G_{\overline{X}}}(Y))$ (here denoted $X'$ instead of $\overline{X}$ to avoid ambiguity of notation), and let us prove that $S^{d} \cap X' \neq \emptyset$. Note that if line 2 or 3 trigger on the first calls, we end up in a recursive call $\texttt{ID}(\{Y\}, S^{i}, G^{i})$ where $\mathcal{V}^{i} = An_G(Y)$ and $S^{i} = An_G(Y) \cap (X \cup (\mathcal{V} \setminus An_{G_{\overline{X}}}(Y)))$, so we can assume without loss of generality that $\mathcal{V} = An_G(Y)$ and $X' = X$ at the start of the call-path. Consequently, $S^{0} \cap X' = X' \neq \emptyset$ trivially. Let us proceed by induction, proving that $\forall i < d, S^{i} \cap X' \neq \emptyset$ implies that $S^{i+1} \cap X' \neq \emptyset$, and then so will be the case for $S^d$, which will prove the lemma.

Let us now consider any depth level $0 < i \leq min(d, d')$ and assume that $S^{i-1} \subseteq X'$. Let us study what happens at depth $i-1$. Lines 1, 3, 4, and 6 cannot trigger (by Lemmas \ref{lemma:observations} and \ref{lemma:line3once} and the fact that $i \leq d'$). If line 2 were to remove any nodes, it would still result in a non-empty proper subset $S^i \subsetneq S^{i-1} \subseteq X'$. If line 7 triggered, $S^i := S^{i-1} \cap C^{i-1} \subsetneq S^{i-1} \subseteq X'$ and $S^i \neq \emptyset$ by Lemma \ref{lemma:line7proper}. In all cases, the next recursive $S$ is not empty and only contains nodes in $X'$: $\emptyset \neq S^{i+1} \subsetneq S^{i} \subseteq X'$.

If line 4 never triggers, we have already proved the result. Otherwise, $d' < d$, and $C^{d'}$ is the C-component whose recursive call we follow along the call-path. In particular, $\mathcal{V}^{d'+1} = \mathcal{V}^{d'}, T^{d'+1} = C^{d'}$, $S^{d'+1} = \mathcal{V}^{d'} \setminus C^{d'} \supsetneq S^{d'}$, so it contains elements of $X'$ and outside $X'$ (because the nodes in $X' \setminus S^{d'}$ have been removed from $G^{d'}$ as well). Additionally, $\forall i > d',\; \mathcal{U}^i = \emptyset$, and can never increase.

For $d' \leq i \leq d$, let us define $S^i = A^i \sqcup B^i \neq \emptyset$, with $A^i := S^i \cap X'$ and $B^i := S^i \setminus X'$. Note that $A^{d'} = S^{d'}, B^{d'} = \emptyset$, and $A^{d'+1} = S^{d'}, B^{d'+1} = (\mathcal{V}^{d'} \setminus C^{d'}) \setminus A^{d'} \neq \emptyset$, since $S^{d'+1} = \mathcal{V}^{d'} \setminus C^{d'} \supsetneq S^{d'} \subseteq X'$. Additionally, $\forall i > d' + 1,\; A^i \subseteq A^{i-1}$ and $B^i \subseteq B^{i-1}$ since we can only remove nodes at this point, given that line 3 and line 4 will not run again (Lemmas \ref{lemma:line3once} and \ref{lemma:line4once}). Let us prove that if $\forall i: d' < i < d$, $A^i \neq \emptyset$, then $A^{i+1} \neq \emptyset$, for which we must check lines 2 and 7. 

If line 2 was triggered at step $i$ and $A^{i+1} = \emptyset$, the next recursion level would have $T^{i+1} = T^i$, $S^{i+1} = B^{i+1} \neq \emptyset$ and $G^{i+1} = G^i[An_{G^i}(T^i)] \subseteq G^i$. 
At level $i+1$, line 6 would trigger: 
    neither line 3 nor line 4 can trigger again; 
    if we assume that line 5 triggers, then $\mathcal{V}^{i+1} \in \mathcal{C}(G^{i+1})$, but $B^{i+1} \subseteq B^{d'+1}$ and no node in $B^{d'+1}$ belonged in the same maximal C-component as $C^{d'}$ in $G^{d'}[\mathcal{V}^{d'} \setminus A^{d'}]$, then all bidirected paths from $B^{i+1}$ to $C^{d'}$ were blocked by $A^{d'}$, and now that $A^{i+1} = \emptyset$, these paths are cut;
    finally, for line 6, if $\mathcal{V}^{i+1} \setminus S^{i+1} = T^{i+1} = C^{d'}$ were not a maximal C-component of the induced subgraph $G^{i+1} = G^i[An_{G^i}(T^i)]$, it must be that a node $b \in S^{i+1} = B^{i+1}$ is connected through a bidirected path to $C^{d'}$ in this graph, which we have just proved to be false, hence proving that the call would return TRUE at line 6.
By contradiction, this proves that line 2 would have not removed all of $A^{d'+1} \subseteq X'$.

Finally, if line 7 was triggered at step $i$ and $A^{i+1} = \emptyset$, when before it was not, necessarily $B^{i+1} \neq \emptyset$, and $\forall c \in T^i = C^{d'}$, $\forall b \in B^{i+1} = B^i \cap C^i$, there are bidirected paths connecting $c$ and $b$ in $G^i \supsetneq G^{i+1}$, and all are intersected by $S^{d'} = A^{d'}$ (otherwise $b \in C^{d'}$). Consequently, either these intersecting nodes are in $A^i$, in which case they would not have been removed from $A^{i+1}$, or they are not in $A^i$, in which case $b$ would have been removed as well for its paths would be cut. Both cases contradict the assumption that $A^{i+1} = \emptyset$.


Bringing everything together and by induction, $\emptyset \subsetneq A^d \subseteq X'$, proving the lemma.
\end{proof}

We now have all we need to prove the theorem.

\begin{proof}[Proof of \cref{thm:doshapid}]
Finally, let us finish the proof by showing that if $(F, F')$ is the $R$-rooted hedge for $P_S(T)$ found at the end of the call-stack, $\forall s \in S \cap X'$, with $X' := X \cup (\mathcal{V} \setminus An_{G_{\overline{X}}}(Y))$ the closure of $X$, then $(F, F')$ is also a hedge for $P_{\{s\}}(Y)$. Note that the hedge still fulfills the conditions $F' \subseteq F, F \cap \{s\} \neq \emptyset, F' \cap \{s\} = \emptyset$, both $R$-rooted, both C-forests. So now we only need to check that $R \subseteq An_{G_{\overline{s}}}(Y)$, but if one such $r \in R$ were not in $An_{G_{\overline{s}}}(Y) \supseteq An_{G_{\overline{X'}}}(Y)$, then $r$ would belong in $X'$ by definition of the closure of $X$ and the fact that the closure of a closure is itself. Additionally, throughout the call-stack, $r$ would have never left the set of intervened nodes without leaving the graph itself, hence $r \in S$, contradicting the fact that $S \cap R = \emptyset$ by construction (see Lemma \ref{lemma:hedge}). Therefore, it must be that $R \subseteq An_{G_{\overline{s}}}(Y)$, proving that $(F, F')$ is a hedge for $P_{\{s\}}(Y)$. By Theorem 4 in \cite{shpitser2006interventional}, $P_{\{s\}}(Y)$ is not identifiable.
\end{proof}

\begin{corollary}
    Any do-Shapley value $\phi_i$ for $i \in [d]$ is (non-parametrically) identifiable if, and only if, all singleton coalition probabilities $(P_{\{k\}}(Y))_{k \in [d]}$ are (non-parametrically) identifiable.
\end{corollary}

\begin{proof}    
    If all singleton coalitions $\{k\} \subseteq [d]$ have identifiable probability terms $P_{\{k\}}(Y)$, there cannot be any coalition $S \subseteq [d]$ for which $P_S(Y)$ is not identifiable (otherwise, by \cref{thm:doshapid}, $\forall k \in S \cup (\mathcal{V} \setminus An_{G_{\overline{S}}}(Y))$, $P_{\{k\}}(Y)$ would not be identifiable). Then, all the respective $\nu(S) = \mathbb{E}[Y \mid do(S = s)]$ are identifiable, and consequently, $\phi_i = \sum_{j=1}^r \nu(c_j) w_i(c_j)$ is also identifiable. Conversely, if one such $P_{\{k\}}(Y)$ were not identifiable, its $\nu(\{k\})$ term cannot be either, nor can $\phi_k$.
\end{proof}

\clearpage
\section{Extensions of the Shapley value}
\label{sec:appx_shapley_interactions}

The Shapley value was extended in multiple ways.
Semivalues \citep{Dubey.1981} , such as the Banzhaf value \citep{Banzhaf.1964}, extend the Shapley values to alternative weighting schemes $(q_s)_{s=0,\dots,n-1} \geq 0$ as
\begin{align}
\phi^q_i = \sum_{S \subseteq [d] \setminus {i }}
[\nu(S \cup {i }) - \nu(S)] q_{|S|}.
\end{align}
If $q$ is a probability distribution over $2^{[d]}$, then they are also referred to as cardinal-probabilistic values \citep{Fujimoto.2006}.
By replacing the weights in \cref{eq:do_shapley}, we can derive an efficient computation for any semivalue.

Another line of work extends the Shapley value to higher-order interactions, known as the Shapley interaction index \citep{Grabisch.1999}, and defined by
\begin{align*}
  \phi_U = \sum_{S \subseteq [d] \setminus U}
\Delta_U(S) p_{\vert S \vert}^{\vert U \vert} \text{ with } \Delta_U(S) := \sum_{L\subseteq U} (-1)^{\vert U\vert - \vert S \vert} \nu(S \cup L),
\end{align*}
and $p_{s}^u := \frac{1}{(d-u+1) \cdot \binom{d-u}{s}}$.
The discrete derivative $\Delta_U(S)$ thereby measures the interaction of $U$ in the presence of $S$, and directly extends the marginal contribution.
We now utilize a result by \citet{zern2023interventional}[Proposition 1] to efficiently compute the Shapley interaction index for value functions of the shape $\mathbf{1}[\underline S \subseteq S \subseteq \overline S]$.

\begin{proposition}[\citet{Zern.2023}]
Given subsets $\underline S \subseteq \overline S \subseteq [d]$ and value function $\nu(S) = \mathbf{1}[\underline S \subseteq S \subseteq \overline S]$, the Shapley interaction index is given by the weights  $\omega_{a,b} := \frac{1}{a+b+1}\binom{a+b}{a}^{-1}$ and
\begin{align*}
    \phi_U = (-1)^{\vert U \cap ([d]\setminus \overline S)\vert} \omega_{\vert \underline S\vert - \vert \overline S \cap U\vert, \vert [d] \setminus (\overline S \cup U)\vert},
\end{align*}
if $U \subseteq A \cup ([d] \setminus B)$, and $\phi_U = 0$ otherwise.
\end{proposition}

Notably, the weights $\omega$ can be precomputed.
Consequently, we obtain the following result by the linearity of the Shapley interaction index \citep{Grabisch.1999}.

\begin{proposition}
The do-Shapley interaction index is given by
    \begin{align*}
        \phi_U = \sum_{j=1}^r \nu(c_j) (-1)^{\vert U \cap ([d]\setminus \overline S_j)\vert} \omega_{\vert \underline S_j \vert - \vert \overline S_j \cap U\vert, \vert [d] \setminus (\overline S_j \cup U)\vert}.
    \end{align*}
\end{proposition}

\begin{proof}
We define $\nu_j(S) := \mathbf{1}[\underline S_j \subseteq S \subseteq \overline S_j]$ for each irreducible set from $j=1,\dots,r$.
Furthermore, we denote $\phi_U[\nu]$ as the Shapley interaction index with respect to $\nu$. Then, by linearity \citep{Grabisch.1999} of $\phi_U$, we obtain
\begin{align*}
    \phi_U[\nu]&= \phi_U[\sum_{j=1}^r \nu(c_j) \cdot \mathbf{1}[\underline S_j \subseteq S \subseteq \overline S_j]] = \sum_{j=1}^r \nu(c_j) \phi_U[\nu_j] = \sum_{j=1}^r \nu(c_j) (-1)^{\vert U \cap ([d]\setminus \overline S_j)\vert} \omega_{\vert \underline S_j \vert - \vert \overline S_j \cap U\vert, \vert [d] \setminus (\overline S_j \cup U)\vert}.
\end{align*}
\end{proof}

With the efficient computation of the Shapley interaction index \citep{Grabisch.1999}, we can directly extract the $n$-Shapley values \citep{lundberg2018consistent,Bord.2023} using the following recursion:

\begin{align*}
    \Phi_U^n := \begin{cases}
        \phi_U &\text{if } \vert U \vert = n,
        \\
        \Phi^{n-1}_U + B_{d-\vert U \vert} \sum_{\substack{K\subseteq [d] \setminus U \\ \vert K \vert + \vert S \vert = n}} \phi_{U \cup K} &\text{if } \vert U \vert < n,
    \end{cases}
\end{align*}
with $\Phi_i^1 := \phi_i$ for all $i \in [d]$.
The $n$-Shapley values have first been introduced by \citet{lundberg2018treeshap}, and were later genearlized to arbitrary order $n$ \citep{Bord.2023}.
Importantly, they satisfy the generalized efficiency axiom for Shapley interactions \citep{Bord.2023}, i.e.
\begin{align*}
    \sum_{U \subseteq [d]: \vert U \vert \leq n} \Phi^n_U = \nu([d]),
\end{align*}
where we defined $\Phi^0_\emptyset := \nu(\emptyset)$.

\textbf{do-Shapley Interactions for SCMs.} Generalizations of the Shapley value to higher-order interactions provide a principled framework for explanations with varying degrees of granularity and expressivity. To illustrate this for do-Shapley values and their interaction-based extensions, we compute Shapley interactions using \texttt{shapiq} \citep{Muschalik.2024a} for SCMs derived from four datasets from the TALENT benchmark \citep{JMLR:v26:25-0512}. Following \cref{sec:experiments}, the SCMs are learned from data using GRaSP \citep{lam2022greedy}, with the corresponding variables described in \cref{tab:variables_Brazilian_houses_reproduced,tab:variables_FOREX_audjpy-hour-High,tab:variables_Laptop_Prices_Dataset,tab:variables_yeast}. The resulting do-Shapley interactions are visualized in \cref{fig:shapiq_Laptop_Prices_Dataset,fig:shapiq_FOREX_audjpy-hour-High,fig:shapiq_yeast,fig:shapiq_Brazilian_houses_reproduced}. First-order explanations recover the standard do-Shapley values, while the highest-order interactions correspond to the Möbius transform of the value function. We present the \texttt{shapiq} explanations as overlays on the learned SCMs, where node sizes represent main effects and edge or hyperedge widths encode interaction strength. The color indicates the direction of the effects (blue denotes a negative interaction and red a positive interaction). While first-order do-Shapley values summarize aggregated causal effects of individual features, increasing the interaction order yields progressively finer-grained insights into the causal structure. At the same time, the growing number of interaction terms poses interpretability challenges, motivating application-specific post-processing and selection of relevant interaction orders.

\begin{figure*}[htb]
    \centering
    \begin{minipage}[b]{0.245\linewidth}
        \centering
        \includegraphics[width=\linewidth]{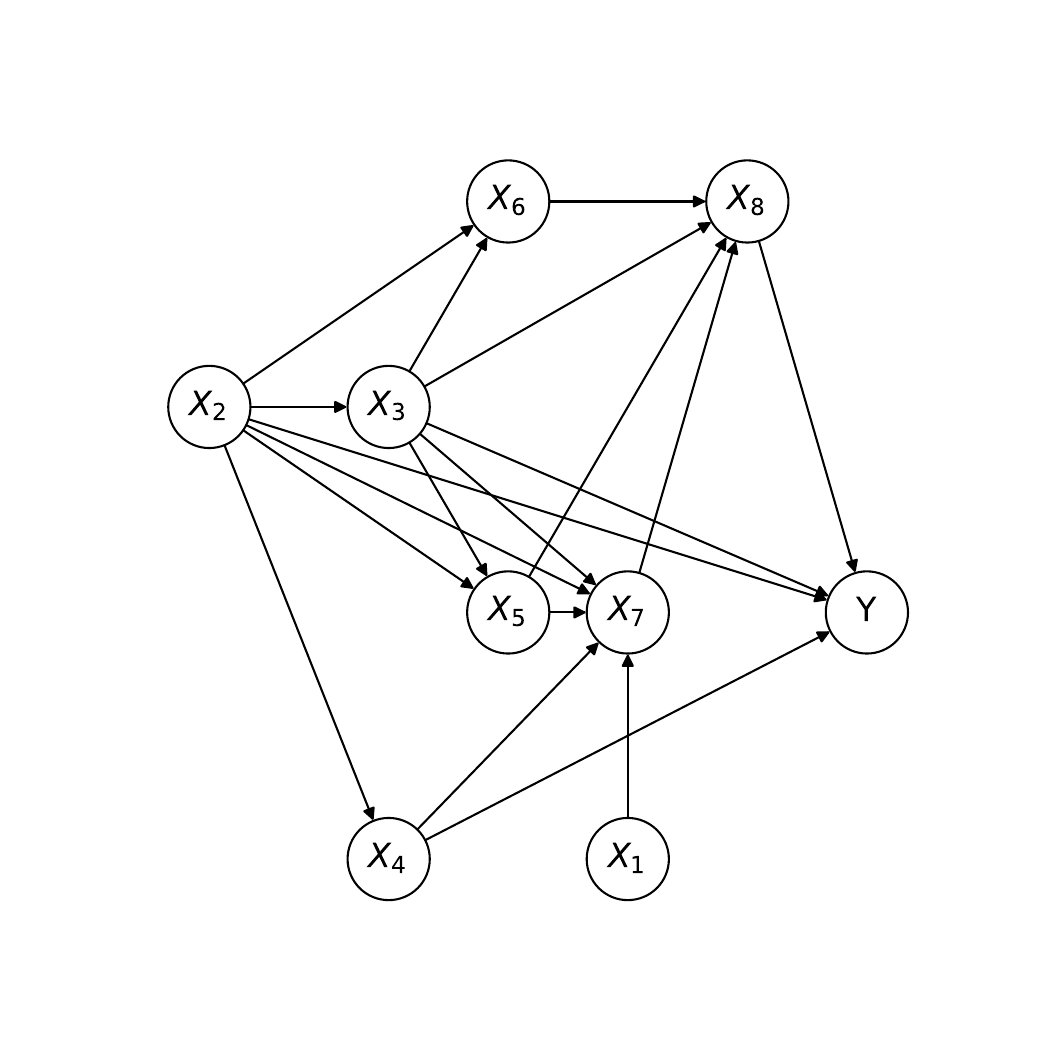}
        \textbf{SCM}\\\phantom{(order 1)}
    \end{minipage}
    \begin{minipage}[b]{0.245\linewidth}
        \centering
        \includegraphics[width=\linewidth]{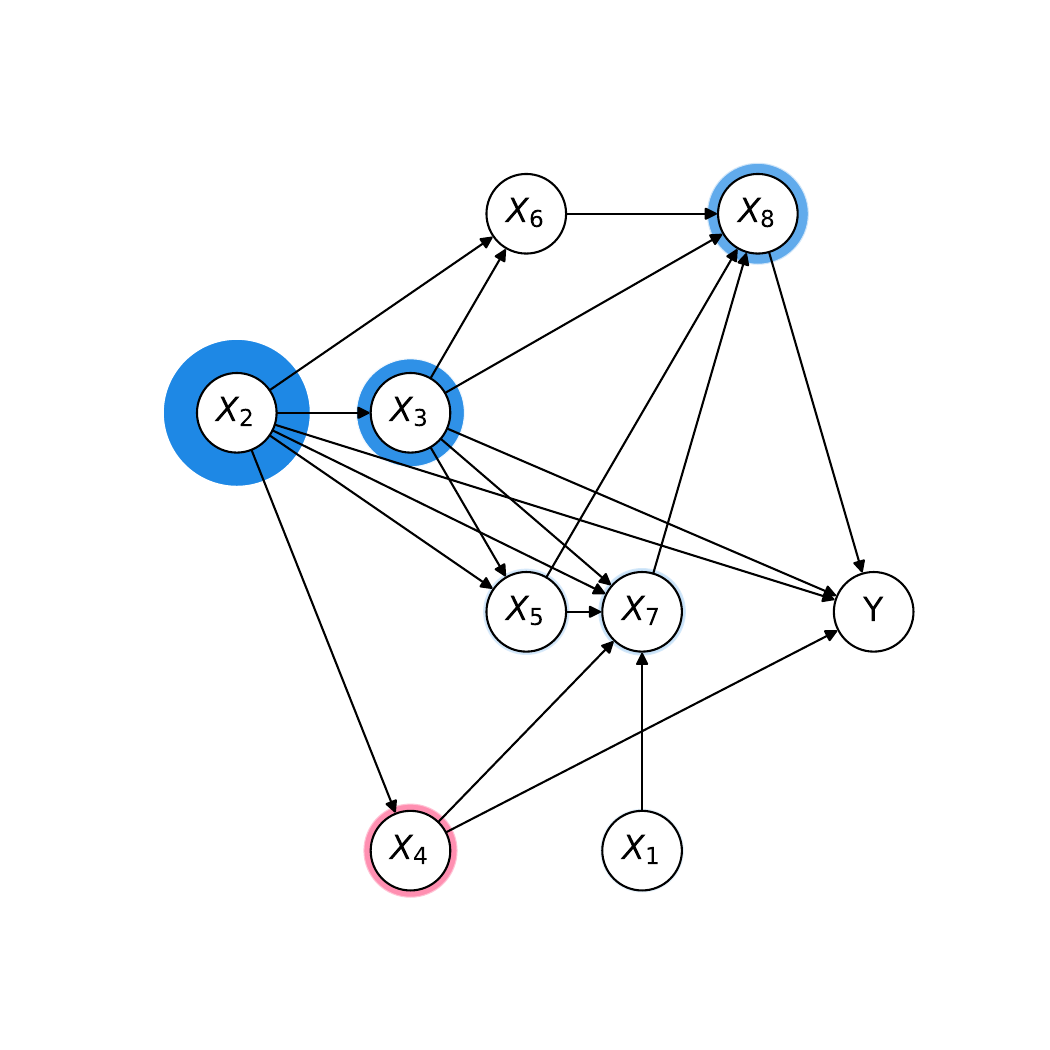}
        \textbf{do-Shapley Values}\\(order 1)
    \end{minipage}
    \begin{minipage}[b]{0.245\linewidth}
        \centering
        \includegraphics[width=\linewidth]{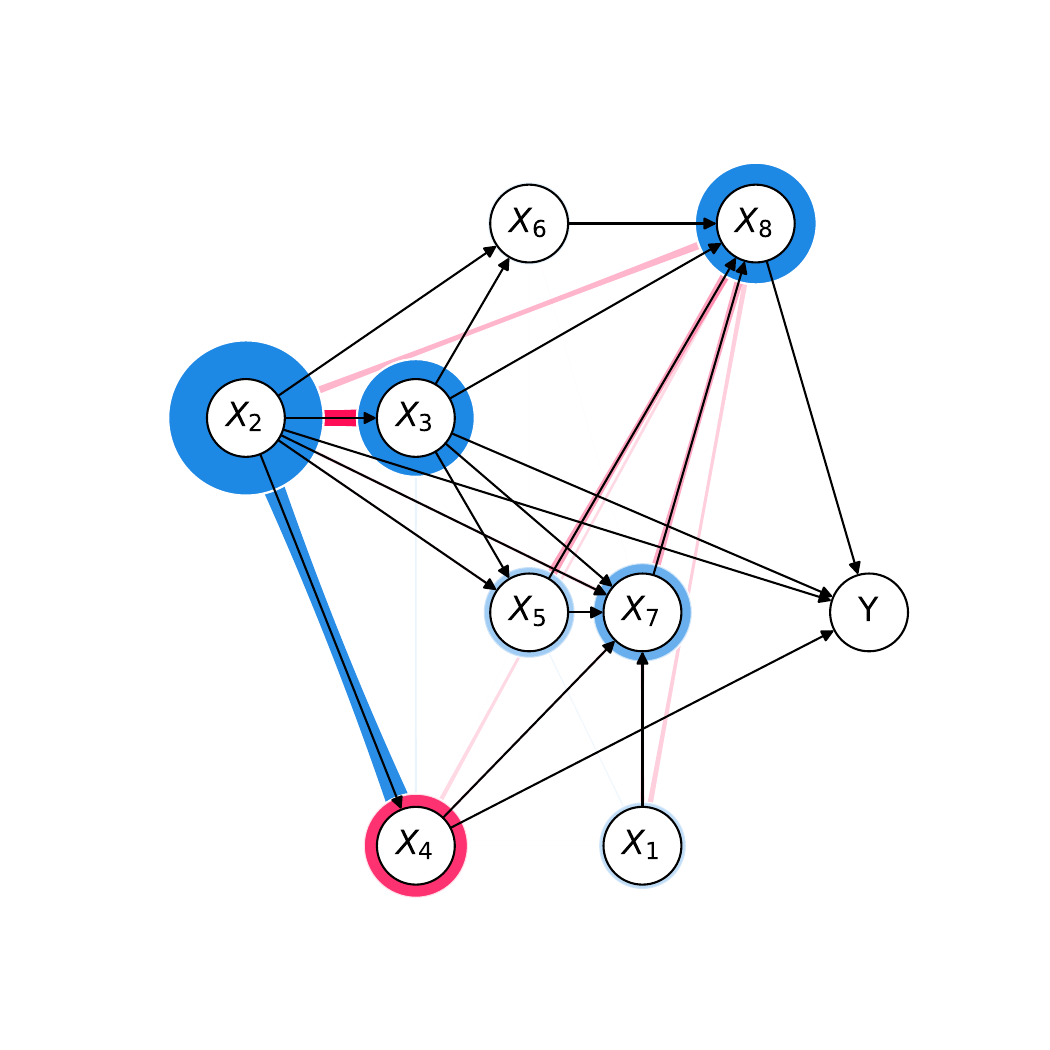}
        \textbf{do-Shapley Interactions}\\(order 2)
    \end{minipage}
    \begin{minipage}[b]{0.245\linewidth}
        \centering
        \includegraphics[width=\linewidth]{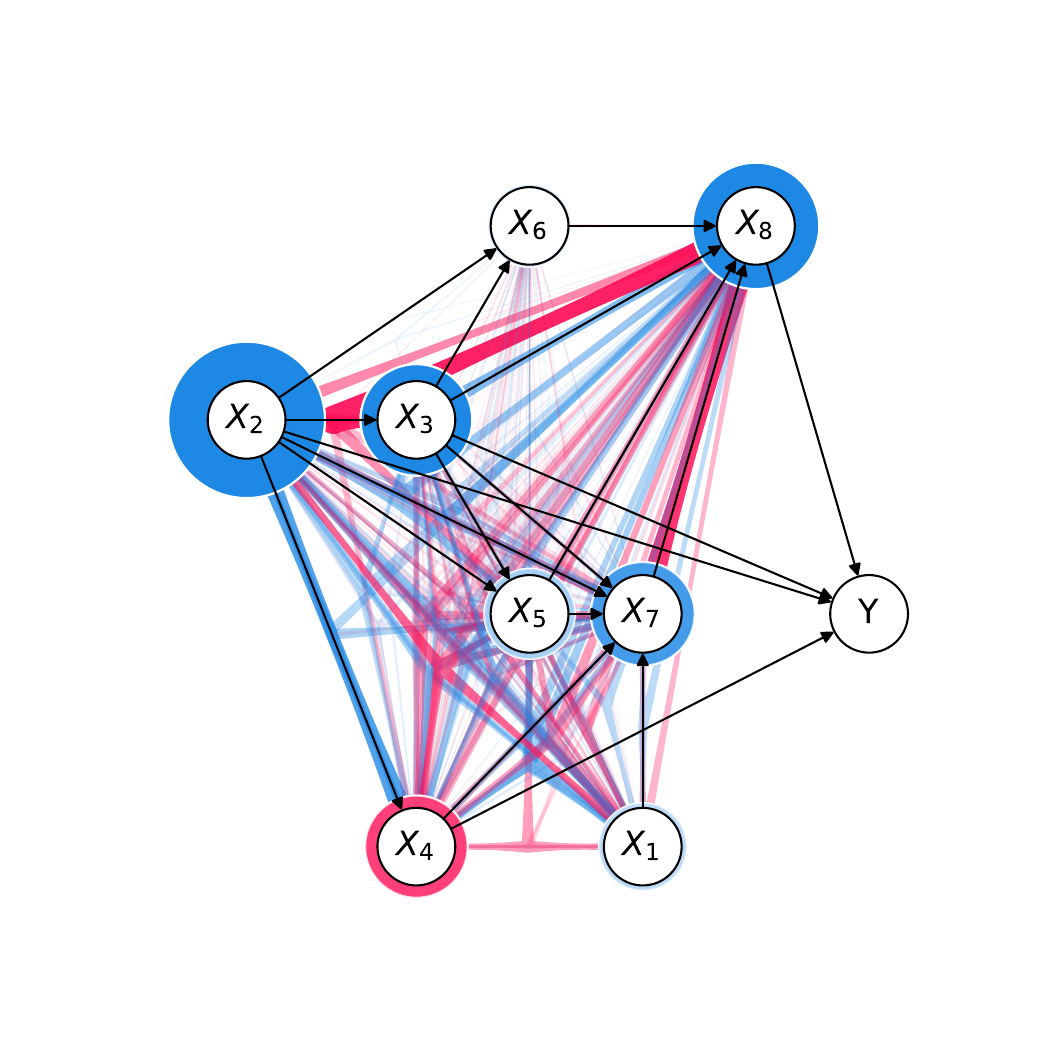}
        \textbf{do-Shapley Interactions}\\(order 8)
    \end{minipage}
    \caption{do-Shapley values and interactions for \textsc{Brazilian\_houses\_reproduced} of increasing order (feature names in \cref{tab:variables_Brazilian_houses_reproduced}). The size of the nodes and edges (hyperedges) denotes the strength of the effect. The color denotes the direction (blue negative, red  positive).}
    \label{fig:shapiq_Brazilian_houses_reproduced}
\end{figure*}

\begin{figure*}[htb]
    \centering
    \begin{minipage}[b]{0.245\linewidth}
        \centering
        \includegraphics[width=\linewidth]{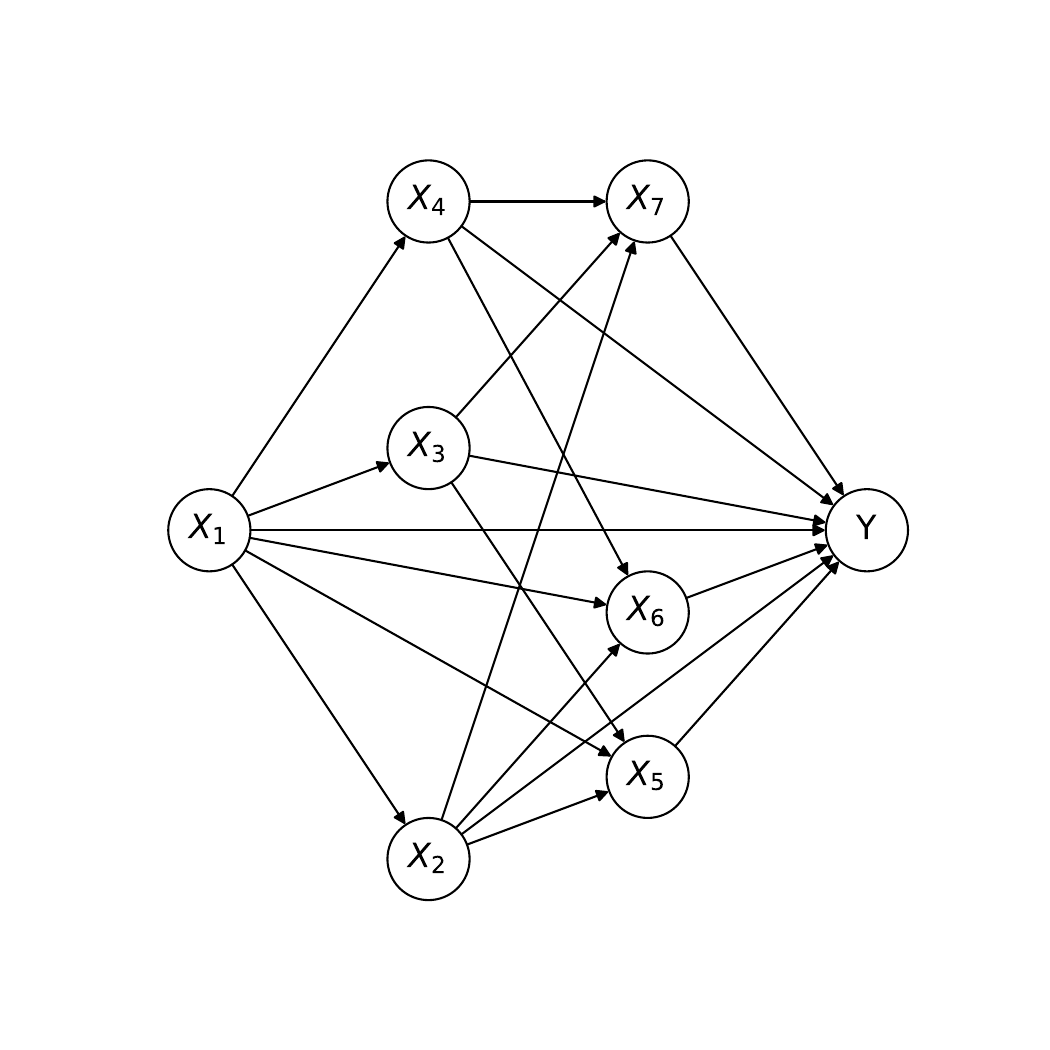}
        \textbf{SCM}\\\phantom{(order 1)}
    \end{minipage}
    \begin{minipage}[b]{0.245\linewidth}
        \centering
        \includegraphics[width=\linewidth]{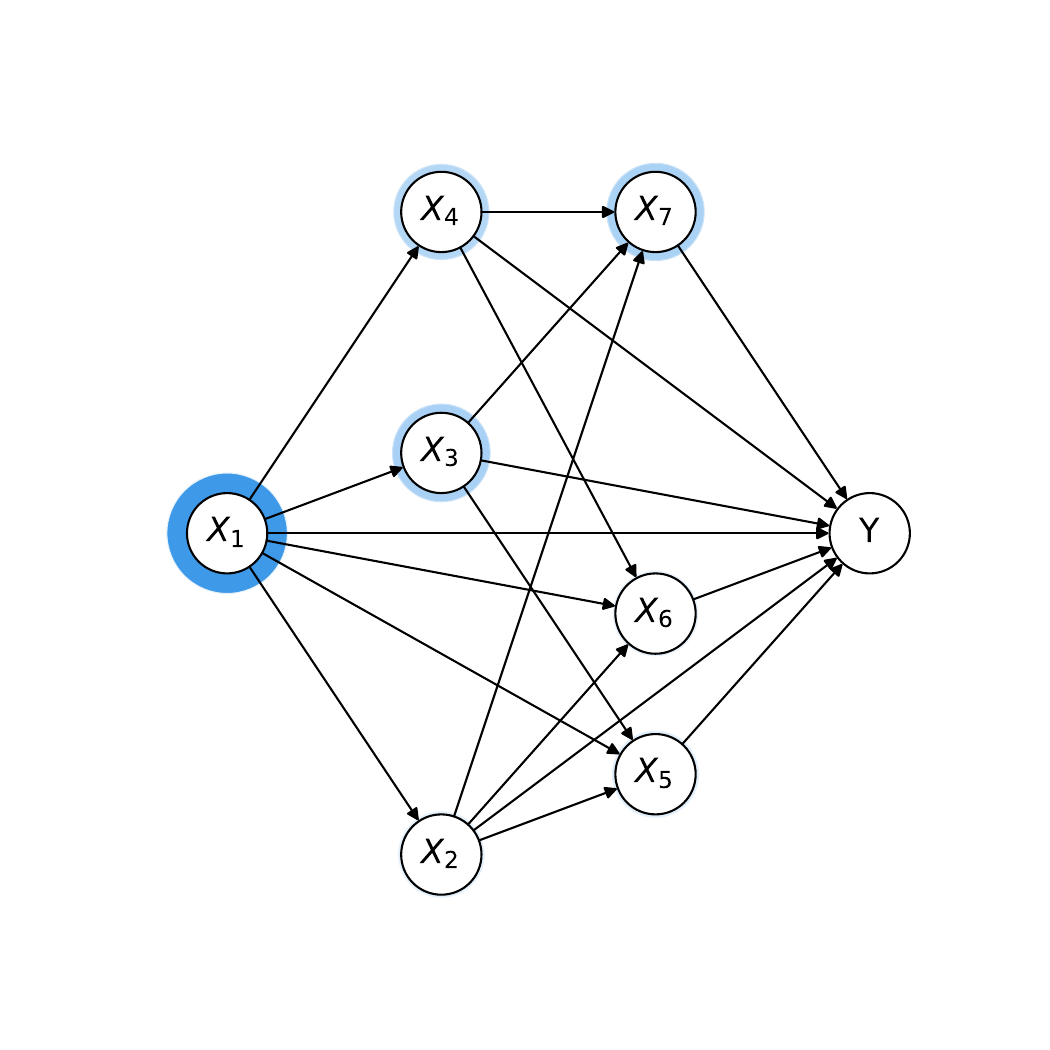}
        \textbf{do-Shapley Values}\\(order 1)
    \end{minipage}
    \begin{minipage}[b]{0.245\linewidth}
        \centering
        \includegraphics[width=\linewidth]{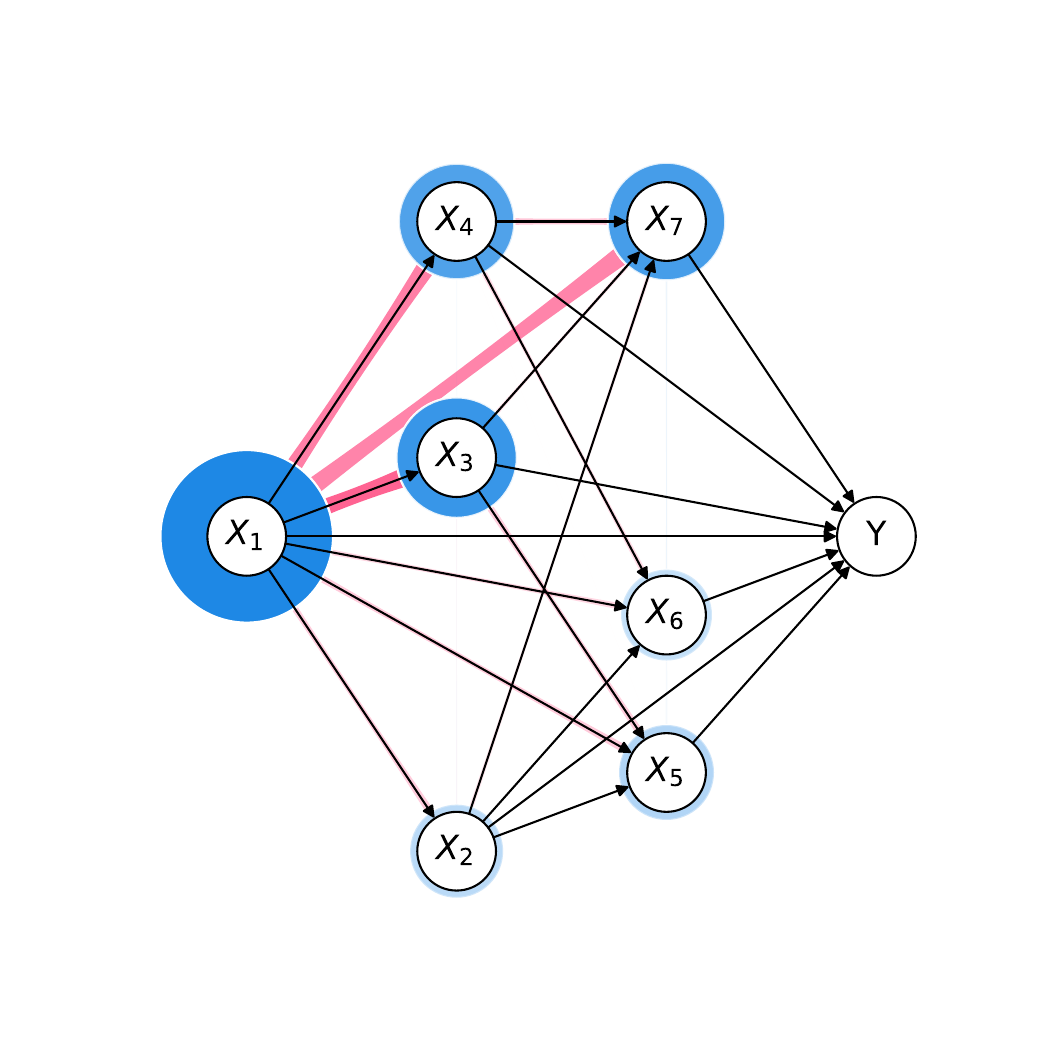}
        \textbf{do-Shapley Interactions}\\(order 2)
    \end{minipage}
    \begin{minipage}[b]{0.245\linewidth}
        \centering
        \includegraphics[width=\linewidth]{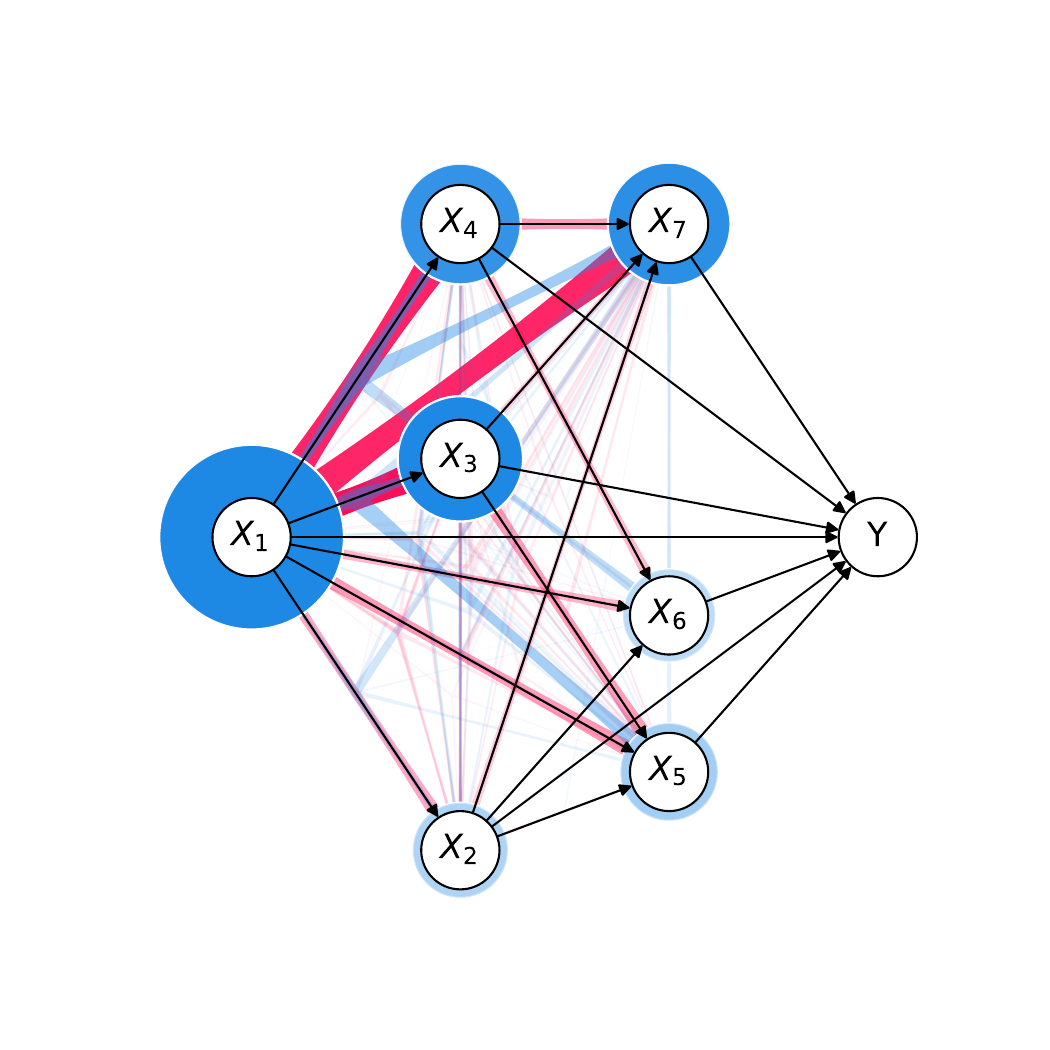}
        \textbf{do-Shapley Interactions}\\(order 7)
    \end{minipage}
    \caption{do-Shapley values and interactions for \textsc{FOREX\_audjpy-hour-High} of increasing order (feature names in \cref{tab:variables_FOREX_audjpy-hour-High}). The size of the nodes and edges (hyperedges) denotes the strength of the effect. The color denotes the direction (blue negative, red  positive).}
    \label{fig:shapiq_FOREX_audjpy-hour-High}
\end{figure*}

\begin{figure*}[htb]
    \centering
    \begin{minipage}[b]{0.245\linewidth}
        \centering
        \includegraphics[width=\linewidth]{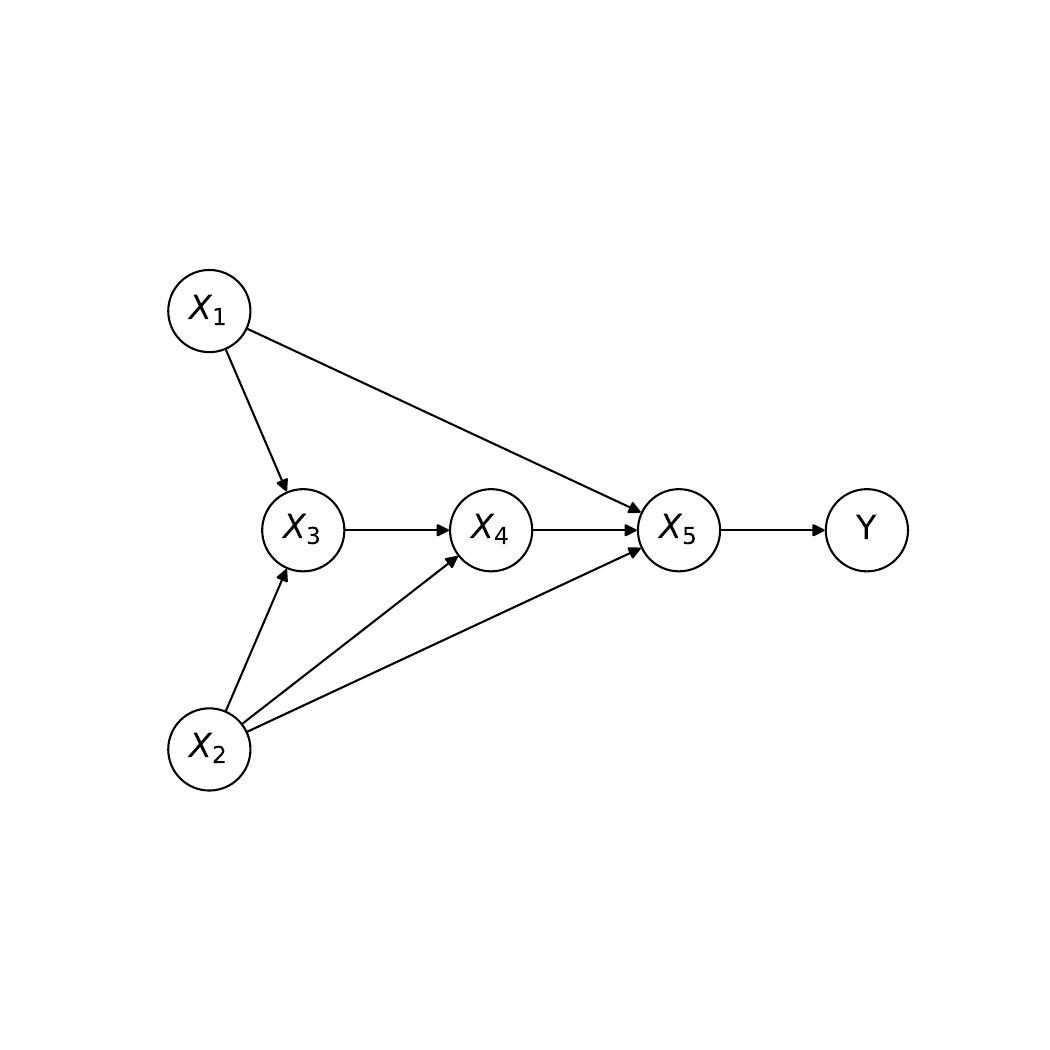}
        \textbf{SCM}\\\phantom{(order 1)}
    \end{minipage}
    \begin{minipage}[b]{0.245\linewidth}
        \centering
        \includegraphics[width=\linewidth]{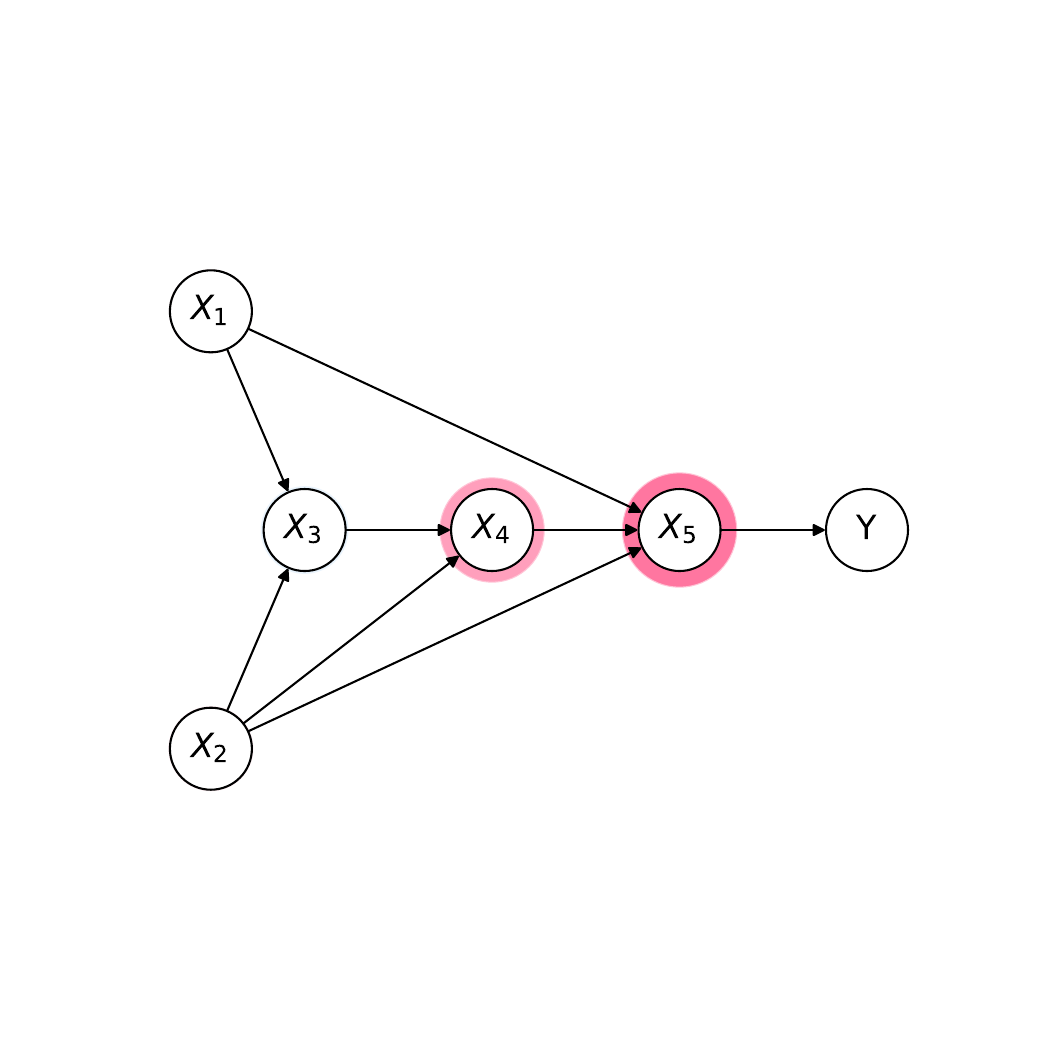}
        \textbf{do-Shapley Values}\\(order 1)
    \end{minipage}
    \begin{minipage}[b]{0.245\linewidth}
        \centering
        \includegraphics[width=\linewidth]{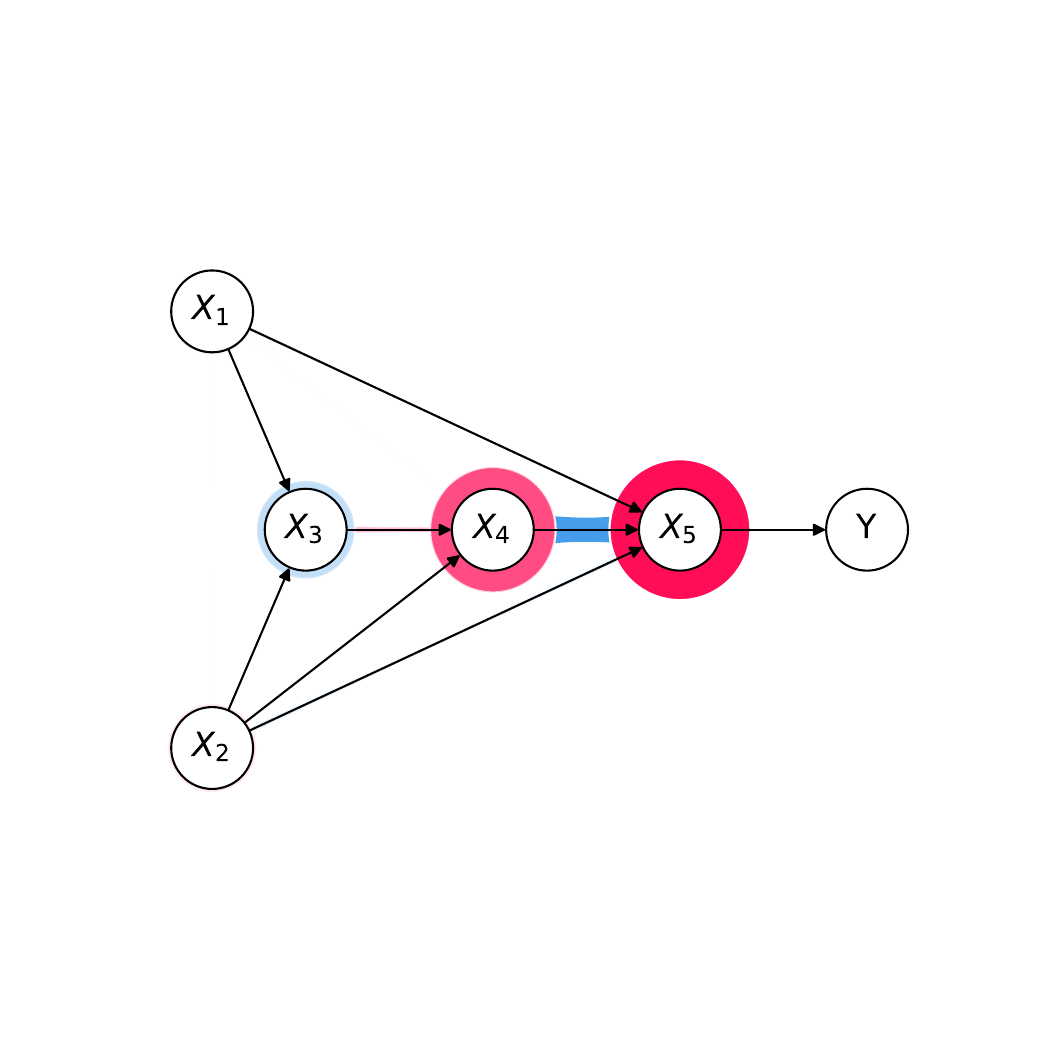}
        \textbf{do-Shapley Interactions}\\(order 2)
    \end{minipage}
    \begin{minipage}[b]{0.245\linewidth}
        \centering
        \includegraphics[width=\linewidth]{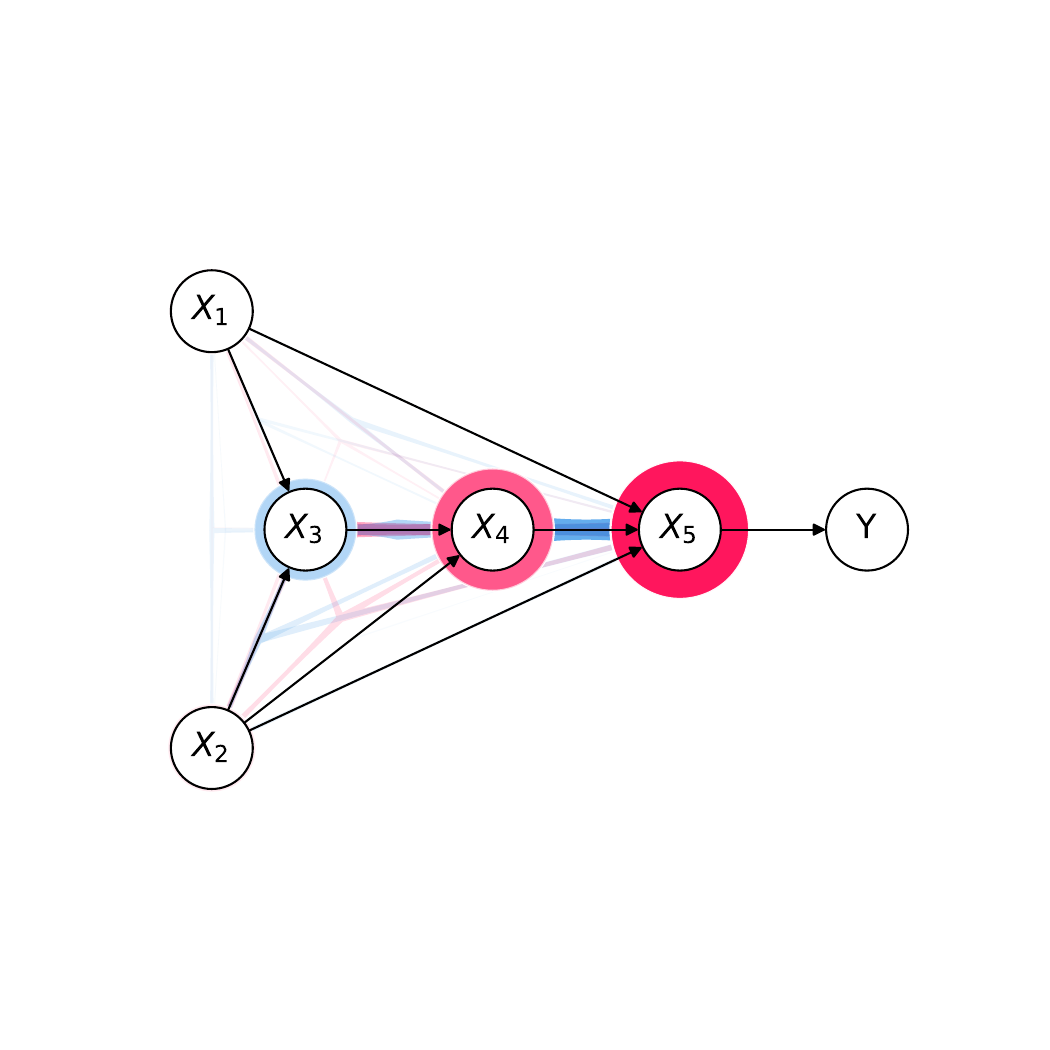}
        \textbf{do-Shapley Interactions}\\(order 5)
    \end{minipage}
    \caption{do-Shapley values and interactions for \textsc{Laptop\_Prices\_Dataset} of increasing order (feature names in \cref{tab:variables_Laptop_Prices_Dataset}). The size of the nodes and edges (hyperedges) denotes the strength of the effect. The color denotes the direction (blue negative, red  positive).}
    \label{fig:shapiq_Laptop_Prices_Dataset}
\end{figure*}

\begin{figure*}[htb]
    \centering
    \begin{minipage}[b]{0.245\linewidth}
        \centering
        \includegraphics[width=\linewidth]{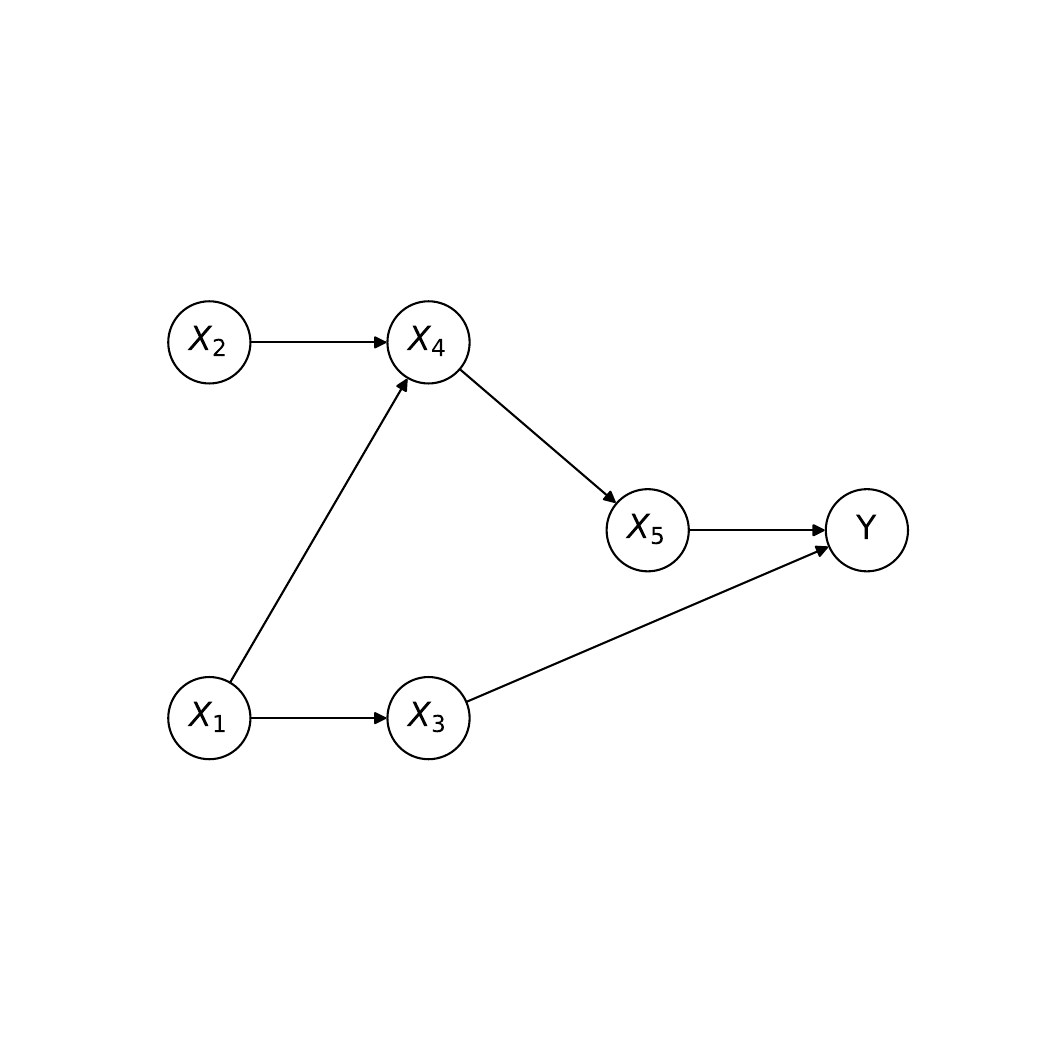}
        \textbf{SCM}\\\phantom{(order 1)}
    \end{minipage}
    \begin{minipage}[b]{0.245\linewidth}
        \centering
        \includegraphics[width=\linewidth]{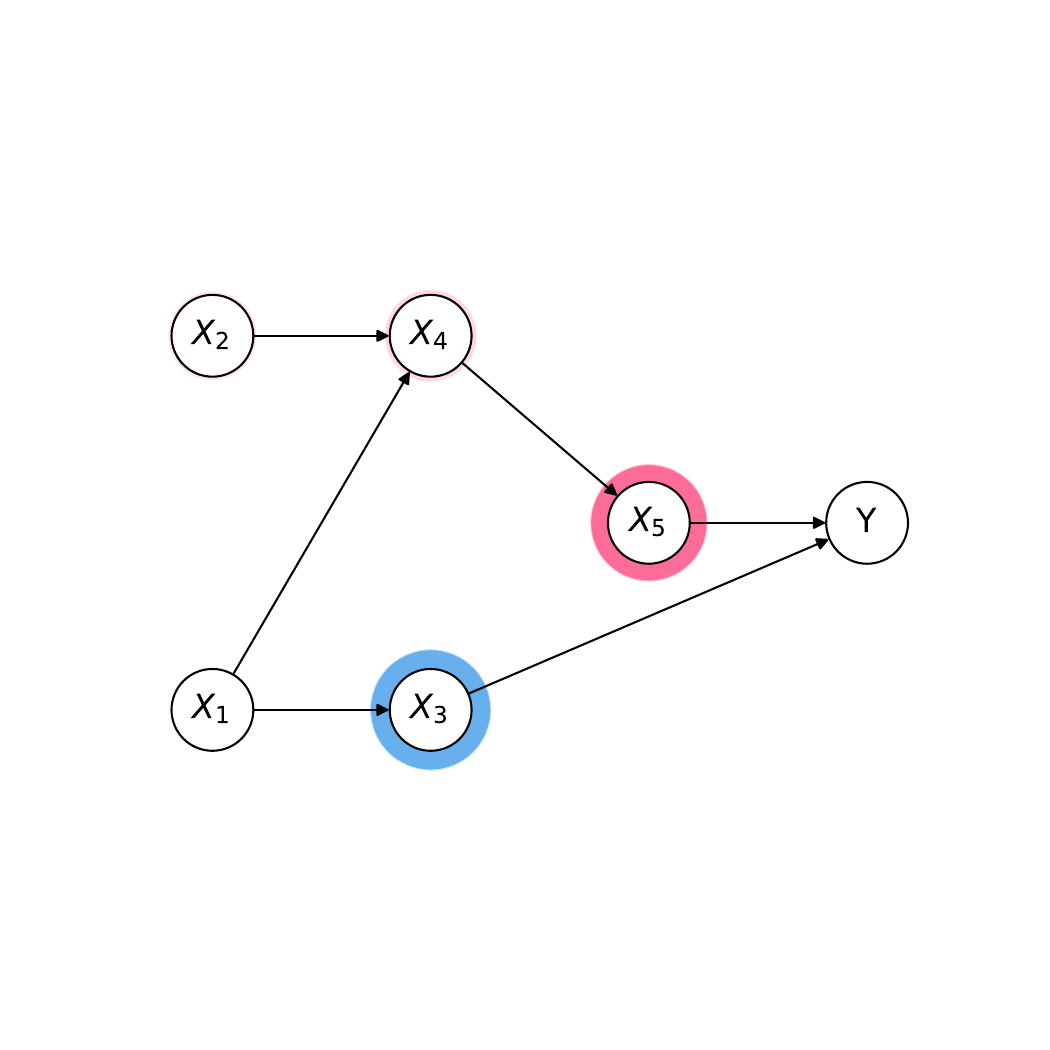}
        \textbf{do-Shapley Values}\\(order 1)
    \end{minipage}
    \begin{minipage}[b]{0.245\linewidth}
        \centering
        \includegraphics[width=\linewidth]{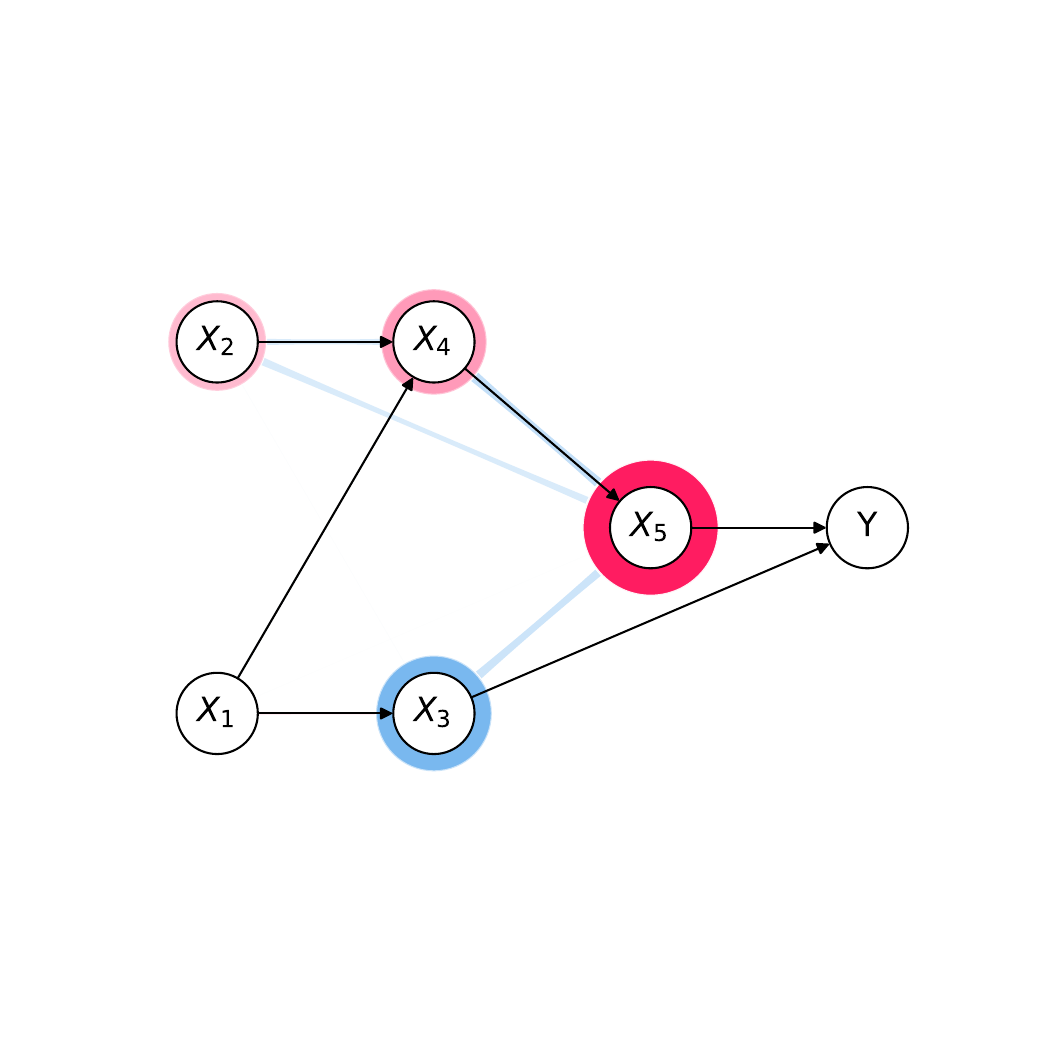}
        \textbf{do-Shapley Interactions}\\(order 2)
    \end{minipage}
    \begin{minipage}[b]{0.245\linewidth}
        \centering
        \includegraphics[width=\linewidth]{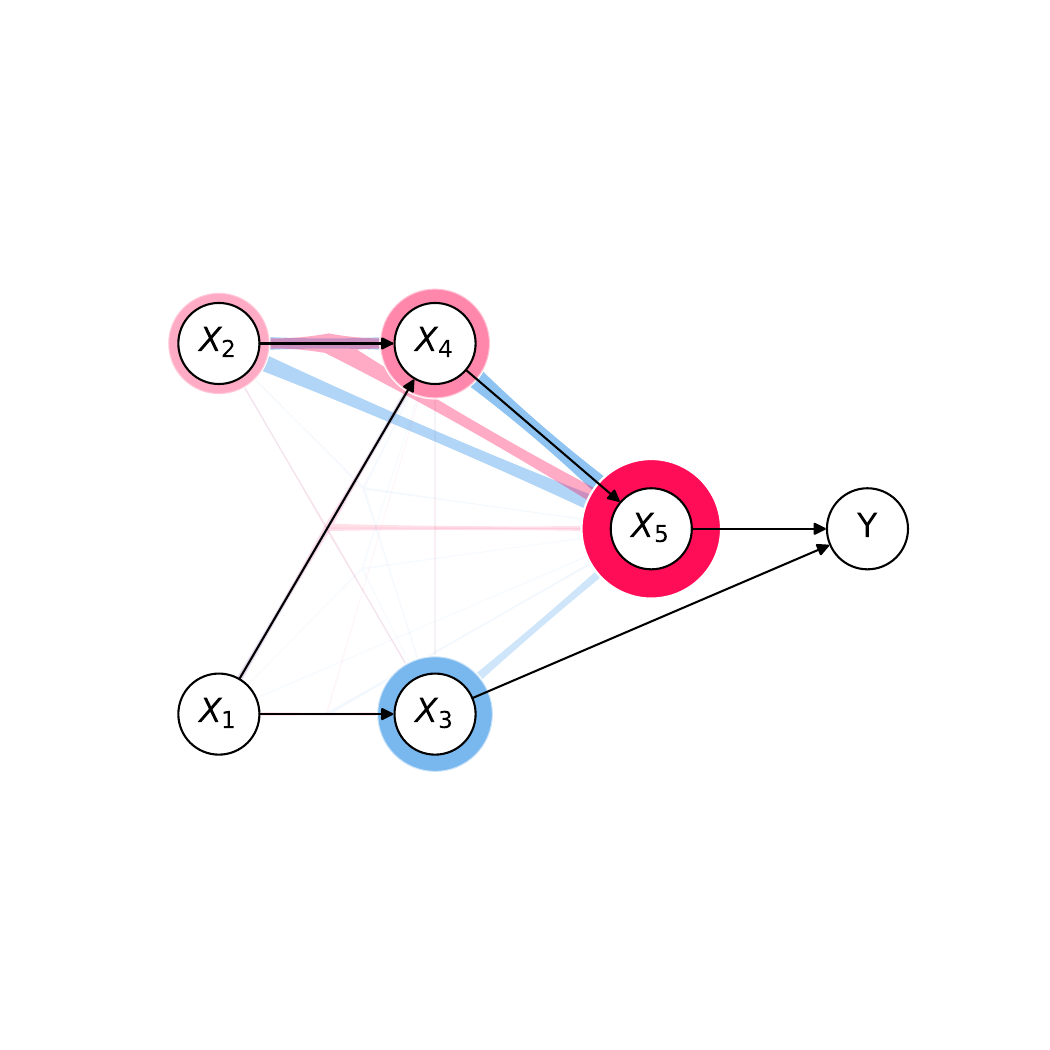}
        \textbf{do-Shapley Interactions}\\(order 5)
    \end{minipage}
    \caption{do-Shapley values and interactions for \textsc{Yeast} of increasing order (feature names in \cref{tab:variables_yeast}). The size of the nodes and edges (hyperedges) denotes the strength of the effect. The color denotes the direction (blue negative, red  positive).}
    \label{fig:shapiq_yeast}
\end{figure*}

\begin{table}[htb]
\centering
\begin{tabular}{llc}
\toprule
\textbf{Node} & \textbf{Feature Name} \\
\midrule
$X_{1}$ & \texttt{hoa\_(BRL)} \\
$X_{2}$ & \texttt{fire\_insurance\_(BRL)} \\
$X_{3}$ & \texttt{parking\_spaces} \\
$X_{4}$ & \texttt{rent\_amount\_(BRL)} \\
$X_{5}$ & \texttt{rooms} \\
$X_{6}$ & \texttt{property\_tax\_(BRL)} \\
$X_{7}$ & \texttt{target} \\
$X_{8}$ & \texttt{bathroom} \\
\midrule
$Y$ & \texttt{area} \\
\bottomrule
\end{tabular}
\caption{Variables for \textsc{Brazilian\_houses\_reproduced}: $d=8$ input features and 9 nodes in total (including target $Y$).}
\label{tab:variables_Brazilian_houses_reproduced}
\end{table}

\begin{table}[htb]
\centering
\begin{tabular}{llc}
\toprule
\textbf{Node} & \textbf{Feature Name} \\
\midrule
$X_{1}$ & \texttt{alm} \\
$X_{2}$ & \texttt{nuc} \\
$X_{3}$ & \texttt{vac} \\
$X_{4}$ & \texttt{gvh} \\
$X_{5}$ & \texttt{mcg} \\
\midrule
$Y$ & \texttt{mit} \\
\bottomrule
\end{tabular}
\caption{Variables for \textsc{yeast}: $d=5$ input features and 6 nodes in total (including target $Y$).}
\label{tab:variables_yeast}
\end{table}

\begin{table}[htb]
\centering
\begin{tabular}{llc}
\toprule
\textbf{Node} & \textbf{Feature Name} \\
\midrule
$X_{1}$ & \texttt{N\_0} \\
$X_{2}$ & \texttt{N\_1} \\
$X_{3}$ & \texttt{N\_5} \\
$X_{4}$ & \texttt{target} \\
$X_{5}$ & \texttt{N\_4} \\
\midrule
$Y$ & \texttt{N\_2} \\
\bottomrule
\end{tabular}
\caption{Variables for \textsc{Laptop\_Prices\_Dataset}: $d=5$ input features and 6 nodes in total (including target $Y$).}
\label{tab:variables_Laptop_Prices_Dataset}
\end{table}

\begin{table}[htb]
\centering
\begin{tabular}{llc}
\toprule
\textbf{Node} & \textbf{Feature Name} \\
\midrule
$X_{1}$ & \texttt{Ask\_Low} \\
$X_{2}$ & \texttt{Bid\_Low} \\
$X_{3}$ & \texttt{Bid\_Close} \\
$X_{4}$ & \texttt{Ask\_Open} \\
$X_{5}$ & \texttt{Ask\_Close} \\
$X_{6}$ & \texttt{Bid\_Open} \\
$X_{7}$ & \texttt{Ask\_High} \\
\midrule
$Y$ & \texttt{Bid\_High} \\
\bottomrule
\end{tabular}
\caption{Variables for \textsc{FOREX\_audjpy-hour-High}: $d=7$ input features and 8 nodes in total (including target $Y$).}
\label{tab:variables_FOREX_audjpy-hour-High}
\end{table}

\clearpage
\section{Additional Experiments}
\label{app:exp}

\begin{figure}[h]
    \centering
    \includegraphics[width=0.6\linewidth]{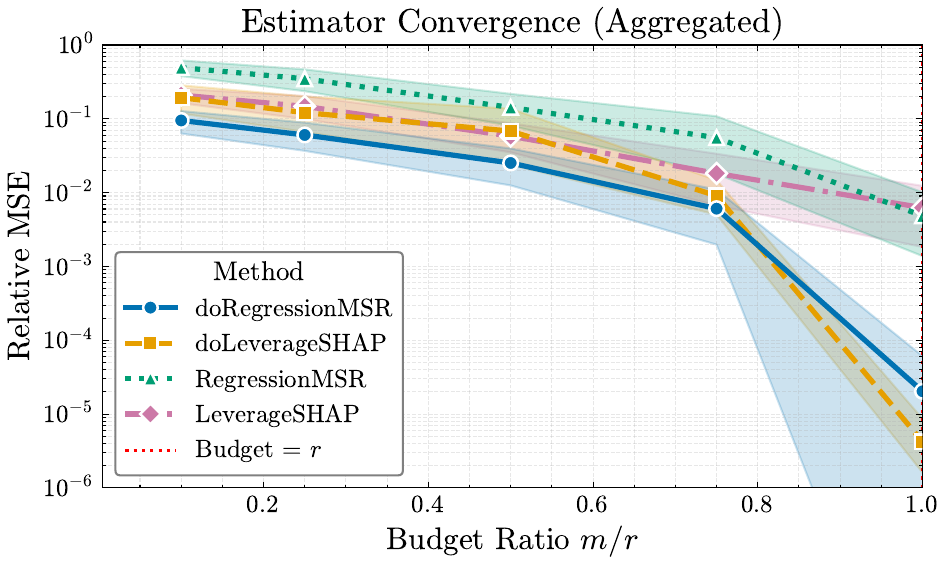}
    \caption{\textbf{Estimator Convergence (Aggregated), Clipped View.} The Relative Mean Squared Error (MSE) of Shapley value estimates versus the budget ratio $m/r$, aggregated across all datasets. This is the same plot as \Cref{fig:convergence}, but where we have adjusted the view to more clearly show the distinctions in estimator performance for when $m/r\leq1$.}
    \label{fig:convergence_clipped}
\end{figure}

\begin{figure}[h]
    \centering
    \includegraphics[width=0.6\linewidth]{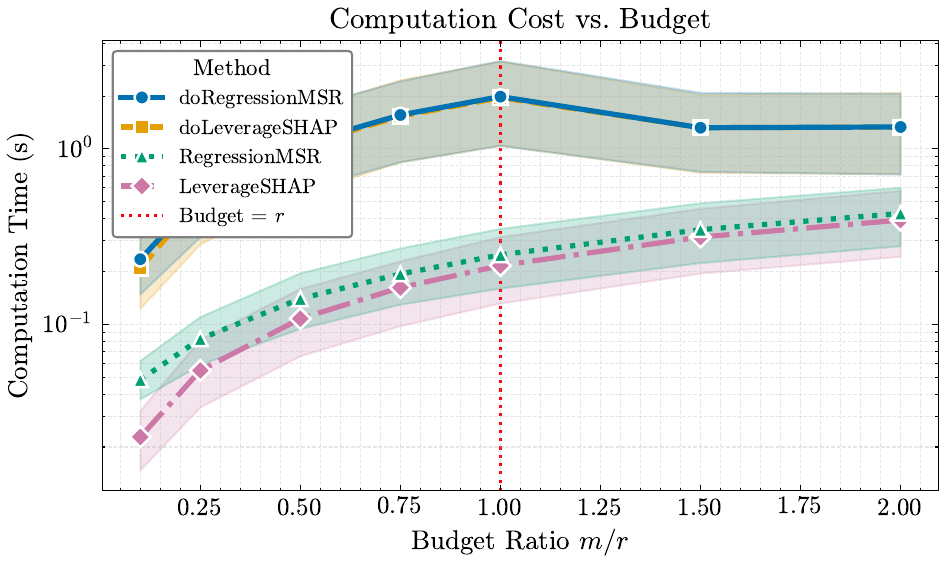}
    \caption{\textbf{Computational Cost.} Average wall-clock time (seconds) versus budget ratio ($m/r$). The structure-aware estimators (i.e. ``do''-variants) incur a consistent runtime overhead compared to their structure-agnostic counterparts. This additional cost represents the time required for graph traversal and boundary sampling to identify distinct equivalence classes. The more computationally demanding structural exploration is the necessary trade-off to achieve the orders-of-magnitude reduction in estimation error seen in e.g. \Cref{fig:convergence}.}
    \label{fig:time_convergence}
\end{figure}

\begin{table*}[h!]
\centering
\caption{Relative MSE statistics by budget ratio ($m/r$).}
\label{tab:mse_stats}
\resizebox{0.75\linewidth}{!}{
\begin{tabular}{l cc | cc}
\toprule
Metric & \texttt{LeverageSHAP} & \texttt{doLeverageSHAP} & \texttt{RegressionMSR} & \texttt{doRegressionMSR} \\
\midrule
\multicolumn{5}{c}{\textbf{Budget Ratio} $m/r = 0.25$} \\
\midrule
Mean & \cellcolor{bronze!30}1.5e-1 & \cellcolor{silver!30}1.2e-1 & 3.5e-1 & \cellcolor{gold!30}6.0e-2 \\
Median & \cellcolor{bronze!30}8.2e-3 & 9.8e-3 & \cellcolor{silver!30}7.9e-3 & \cellcolor{gold!30}3.8e-4 \\
$Q_1$ (25\%) & 8.5e-4 & \cellcolor{bronze!30}4.1e-4 & \cellcolor{silver!30}1.5e-4 & \cellcolor{gold!30}8.1e-8 \\
$Q_2$ (50\%) & \cellcolor{bronze!30}8.2e-3 & 9.8e-3 & \cellcolor{silver!30}7.9e-3 & \cellcolor{gold!30}3.8e-4 \\
$Q_3$ (75\%) & \cellcolor{bronze!30}2.2e-1 & \cellcolor{silver!30}5.9e-2 & 3.0e-1 & \cellcolor{gold!30}1.1e-2 \\
$Q_4$ (Max) & \cellcolor{bronze!30}2.5e+00 & 5.6e+00 & \cellcolor{silver!30}2.4e+00 & \cellcolor{gold!30}1.2e+00 \\
\midrule
\multicolumn{5}{c}{\textbf{Budget Ratio} $m/r = 0.5$} \\
\midrule
Mean & \cellcolor{silver!30}5.8e-2 & \cellcolor{bronze!30}6.8e-2 & 1.4e-1 & \cellcolor{gold!30}2.5e-2 \\
Median & \cellcolor{bronze!30}1.9e-3 & 4.8e-3 & \cellcolor{silver!30}8.8e-4 & \cellcolor{gold!30}1.1e-5 \\
$Q_1$ (25\%) & 2.0e-4 & \cellcolor{bronze!30}1.4e-4 & \cellcolor{silver!30}1.5e-5 & \cellcolor{gold!30}5.9e-9 \\
$Q_2$ (50\%) & \cellcolor{bronze!30}1.9e-3 & 4.8e-3 & \cellcolor{silver!30}8.8e-4 & \cellcolor{gold!30}1.1e-5 \\
$Q_3$ (75\%) & \cellcolor{bronze!30}1.3e-2 & 1.9e-2 & \cellcolor{silver!30}9.7e-3 & \cellcolor{gold!30}1.4e-3 \\
$Q_4$ (Max) & \cellcolor{silver!30}1.2e+00 & 5.6e+00 & \cellcolor{bronze!30}2.0e+00 & \cellcolor{gold!30}7.0e-1 \\
\midrule
\multicolumn{5}{c}{\textbf{Budget Ratio} $m/r = 0.75$} \\
\midrule
Mean & \cellcolor{bronze!30}1.8e-2 & \cellcolor{silver!30}9.1e-3 & 5.5e-2 & \cellcolor{gold!30}6.1e-3 \\
Median & \cellcolor{bronze!30}4.8e-4 & 9.9e-4 & \cellcolor{silver!30}1.4e-4 & \cellcolor{gold!30}5.7e-7 \\
$Q_1$ (25\%) & 7.7e-5 & \cellcolor{bronze!30}2.2e-5 & \cellcolor{silver!30}2.8e-6 & \cellcolor{gold!30}3.8e-10 \\
$Q_2$ (50\%) & \cellcolor{bronze!30}4.8e-4 & 9.9e-4 & \cellcolor{silver!30}1.4e-4 & \cellcolor{gold!30}5.7e-7 \\
$Q_3$ (75\%) & \cellcolor{bronze!30}2.5e-3 & 6.0e-3 & \cellcolor{silver!30}2.3e-3 & \cellcolor{gold!30}1.5e-4 \\
$Q_4$ (Max) & \cellcolor{bronze!30}6.6e-1 & \cellcolor{silver!30}3.6e-1 & 2.0e+00 & \cellcolor{gold!30}3.3e-1 \\
\midrule
\multicolumn{5}{c}{\textbf{Budget Ratio} $m/r = 1.0$} \\
\midrule
Mean & 6.2e-3 & \cellcolor{gold!30}4.2e-6 & \cellcolor{bronze!30}4.8e-3 & \cellcolor{silver!30}2.0e-5 \\
Median & 2.4e-4 & \cellcolor{gold!30}4.1e-29 & \cellcolor{bronze!30}4.4e-6 & \cellcolor{silver!30}3.5e-13 \\
$Q_1$ (25\%) & 2.8e-5 & \cellcolor{gold!30}1.1e-30 & \cellcolor{bronze!30}2.0e-7 & \cellcolor{silver!30}2.6e-14 \\
$Q_2$ (50\%) & 2.4e-4 & \cellcolor{gold!30}4.1e-29 & \cellcolor{bronze!30}4.4e-6 & \cellcolor{silver!30}3.5e-13 \\
$Q_3$ (75\%) & 1.2e-3 & \cellcolor{gold!30}6.1e-26 & \cellcolor{bronze!30}5.1e-4 & \cellcolor{silver!30}2.6e-12 \\
$Q_4$ (Max) & 3.3e-1 & \cellcolor{gold!30}3.7e-4 & \cellcolor{bronze!30}3.0e-1 & \cellcolor{silver!30}3.7e-3 \\
\bottomrule
\end{tabular}
}
\end{table*}

\end{document}